\documentclass[msom,nonblindrev]{workingpaperJJB}

\usepackage{xspace,nicefrac,placeins,booktabs}
\RequirePackage{tgtermes}
\RequirePackage{newtxtext}
\RequirePackage{newtxmath}
\RequirePackage{bm}
\RequirePackage{endnotes}\newcommand{\ma}[1]{\mathbf{#1}}

\OneAndAHalfSpacedXII 

\usepackage{algorithm}
\usepackage{algpseudocode}
\usepackage{tikz}
\newcommand{\R}{\mathbb{R}}
\newcommand{\ve}[1]{\boldsymbol{#1}}
\newcommand{\mac}[1]{\mathcal{#1}}
\newcommand{\pr}{\mathbf{P}\xspace}
\newcommand{\ex}{\mathbf{E}\xspace}

\DeclareMathOperator{\trace}{trace}

\usepackage{natbib}
 \bibpunct[, ]{(}{)}{,}{a}{}{,}%
 \def\bibfont{\small}%
\let\textcite\citet

\EquationsNumberedThrough    

\TheoremsNumberedThrough     
\ECRepeatTheorems  %

\MANUSCRIPTNO{MSOM-0001-2024.00}

\begin{document}


\RUNAUTHOR{}

\RUNTITLE{Optimizing Digital Therapeutic Interventions}

\TITLE{Optimizing Digital Therapeutic Interventions: Online Learning under Endogenous Adherence}

\ARTICLEAUTHORS{%
\AUTHOR{Eric Pulick}
\AFF{Department of Industrial and Systems Engineering,
University of Wisconsin-Madison, \EMAIL{pulick@wisc.edu}}

\AUTHOR{Stephanie Carpenter}
\AFF{College of Health Solutions,
Arizona State University, \EMAIL{stephanie.m.carpenter@asu.edu}}

\AUTHOR{Matthew Buman}
\AFF{College of Health Solutions,
Arizona State University, \EMAIL{matthew.buman@asu.edu}}

\AUTHOR{Yonatan Mintz}
\AFF{Department of Industrial and Systems Engineering,
University of Wisconsin-Madison, \EMAIL{ymintz@wisc.edu}}
} 

\ABSTRACT{
\textbf{Problem definition:} A critical challenge facing clinicians managing chronic disease interventions is sustaining long-run patient health given limited information and resources. Digital therapeutics (DTs) provide a cost-effective way to manage interventions at scale through repeated interactions (e.g. daily treatment recommendations), but patient success is highly dependent on their adherence. Behavioral psychology suggests that both treatment recommendations and past adherence affect future adherence, yet existing decision support frameworks for DTs model only recommendation effects or treat adherence as exogenous context, leaving a key gap in model and algorithm development. \textbf{Methodology/results:} To address this gap, we present a DT decision support framework that captures both recommendation and adherence effects, allowing clinicians to better plan treatment recommendations. We model a patient's time-varying capacity for engagement with treatment using a linear dynamical system (LDS) that captures both recommendation and adherence effects, endogenously connected to adherence behavior with a logit link. We establish finite-time identification guarantees for this model, extending LDS results to our setting. Next, we propose an optimism-based algorithm, \textsc{UCB-BOLD}, for online treatment selection and prove that it achieves sublinear regret. We evaluate \textsc{UCB-BOLD} against benchmarks via ablation studies on a synthetic patient cohort generated using micro-randomized trial data collected to develop a workplace sedentary behavior intervention. \textbf{Managerial implications:} DT decision support tools can include dynamical models to enable decision makers to efficiently use the data in DT settings to improve patient health through effective resource allocation. While myopic or heuristic approaches suffice for some patient types, the benefits of explicitly planning around recommendation and adherence effects are significant for others; \textsc{UCB-BOLD} achieves 2-3x lower conditional value-at-risk regret than the next-best benchmark.}



\maketitle

\section{Introduction}
Chronic disorders, including metabolic disease and addiction, affect millions of Americans (34.7\% and 16.8\% of adults, respectively) \citep{palaniappan20262026,SAMHSA2025nsduh}. A cost-effective approach to manage treatment of these disorders is through the use of digital therapeutics (DTs). DTs provide individualized behavioral interventions to address a wide range of conditions, such as lack of physical activity (PA) \citep{springAdaptiveBehavioralIntervention2024}, 
 mental illness \citep{ben-zeevMobileHealthMHealth2018}, substance use disorders (SUDs) \citep{gustafsonSmartphoneApplicationSupport2014d}, and diabetes \citep{sepahEngagementOutcomesDigital2017}. DTs interact repeatedly with patients but depend on patient adherence to be effective. This creates a sequential decision-making (SDM) problem for the DT provider; at each decision point, the provider gathers relevant patient data and chooses which (if any) treatment to recommend, while the patient chooses whether to adhere to the given recommendation. Both the DT's recommendation and the patient's adherence decision can affect the patient's future adherence behavior. For instance, a recommendation may impose cognitive burden regardless of adherence, and adherence itself can be either fatiguing or habit-forming. Further, these effects may vary by individual or treatment, and must be learned online. Existing DT frameworks model only recommendation effects or treat adherence as exogenous context. Thus there is a critical need for frameworks that model both adherence and recommendation effects to help DTs better support long-term patient health. 

In this paper, we present an SDM framework for DT decision support. We model a patient's capacity for engaging with treatment as a linear dynamical system (LDS) whose dynamics depend on previous capacity, treatment recommendations, and patient adherence, with capacity connected endogenously to adherence using a logit link. This framework can capture both beneficial (e.g., habit-forming) and detrimental (e.g., burden or fatigue) effects. The model is deliberately parsimonious, designed to capture the relevant recommendation and adherence phenomena while remaining interpretable for clinicians and tractable for optimization. We present a statistical analysis of this model and demonstrate finite-time identification guarantees. We also propose an online optimization algorithm for DT treatment recommendations, \textsc{UCB-BOLD}, and prove that it achieves sublinear regret in this setting. We conclude with a case study on a synthetic patient cohort, calibrated using data from a sedentary behavior micro-randomized trial (MRT), where we show that \textsc{UCB-BOLD} outperforms benchmark algorithms. This framework is aimed particularly at long-running behavioral interventions, where sustaining patient engagement over time is especially relevant.

\subsection{Problem Context: Digital Therapeutics and Adaptive Interventions}
Sustained patient engagement is an important determinant of long-term patient outcomes in conditions targeted by DT interventions. Within SUD, for instance, patient engagement with months-to-years of continuing care (i.e., lower intensity, but sustained monitoring and support) is an important predictor of durable patient recovery \citep{mclellanCanSubstanceUse2013,proctorContinuingCareModel2014}. Similarly, long-term patient engagement is needed to achieve the health benefits associated with increased PA or weight management \citep{jakicicPhysicalActivityPrevention2019}, reflected in the length of recent studies like the $12$-month HeartSteps II trial \citep{spruijt-metzAdvancingBehavioralIntervention2022}. Across these settings, however, patient engagement has been observed to decrease over time \citep{baumelObjectiveUserEngagement2019}, with diminished intervention effects following repeated prompt delivery \citep{klasnjaEfficacyContextuallyTailored2019}. This motivates DT approaches that better manage patient engagement to improve long-run patient outcomes. 

The just-in-time-adaptive-intervention (JITAI) and mHealth literature has increasingly characterized engagement as a time-varying rather than static construct \citep{perskiConceptualisingEngagementDigital2017a,nahum-shaniEngagementDigitalInterventions2022a,nahum-shaniScienceEngagementDigital2024b}. Theoretical frameworks, such as the Cumulative Complexity model \citep{shippeeCumulativeComplexityFunctional2012a} and its extension to the digital health setting \citep{crossDigitalCumulativeComplexity2024}, propose viewing patient adherence and engagement behavior as functions of an underlying, time-varying patient capacity. This capacity is affected by mechanisms like burden, fatigue, and habit formation that the JITAI literature has highlighted as important considerations for future intervention design \citep{nahum-shaniJustinTimeAdaptiveInterventions2026}. These conceptual frameworks align with a broader shift in the psychology literature toward dynamical system models as tools to understand and approximate individual behavior \citep{brigantiNetworkAnalysisOverview2024a,perskiIterativeDevelopmentRefinement2025}, further motivating the integration of explicit generative models into DT decision support tools. 

To this point, though, few personalized JITAI (pJITAI) algorithms have incorporated formal dynamics models in their construction \citep{perskiIterativeDevelopmentRefinement2025} as a way to capture both recommendation and adherence effects. Prior work has made significant methodological and deployment contributions (see Section~\ref{sec:lit_review}), but no prior approach combines a parametric dynamical model for mHealth patient engagement with system identification and algorithmic regret guarantees as we do here.

\subsection{Contribution}\label{sec:contributions}
We develop a decision support framework for DT providers with four specific contributions:
\begin{enumerate} 
\item We propose an LDS with stochastic, state-dependent adherence to model a patient's time-varying engagement capacity. This framework enables digital therapeutics to plan based on both recommendation and adherence effects. It also frames time-varying patient behavior in a clinically interpretable way, aligned with trends toward dynamic patient models in the psychology literature.
\item We establish finite-time system identification guarantees under a structured policy class that permits arbitrary policy behavior (possibly adaptive or adversarial) outside of exploration steps. We show that dynamics and adherence parameters are identifiable at the typical $\tilde {\mac O}(\nicefrac{1}{\sqrt{T}})$ rate, extending LDS guarantees to a setting with $1$-sparse inputs and state-dependent adherence.
\item  We develop an optimism-based algorithm, \textsc{UCB-BOLD}, to optimize DT treatment recommendations. We demonstrate that \textsc{UCB-BOLD} achieves $\tilde{\mac O}(\sqrt{T})$ regret in this context, making it competitive with other reinforcement learning (RL) algorithms.
\item We fit a Bayesian hierarchical model with data from a sedentary behavior MRT. We use this model to generate synthetic patients and evaluate \textsc{UCB-BOLD} against heuristic, myopic, and model-free RL benchmark algorithms on simulated trajectories in ablation studies. We find that \textsc{UCB-BOLD} achieves 2-3x lower conditional value-at-risk regret than the next best benchmark.
\end{enumerate}

\section{Literature review}\label{sec:lit_review}
Our work draws from four main streams of literature: SDM in healthcare operations, JITAIs, mHealth-motivated bandits, and model-based control.
\subsection{SDM in healthcare operations}
There is a rich body of work exploring patient-specific SDM in healthcare operations. For example, Markov decision process (MDP) and partially-observed Markov decision process (POMDP) models have been used to analyze organ transplantation \citep{alagozOptimalTimingLivingDonor2004}, treatment initiation \citep{shechterOptimalTimeInitiate2008}, chronic disease management \citep{masonOptimizingStatinTreatment2012,helmDynamicForecastingControl2015}, and disease screening \citep{ayerForumPOMDPApproach2012}. Recent work has emphasized perspectives from machine learning \citep{bertsimasPredictivePrescriptiveAnalytics2020a}, focusing on data-driven guarantees for personalized healthcare settings \citep{bastaniOnlineDecisionMaking2020b}. Our work contributes to this stream of literature, specifically focused on the setting of online learning of individual behavior models \citep{aswaniBehavioralModelingWeight2019b}. In contrast to these works, we focus on explicit modeling of patient engagement and its impact on policy decisions.

\subsection{Just-in-time-adaptive-interventions}
JITAIs are principled frameworks for delivering personalized behavioral interventions \citep{nahum-shaniJustinTimeAdaptiveInterventions2018a}, often optimized with data from MRTs \citep{klasnjaMicrorandomizedTrialsExperimental2015}. JITAIs have been deployed effectively at scale, for instance in the HeartSteps program for physical activity \citep{klasnjaEfficacyContextuallyTailored2019,spruijt-metzAdvancingBehavioralIntervention2022} and the Oralytics program for oral care \citep{trellaDeployedOnlineReinforcement2025a}. Our work complements this literature, particularly for long-running interventions such as SUD continuing care \citep{proctorContinuingCareModel2014}.

\subsection{mHealth contextual and non-stationary bandits}
Recently deployed pJITAIs use Thompson sampling contextual bandits that handle temporal dynamics through hand-crafted features, such as recent treatment dosage \citep{liaoPersonalizedHeartStepsReinforcement2020a}. Extensions include partial pooling of information across patients to better leverage scarce samples \citep{tomkinsIntelligentPoolingPracticalThompson2021,abbottRoMERobustMixedEffects2024}. Importantly, these methods do not have explicit dynamics models, instead treating context as exogenous and tuning rewards to account for time-varying processes. Non-stationary and restless bandit approaches are an alternative means of capturing time-varying behavior in the bandit setting \citep{whittleRestlessBanditsActivity1988}, with extensions exploring various structural assumptions \citep{garivier2008upper,besbes2014stochastic}. mHealth-motivated non-stationary bandit work has emerged from both model-free \citep{baekPolicyOptimizationPersonalized2025} and model-based \citep{mintzNonstationaryBanditsHabituation2020b} perspectives. \citet{mintzNonstationaryBanditsHabituation2020b} and extensions \citep{heNonstationaryBanditsHabituation2024a,liPowerConstrainedNonstationary2025} are closest to our work, modeling the case where each arm has action-dependent habituation and recovery dynamics. Our approach differs in modeling how both controller actions \textit{and} endogenous patient adherence decisions affect future engagement. We also model a common, time-varying engagement state rather than per-arm states. Methodologically, our  algorithm relies on optimism \citep{laiAsymptoticallyEfficientAdaptive1985,auerNearoptimalRegretBounds2008a}, adapting upper confidence bound (UCB) techniques deployed across linear \citep{abbasi-yadkoriImprovedAlgorithmsLinear2011} and logistic \citep{fauryImprovedOptimisticAlgorithms2020} bandit settings. Our proposed algorithm adapts the structure of LSVI-UCB in \textcite{jinProvablyEfficientReinforcement2023}, constructed for episodic linear MDPs, to the discounted, infinite horizon setting with hybrid linear-logistic system dynamics.

\subsection{Model-based control approaches}
A related research area applies control-system tools such as system identification and model predictive control (MPC) to behavioral interventions \citep{riveraUsingEngineeringControl2007,heklerAdvancingModelsTheories2016,elmistiriModelPredictiveControl2025a}. Closely related to our modeling focus on mHealth engagement are \citet{elmistiriSystemIdentificationUser2024} and \citet{delatorreDynamicBayesianNetwork2025}, which model time-varying engagement states with dynamics motivated by continuous time fluid processes. This line of research shares our model-based approach, but typically evaluates performance empirically in simulation rather than with formal regret analysis. Closer to our work is a line of model-based RL mHealth literature targeting provable guarantees, such as \citet{heModelBasedReinforcement2023a} for heparin dosing, \citet{liAdaptiveOptimizationApproach2026} for weight-loss, and \citet{pulickAdaptiveControlApproach2025b} for substance use disorders. Similar to our work, these fit patient-specific dynamics models, but their dynamics are deterministic with noise affecting an observation model, while in our formulation process noise affects the state dynamics directly. 

Our system identification result
follows recent work establishing finite-time identification and regret guarantees for specific control model classes \citep{tsiamisStatisticalLearningTheory2023}. We build on the block-martingale small ball machinery for finite-time LDS analysis from \citet{simchowitzLearningMixingSharp2018}, which has been extended to settings like the linear quadratic regulator \citep{deanSampleComplexityLinear2020} and LDS with non-linear policies \citep{liNonasymptoticSystemIdentification2023}. We extend this core machinery to our LDS setting with $1$-sparse inputs and state-dependent adherence under a flexible exploratory policy class.
        
\section{Engagement state dynamics and adherence model}\label{sec:model}
Our motivating setting is a DT where the decision-maker selects a daily treatment recommendation to support long-term patient health. Treatment recommendations themselves can impact a patient's future adherence (e.g., recommendation burden). Similarly, adherence itself can affect future patient adherence negatively (e.g., fatigue) or positively (e.g., increased self-efficacy). Our model captures these effects with a scalar, time-varying construct: the patient's engagement state, representing their capacity for adherence. We model patient adherence as stochastic decisions depending endogenously on the engagement state, with higher values yielding higher adherence probability. The model can express empirically observed behavior, such as burden or fatigue reducing patient adherence, via the patient's time-varying engagement state. This structure re-frames the decision-maker's problem from myopic decision-making (i.e., selecting actions with the largest immediate reward) to long-term planning around relevant engagement state dynamics (i.e., accounting for recommendation- and adherence-driven effects on a patient's future engagement state). A planning algorithm can strategically ask more or less of a patient over time to maintain engagement and accrue greater long-run rewards than myopic approaches. To reflect patient heterogeneity, each patient has an unknown set of dynamics and adherence parameters; the DT decision-maker does not know these parameters \emph{a priori} and must learn them online through repeated interaction with the patient.
\subsection{Model definition}
We model the patient's time-varying engagement state and adherence behavior as follows: 
\begin{align}
    x_{t+1} &= ax_t+b^\top u_t +c^\top d_t+w_t,\label{eq:dynamics_adh}
        &&d^i_t|x_t,u^i_t\sim\text{Ber}\bigl(\hat p_i(x_t,u^i_t)\bigr),\quad \hat p_i(x,u^i) = u^i \sigma(x+\mu_i)
\end{align}
Let there be $M$ available treatments and at most one treatment can be recommended per day, indexed by $t=1,\dots,T$. The patient's engagement state is given by $x_t\in \R$. $u_t \in \mac U:=\{v\in \{0,1\}^M:\Vert v\Vert_0\leq 1\}$ is the $1$-sparse treatment recommendation, where $u^i_t$ is $1$ if recommending treatment $i$ and $0$ otherwise. $|\mac U|=M+1$ as the decision maker may choose the null action $u_t=\ve{0}_M$. $d_t \in \mac U$ is the patient's $1$-sparse adherence decision. $d_t^i$ is $1$ if the patient adheres to the recommendation and $0$ otherwise. $\sigma$ is the standard sigmoid. The initial state, $x_1$, has the same distribution as the noise $w$.
\begin{assumption}{\textbf{(Full observation)}}\label{asm:observed_quantities}
$x_t$, $u_t$, and $d_t$ are fully observed with no missingness.
\end{assumption}
This assumption comes from our DT setting, where it is common to gather self-reported quantities from patients through daily check-ins, called ecological momentary assessments (EMAs) \citep{shiffmanEcologicalMomentaryAssessment2008a}. Further, assuming treatments are hosted on the DT itself, the platform can reasonably monitor the patient's adherence decisions in response to the recommended treatment. 

Parameter $a\in \Theta_a :=[0,\bar a]$ captures the state's persistence over time for known $\bar a \in (0,1)$. $b \in \Theta_b :=\{v \in \R^M:\Vert v \Vert_\infty \leq  \bar b\}$ captures the effect of recommendation on the next state, for known $\bar b \in \R_+$. $c \in \Theta_c:=\{v \in \R^M: \Vert v \Vert_\infty \leq \bar c\}$ captures the effect of adherence on the next state, for known $\bar c \in \R_+$. $\mu\in \Phi:=[-\bar \mu, \bar \mu]^M$ contains the treatment-specific shift parameters in the adherence model for known $\bar \mu \in \R_+$. Collectively, $\theta = [a \ b^\top \ c^\top]^\top$ and $\mu = [\mu_1,\dots,\mu_M]^\top$ are fixed but unknown and must be learned online. Last, we assume the following properties for the process noise $w$:

\begin{assumption}{\textbf{(Noise)}}\label{asm:noise_properties}
$\{w_t\}_{t=1}^T$ are independent, identically distributed, zero-symmetric, $\sigma_s^2$-subgaussian, bounded as $\vert w_t\vert \leq \bar w$ $(\bar w >0)$, have a log-concave density, and known variance $\sigma_w^2>0$.
\end{assumption}
These assumptions are relatively mild and are satisfied by zero-mean, symmetrically truncated Gaussian or uniform noise. Similar analysis can be used without assuming boundedness at the expense of heavier technical machinery. Last, we note two important properties for our system. The construction of the system ensures $x_t$ is bounded, and this boundedness sets upper and lower bounds for patient adherence. Both properties play an important role in our downstream analysis.
\begin{proposition}{\textbf{(Bounded state)}}\label{prop:absolute_bounded_x}
Under Assumptions~\ref{asm:observed_quantities}-\ref{asm:noise_properties}, $x_t \in \mac X:=[-C_x, C_x]$ for $t\geq 1$, with $C_x:=\frac{\bar b +\bar c +\bar w}{1-\bar a}$. Further, $\Vert [x_t \ u_t^\top \ d_t^\top]^\top\Vert^2_2 \leq C_x^2+2$.
\end{proposition}
The proof (Appendix~\ref{app:gen_model}) follows from strictly stable dynamics and bounded inputs. 
\begin{remark}{\textbf{(Bounded probabilities)}}\label{rem:adh_prob_bounded}
By Proposition~\ref{prop:absolute_bounded_x} and $\mu\in[-\bar \mu,\bar \mu]^M$, the probability of adherence (when $u^i=1$) is bounded below by $\underline p := \sigma(-C_x-\bar \mu)$ and above by $\overline p:= \sigma(C_x+\bar \mu)$.
\end{remark}

\section{System identification}\label{sec:identification}
An important precondition for data-driven DT algorithm design is that the patient model can be learned from observed trajectory data. We characterize high-probability guarantees for learning $\theta$ and $\mu$ that hold for finite samples, bounding the rate at which estimation error decays with respect to trajectory length, $T$, the number of treatments, $M$, and other system-specific parameters. Extending these types of finite-time guarantees to a variety of model classes has become an active focus of the system identification literature \citep{tsiamisStatisticalLearningTheory2023}. Our model departs from existing LDS work due to the presence of $1$-sparse inputs and state-dependent adherence decisions. Our identification result, given as Theorem~\ref{thm:system_sample_complexity}, demonstrates that we can recover the standard $\tilde{\mac O}(\nicefrac{1}{\sqrt{T}})$ learning rate for LDS systems under a flexible exploratory policy class that we define. This complements the Section~\ref{sec:policy_optimization} analysis of the separate \textsc{UCB-BOLD} algorithm, together establishing both identification guarantees for the patient model, and algorithmic regret guarantees for the proposed SDM setting.

\subsection{Estimators and exploratory policy class}
We estimate the dynamics parameters, $\theta$, with ordinary least squares (OLS), regressing the next state $x_{t+1}$ on the feature vector $z_t:=[x_t \ u_t^\top \ d_t^\top]^\top$: $\hat \theta_T := \arg \min_{\theta \in \R^{2M+1}} \nicefrac{1}{2}\sum_{t=1}^T (x_{t+1} - \theta^\top z_t)^2$. We estimate the adherence parameters, $\mu_i$, by maximum likelihood estimation (MLE): $\hat\mu_{i,T}:= \arg \max_{\mu_i \in [-\bar \mu, \bar \mu]}\sum_{t=1}^T \mathbf{1}_{\{u_t=e_i\}}\bigl[d^i_t\log(\sigma(x_t+\mu_i))+(1-d^i_t)\log(1-\sigma(x_t+\mu_i))\bigr]$. We assume that the decision-maker acts per an \textit{$(r,k)$-exploratory policy}. This policy class adapts continuous white noise excitation commonly used in system identification to our discrete action set.
\begin{definition}{\textbf{$(r,k)$-Exploratory Policy}.}
Let $r\in(0,1]$ and $k \in \mathbb N$. A policy is an $(r,k)$-exploratory policy if for every $k$-block in a trajectory (i.e., $t \in \{j+1,\dots,j+k\}$ for any $j\geq0$), the policy almost surely takes at least $\lceil r \cdot k\rceil$ actions drawn uniformly from $\mac U$ independent of $\mac F_{t-1}$.
\end{definition}

This definition permits the decision-maker to choose which steps are exploratory based on history (including the current state $x_t$), so long as the block-level exploration condition is satisfied. The limiting cases of $r=1$ or $k=1$ imply a purely random policy, but this framework can flexibly accommodate a randomized trial where the exploration level is fixed (i.e., to roughly $r$) and a separate policy is used in non-exploratory time steps. Further, our identification result does not rely on the policy for non-exploratory actions, which can be adversarial or non-stationary.

\subsection{Identification strategy}
We refer to the patient's true parameters as $\theta^*$ and $\mu^*$. Our identification result combines separate analysis of the OLS and MLE estimators. The OLS-portion of our system identification result builds on a result by \textcite{simchowitzLearningMixingSharp2018}, restated as Theorem~\ref{thm:simch_sample_complexity}. The analysis of the MLE estimators for the sigmoid shifts relies on standard concentration arguments for generalized linear models (GLM) \citep{filippiParametricBanditsGeneralized2010} (see Appendix~\ref{app:GLM_identification}). 
\begin{theorem}{(\textcite{simchowitzLearningMixingSharp2018} Theorem 2.4)}\label{thm:simch_sample_complexity}
    Fix $\delta\in(0,1)$, $T \in \mathbb N$, and $0\prec \Gamma_{sb} \preceq \bar \Gamma$. Consider a random process $(Z_t,Y_t)_{t\geq 1} \in (\R^d\times \R^n)$ and a filtration $\{\mathcal F_t\}_{t\geq 1}$. Suppose the following conditions hold: (a) $Y_t=\ma{\theta}^*Z_t + \eta_t$, where $\eta_t|\mac F_{t}$ is $\sigma^2$-sub-Gaussian and mean zero, (b) $(Z_t)_{t\geq 1}$ is an $\{\mathcal F_t\}_{t\geq 1}$-adapted random process satisfying the $(k,\Gamma_{sb},p)$-small ball condition, and (c) $\pr[\sum_{t=1}^TZ_tZ_t^\top \not \preceq T \bar \Gamma] \leq \delta$.
    Then if $T\geq \nicefrac{10k}{p^2}(\log(\nicefrac{1}{\delta})+2d\log(\nicefrac{10}{p})+\log \det(\bar \Gamma \Gamma_{sb}^{-1}) )$,
    we have $\pr[\Vert \hat{\ma{\theta}}(T) - \ma{\theta}^* \Vert_{\text{op}}> 
    \frac{90\sigma}{p} \sqrt{\tfrac{n+d\log(\nicefrac{10}{p})+\log\det\bar\Gamma\Gamma_{sb}^{-1}+\log(\nicefrac{1}{\delta})}{T\lambda_{\text{min}}(\Gamma_{sb})}}] \leq 3\delta$. 
\end{theorem}

This result provides flexible machinery for finite-time linear system analysis. Intuitively, precondition (b) ensures that the system is sufficiently excited while precondition (c) ensures the empirical covariance is bounded such that self-normalized martingale arguments can be used. Precondition (b) relies on the following definition from \textcite{simchowitzLearningMixingSharp2018}:
\begin{definition}{\textbf{Block Martingale Small Ball (BMSB)}.}\label{def:bmsb}
    Let $(Z_t)_{t\geq 1}$ be an $\{\mac F_t\}_{t\geq 1}$-adapted random process taking values in $\R^d$. $(Z_t)_{t\geq 1}$ satisfies the $(k,\Gamma_{sb},p)$-BMSB condition for a positive integer $k$, a positive definite matrix $\Gamma_{sb}$, and $0\leq p \leq 1$, if for any fixed $\lambda \in \R^d$ with $\Vert \lambda \Vert_2=1$, the process $(Z_t)_{t\geq 1}$ satisfies $\frac{1}{k}\sum_{i=1}^k\pr (\vert \lambda^\top Z_{j+i}\vert \geq \sqrt{\lambda^\top \Gamma_{sb}\lambda} \vert \mac F_j)\geq p$ almost surely for any $j\geq 0$.
\end{definition}

Our primary contribution in this section is adapting the BMSB machinery to our system. We provide our system-specific preconditions (b) and (c) here.
\begin{proposition}{\textbf{Informal precondition (b)}.}\label{prop:bmsb}
    Let Assumptions~\ref{asm:observed_quantities}-\ref{asm:noise_properties} hold and actions come from an $(r,k)$-exploratory policy. Then for an appropriate $\tau \in (0,1)$ (see Lemma~\ref{lem:tau} in Appendix~\ref{app:ols_identification}), our system satisfies the $(k,\frac{1-\tau^2}{8M} I_{2M+1},p)$ BMSB condition for $p$ of order $\frac{r}{M}$.
\end{proposition}
The proof partitions the unit sphere into three cases, showing that the excitation condition can always be satisfied by randomness from one of the process noise, random action selection, and random adherence outcomes. The formal version is given as Proposition~\ref{prop:bmsb_formal} in Appendix~\ref{app:ols_identification}.

\begin{proposition}{\textbf{Precondition (c)}}\label{prop:ub-cond-formal}
Let Assumptions~\ref{asm:observed_quantities}-\ref{asm:noise_properties} hold. Then for $\bar \Gamma:= (C_x^2+2)\ma{I}_{2M+1}$, $\pr[\sum_{t=1}^Tz_tz_t^\top \not \preceq T \bar \Gamma] =0$, i.e. $\sum_{t=1}^T z_tz_t^\top \preceq T \bar \Gamma$ almost surely.
\end{proposition}
\textcite{simchowitzLearningMixingSharp2018} derive a high probability bound using $\bar \Gamma$ from the controllability Gramian, as this matrix is available in closed form for their system. This type of analysis is not possible for our system due to the sigmoidal adherence mechanism, but our system's bounded covariates permit us to construct $\bar \Gamma$ such that our bound holds deterministically.

\subsection{Identification result}
We synthesize the $\theta$ and $\mu$ concentration arguments into an identification result for our system.
\begin{theorem}{}\label{thm:system_sample_complexity}
    Let Assumptions~\ref{asm:observed_quantities}-\ref{asm:noise_properties} hold and actions come from an $(r,k)$-exploratory policy. Fix $\delta\in(0,1)$, $\tau\in (0,1)$ per Lemma~\ref{lem:tau}, and $T \in \mathbb N$. Let $p = \min\{p_1(\tau),\frac{r\underline p}{M+1}\}$ (see  Appendix~\ref{app:ols_identification}). If $T\geq \max\{T_{\text{OLS}},T_{\text{MLE}},2k\}$ for $T_{\text{OLS}} = \tfrac{10k}{p^2}[\log(\nicefrac{6}{\delta}) +(2M+1)(2\log(\nicefrac{10}{p})+\log (\frac{8M(C_x^2+2)}{1-\tau^2}))]$, $T_{\text{MLE}}=k + \nicefrac{8}{r}(M+1)\log(\nicefrac{4M}{\delta})$, then with probability at least $1-\delta$:
    \begin{align}
        &\Vert \hat \theta(T) - \theta^* \Vert_{2}\leq 
        \tfrac{90\sigma_s}{p} \sqrt{\tfrac{8M}{T(1-\tau^2)}[1+\log(\nicefrac{6}{\delta})+(2M+1)\log(\tfrac{80M(C_x^2+2)}{p(1-\tau^2)})]},\\
        \text{and} \qquad
        &\max\nolimits_{i \in \{1,\dots,M\}} \vert \hat \mu_{i,T} - \mu_i^*\vert \leq \tfrac{4(M+1)\sqrt{2\log(\nicefrac{8M}{\delta})}}{\underline p^2 r\sqrt{T}}.
    \end{align}
\end{theorem}
\begin{remark}{}
Under the conditions of Theorem~\ref{thm:system_sample_complexity}, the burn-in time for OLS estimation dominates that of MLE estimation (since $p$ is of order $\tfrac{r}{M}$) and the $T\geq 2k$ condition is non-binding for typical values of $M,r,\delta$. In this case, the burn-in requirement is $T= \tilde \Omega (\tfrac{kM^3}{r^2} )$, the OLS estimation error $\Vert \hat \theta(T) - \theta^* \Vert_{2} = \tilde{\mac O}( \tfrac{M^2}{r\sqrt{T}})$, and the MLE estimation error $\max_{i\in \{1,\dots,M\}}\vert \hat \mu_{i,T}-\mu^*_i\vert = \tilde {\mac O} ( \tfrac{M}{r\sqrt{T}})$.
\end{remark}

The proof for Theorem~\ref{thm:system_sample_complexity} can be found in Appendix~\ref{app:ec_sys_id_result}. Thus we demonstrate that the $\theta$ and $\mu$ parameters for the patient model proposed in Section~\ref{sec:model} are identified at the standard parametric $\tilde {\mac O}\left(\nicefrac{1}{\sqrt{T}}\right)$ rate under $(r,k)$-policies. We emphasize the flexibility of this policy class, allowing the decision-maker to enact arbitrary policies (i.e., non-stationary or adversarial) in non-exploratory steps. Further, the result only depends on the model and policy parameters, not the properties of any specific realized trajectory. The primary identification bottleneck is the estimation of the dynamics parameters, $\tilde{\mac O}\left(\frac{M^2}{r\sqrt{T}}\right)$, as compared to the adherence shift parameters, $\tilde{\mac O}\left(\frac{M}{r\sqrt{T}}\right)$. This makes sense, as we use the bounded state space to lower bound per-sample Fisher information for $\mu_i$ in each adherence Bernoulli trial; this lower bound means that we do not need to account for the autocorrelation of $x_t$ and thus reduces $\mu$ estimation to Bernoulli concentration for each of the $M$ treatments. In contrast, we cannot sidestep the autocorrelation of the feature vector, which we handle with the heavier BMSB machinery.

The BMSB machinery provides a natural framework for analyzing our system; while the state-dependent adherence mechanism complicates the distribution of the covariates themselves, the regression problem is still linear in $\theta$ (i.e., precisely the motivating setting of the BMSB machinery). We show that this machinery can be adapted to handle the unique departures of our system from a standard LDS. In particular, we show how the exploration induced by the $(r,k)$-policy ensures that the BMSB condition is satisfied and we exploit our bounded state space to bound the empirical covariance. One consequence of the combination of the $(r,k)$-policy class, $1$-sparse inputs, and this BMSB machinery is that the resulting rate $\tilde{\mac O}\left(\frac{M^2}{r\sqrt{T}}\right)$ is likely loose in $M$ (i.e., the \textcite{simchowitzLearningMixingSharp2018} LDS result is $\tilde{\mac O}(\sqrt{d})$ in the covariate dimension $d$, but assumes continuous noise that excites all directions simultaneously). We suspect that an analysis that separates learning of the autoregressive parameter $a$ from the bandit-like parameters $b$ and $c$ may yield tighter rates, but this separation is non-trivial, as the state-dependent adherence mechanism couples analysis of the parameters. Such analysis is a natural direction for future work.

\section{Policy optimization}\label{sec:policy_optimization}
Section~\ref{sec:identification} demonstrated that our patient model is identifiable under the $(r,k)$-exploratory policy class. This type of forced exploration is needed to identify all patient parameters to a given accuracy, but comes at the cost of linear regret rates. While this makes sense in some settings, for instance randomized trials, a DT decision-maker is typically motivated to learn patient parameters only to an accuracy sufficient to make good recommendation decisions. To this end, we adopt the principle of optimism \citep{laiAsymptoticallyEfficientAdaptive1985} underpinning upper confidence bound (UCB) approaches. In this paradigm, algorithms adapt their exploration based on current parameter uncertainty, only exploring where uncertainty is large enough that a better policy might still exist. We propose an optimism-based algorithm, \textsc{UCB-BOLD}, that extends UCB approaches from bandit settings to the RL setting faced by DT decision-makers, and prove that it achieves sublinear regret in this context.

\subsection{Problem formulation}
In our DT setting, a decision-maker is ultimately motivated to achieve long-term patient health. Patient adherence is needed to provide the clinical benefit of treatment, but this is conditional on the patient staying engaged with the DT and not dropping out. To capture these separate motivations, we make the following assumption about the reward structure for our RL setting.
\begin{assumption}{\textbf{(Reward structure)}}\label{asm:rewards}
    Time invariant rewards $r(x,u,d) := -\beta \sigma(\beta_0-x)+\sum_{i=1}^M\rho_id_i$ for known $\rho_i\in[0,\bar \rho]$, $\beta \in \R_+$, and $\beta_0 \in \R$. Rewards are discounted by known $\gamma \in (0,1)$.
\end{assumption}
The decision maker receives a fixed, known reward, $\rho_i$, when the patient adheres to the $i^{th}$ treatment recommendation, implying clinical guidance is available to assign values to the treatment set. Since we do not explicitly model the possibility of a patient being completely removed from the DT, the $\beta$-scaled sigmoid enforces a tunable penalty for low state values. A decision-maker can tune the penalty and the discount factor $\gamma$ to match their specific setting. The decision-maker must identify a policy $\pi$ to maximize their discounted, infinite horizon return: $J^*(x) =\max_{\pi\in \Pi} \ex[\sum_{t=1}^\infty \gamma^{t-1}r(x_t,u_t,d_t)]$, where the system structure and constraints are defined as per Section~\ref{sec:model} and $J^*(x)$ is the optimal value function. While we focus our analysis on sigmoidal penalties and rewards, our approach readily extends to any bounded and Lipschitz continuous stage cost functions.

\subsection{Optimism-based algorithm}
Here we present our epoch-based algorithm for optimizing treatment recommendation decisions: Upper Confidence Bound with Binary Outcomes and Linear Dynamics (\textsc{UCB-BOLD}), shown as Algorithm~\ref{alg:ucb_epoch}. At the start of each epoch, indexed by $k$, \textsc{UCB-BOLD} uses observations up to the current time step to calculate parameter estimates and confidence sets, i.e., $\hat \theta^{(k)}=\hat \theta_t$ and $\hat \mu^{(k)}=\hat \mu_t$ (with the confidence sets calculated similarly). These quantities define an optimistic surrogate system, identical to the true system except that it is parameterized by $\hat \theta^{(k)}$ and $\hat \mu^{(k)}$ and it uses an augmented, optimistic reward that encodes current uncertainty in the parameters via the confidence sets. The algorithm solves for the optimal value function of the surrogate system using value iteration (VI) on a discretized grid, and implements the resulting stationary policy for the remainder of the epoch. A new epoch begins when a triggering condition is met, based on either the determinant of the regularized second moment matrix or the number of times each action has been taken. When a new epoch begins, the estimators and confidence sets are re-calculated and the process repeats. While the within-epoch policy is deterministic, the algorithm is prompted to explore strategically based on parameter uncertainty via the surrogate system's augmented reward. Unlike the forced exploration of an $(r,k)$-policy, this optimism-driven exploration only reduces parameter uncertainty to a point where actions can be determined to be suboptimal, leading to a sublinear regret rate. 
\begin{algorithm}[h]
\small
\caption{Upper Confidence Bound with Binary Outcomes, Linear Dynamics (UCB-BOLD)}
\label{alg:ucb_epoch}
\begin{algorithmic}[1]
\Require Confidence $\delta \in (0,1)$, discount factor $\gamma \in (0,1)$, regularizers $\lambda_1,\lambda_2 >0$, constants $C_d,C_N>0$.
\State Set epoch count $k=1$ and initialize action counts $N_{i}=0$ ($i=1,\dots,M$).
\State Set $\bar V^{(k)} = \lambda_1 I_{2M+1}$, $\hat\theta^{(k)} = \hat \theta_1$, $\hat \mu^{(k)} = \hat \mu_1$, $\alpha^{\theta, (k)}=\alpha^\theta_1(\nicefrac{\delta}{2})$, $\alpha_{i}^{\mu,(k)}=\alpha^\mu_{i,1}(\nicefrac{\delta}{2})$, action counts $N_{i}^{(k)}=N_i\; \forall i$.
\State Calculate optimistic $\tilde J_{\delta}^{(k)}(x)$ using VI with $\tilde {\mac T}_{\delta}^{(k)}$ and calculate stationary optimal policy $\pi^{(k)}$ per $\tilde J_{\delta}^{(k)}$
\For{t = 1,...,T}
    \State Observe $d_{t-1}$, $x_t$, update $\bar V_t$
    \If {$\det(\bar V_t)> (1+C_d)\det(\bar V^{(k)})$ OR $N_{i} > (1+C_N) N_{i}^{(k)}$ for any $i=1,\dots,M$}
        \State Set $k = k +1$, $\bar V^{(k)}=\bar V_{t}$, $\hat\theta^{(k)} = \hat \theta_{t}$, $\hat \mu^{(k)}=\hat\mu_{t}$, $\alpha^{\theta,(k)} = \alpha^\theta_{t}(\nicefrac{\delta}{2})$, $\alpha^{\mu,(k)}_{i} = \alpha^\mu_{i,t}(\nicefrac{\delta}{2})$, $N_{i}^{(k)}=N_{i}$ for $i=1,\dots,M$
        \State Calculate optimistic $\tilde J_{\delta}^{(k)}(x)$ using VI with $\tilde {\mac T}_{\delta}^{(k)}$ and calculate stationary optimal policy $\pi^{(k)}$ per $\tilde J_{\delta}^{(k)}$
    \EndIf
    \State Take action $u_t=\pi^{(k)}(x_t)$
    \State Set $N_{i}=N_{i}+\mathbf{1}_{\{u_t=e_i\}}$, $i=1,\dots, M$
\EndFor
\end{algorithmic}
\end{algorithm}
\vspace{-10pt}

\subsection{\textsc{UCB-BOLD} components}
As in Section~\ref{sec:identification}, we use separate estimation strategies for the dynamics ($\theta$) and adherence ($\mu$) parameters, as well as separate techniques for constructing their respective confidence sets. These elements allow us to define the optimistic surrogate system necessary for our regret analysis.

\subsubsection{Estimator and confidence set construction}
At each time step, we use $\ell_2$-regularized regression to compute a regularized least squares estimator for $\theta$, i.e., for some $\lambda_{1}>0$: 
\begin{align}
    \hat{\theta}^{RLS}_t := \argmin\nolimits_{\theta \in \R^{2M+1}}\tfrac{1}{2}[ \textstyle\sum\nolimits_{s=1}^{t-1} (x_{s+1} - \theta^\top z_s)^2+\lambda_{1} \Vert \theta\Vert_2^2]
\end{align}
$\hat \theta_t$ is the weighted projection of $\hat \theta_t^{RLS}$ onto the known parameter space $\Theta=\Theta_a\times\Theta_b\times\Theta_c$: $\hat \theta_t := \arg\min_{\theta \in \Theta} \Vert \theta - \hat \theta^{RLS}_t \Vert_{\bar V_t}^2$, where $\bar{V}_{t} := \sum_{s=1}^{t-1} z_sz_s^\top + \lambda_1 I_{2M+1}$ is the regularized second moment matrix available at the beginning of time step $t$. As $\Theta$ is a convex set, we can compute $\hat\theta_t$ using the quadratic program: $\min_{\theta\in\Theta}\nicefrac{1}{2}\theta^\top \bar V_t \theta -(\bar V_t\hat \theta_t^{RLS})^\top \theta$. We form confidence sets for $\theta$ by Proposition~\ref{prop:theta_confidence_set}.
\begin{proposition}{\textbf{($\theta$ confidence set)}}\label{prop:theta_confidence_set}
Let Assumptions~\ref{asm:observed_quantities}-\ref{asm:noise_properties} hold. Let $\delta \in (0,1)$.
Then with probability at least $1-\delta$, the event $\mac E_\theta(\delta) := \{\theta^* \in \mac C^\theta_t(\delta), \; t\geq 1\}$ holds, where: 
$\label{eq:c_t}
        \mac C^\theta_t(\delta) = \{ \theta \in \R^{2M+1} : \Vert\hat{\theta}_t - \theta\Vert_{\overline{V}_t} \leq \alpha^\theta_t(\delta)   \}$,
    with $\alpha^\theta_t(\delta) := \sigma_s\sqrt{(2M+1)\log(\frac{1}{\delta} +\frac{t(C_x^2 +2)}{\delta\lambda_1} )} +\sqrt{\lambda_1(\bar a^2 + M \bar b^2 + M\bar c^2)}$.
\end{proposition}
The proof is in Appendix~\ref{app:ec_confidence_sets}. The dynamics portion of our model maps to a result for the linear bandit setting given by \textcite{abbasi-yadkoriImprovedAlgorithmsLinear2011}. The proof shows that our system satisfies the conditions of their result, and that the confidence sets can be re-centered from $\hat \theta_t^{RLS}$ to $\hat \theta_t$. To estimate each sigmoid shift $\mu_i$ we use regularized MLE at each time step, i.e., for some $\lambda_{2} >0$:
\begin{align}
    \hat\mu_{i,t}^{MLE} = \argmax\nolimits_{\mu_i \in \R} \textstyle\sum\nolimits_{s=1}^{t-1}\mathbf{1}_{\{u_s=e_i\}}\left [d^i_s\log(\sigma(x_s+\mu_i))+(1-d^i_s)\log(1-\sigma(x_s+\mu_i))\right] -\tfrac{\lambda_{2}}{2}\mu_i^2\label{eq:regloglik}
\end{align}
Our final estimator $\hat\mu_{i,t}$ is the projection of $\hat \mu_{i,t}^{MLE}$ onto the known parameter space $[-\bar \mu, \bar \mu]$. We form confidence sets for $\mu$ per Proposition~\ref{prop:mu_confidence_set}.
\begin{proposition}{\textbf{($\mu$ confidence set)}}\label{prop:mu_confidence_set}
Let Assumptions~\ref{asm:observed_quantities}-\ref{asm:noise_properties} hold. Let $\delta\in(0,1)$. Then with probability at least $1-\delta$ the event $\mac E_\mu(\delta):=\{\mu^* \in \mac C^{\mu}_t(\delta), t\geq 1\}$ holds where:
$\mac C^\mu_t(\delta) = \{\mu\in \R^M: \left \vert \hat \mu_{i,t} - \mu_i \right \vert \leq \alpha^\mu_{i,t}(\delta),\; i=1,\dots,M \}$,
for $\alpha^\mu_{i,t}(\delta) := \frac{e^{3\bar \mu}}{\sqrt{H_t(\hat \mu_{i,t})}}\bigl[\frac{\sqrt{\lambda_2}}{2}+\frac{2}{\sqrt{\lambda_2}}\bigl(\bar \mu +\log \bigl(\frac{2M\sqrt{H_t(\hat \mu_{i,t})}}{\delta \sqrt{\lambda_2}} \bigr)\bigr)\bigr]+\frac{e^{2\bar \mu}\lambda_2 \bar \mu}{H_t(\hat \mu_{i,t})}$ and $H_t(\hat \mu_{i,t}) := \lambda_2 + \sum\nolimits_{s=1}^{t-1}\mathbf{1}_{\{u_s=e_i\}}\sigma'(x_s+\hat \mu_{i,t})$.
\end{proposition}
The proof is provided in Appendix~\ref{app:ec_confidence_sets}. We build on a result from \textcite{fauryImprovedOptimisticAlgorithms2020}, which considers the adjacent case of estimating scale parameters in the GLM bandit setting. We show that their result (see Theorem~\ref{thm:faury_tail}) can be adapted to form confidence sets for our shift parameters $\mu_i$.

\subsubsection{Optimistic surrogate system} The surrogate system is structured identically to the true system, except that it is parameterized by the current estimates $\hat \theta$ and $\hat \mu$ and adds an exploration bonus \citep{azarMinimaxRegretBounds2017,jinProvablyEfficientReinforcement2023} to the reward. This bonus is constructed specifically so that the surrogate system's optimal value function overestimates $J^*$ and is given by: 
\begin{align}
b_{t,\delta}(x,u):= (\bar \rho + \gamma L_1\bar c)\bigl[\textstyle\sum\nolimits_{i=1}^Mu^i\kappa(x,\hat \mu_{i,t})\alpha_{i,t}^\mu(\nicefrac{\delta}{2})\bigr] + \gamma L_1 \alpha^\theta_t(\nicefrac{\delta}{2})\ex_{\hat \mu_t}\big[\Vert z \Vert_{\bar V_t^{-1}}\vert x,u\big] \label{eq:bonus}
\end{align}
$L_1$ is a Lipschitz constant (see Lemma~\ref{lem:j*_xlip}), $\kappa(x,\hat \mu_{i,t}):=\min(\nicefrac{1}{4},e^{2\bar \mu}\sigma'(x+\hat\mu_{i,t}))$ is a bound on the derivative of the sigmoid, and the $\alpha^\theta_t$ and $\alpha^\mu_{i,t}$ terms come from the confidence sets in Proposition~\ref{prop:theta_confidence_set} and Proposition~\ref{prop:mu_confidence_set}. The expectation is over the adherence outcome. Using an additive bonus simplifies VI as parameter uncertainty is captured in the reward instead of the dynamics. The bonus term encourages exploration by rewarding actions that reduce parameter uncertainty (i.e., high parameter uncertainty manifests as a large bonus value). This bonus allows us to define a modified Bellman operator for the optimistic surrogate system, $\tilde{\mac T}_{t,\delta}$, parameterized by $\hat \theta_t$ and $\hat \mu_t$ for $\delta\in(0,1)$. Our analysis also uses the typical Bellman operator, ${\mac T}$, and the policy Bellman operator, $\mac T^{\pi}$, both parameterized by $\theta^*$ and $\mu^*$, defined as follows:
\begin{align}
\tilde {\mac T}_{t,\delta} J(x) &:= \max_{u\in \mac U} \big\{\ex_{\hat\theta_t,\hat\mu_t}\big [r(x,u,d)+b_{t,\delta}(x,u)+\gamma J(f(x,u,d,w))\big]\big\}\\
\mac T J(x)&:= \max_{u\in \mac U}\bigl\{\ex_{\theta^*,\mu^*} \bigl [r(x,u,d)+\gamma J(f(x,u,d,w))\bigr]\bigr\}\\
\mac T^{\pi} J(x)&:= \ex_{\theta^*,\mu^*} \big [r(x,\pi(x),d)+\gamma J(f(x,\pi(x),d,w))\big]
\end{align}
\begin{lemma}{}\label{lem:unique_opt}
$\mac T$, $\mac T^{\pi}$, and $\tilde{\mac T}_{t,\delta}$ are monotone and converge uniformly to unique fixed points $J^*(x)$, $ J^\pi(x)$, and $\tilde J_{t,\delta}(x)$ (with $J^*(x)$, $\tilde J_{t,\delta}(x)$ optimal). $J^*(x)$ and $ J^\pi(x)$ are bounded by $C_J := \frac{\beta + \bar \rho}{1-\gamma}$. $\tilde J_{t,\delta}(x)$ is bounded for $t\leq T$ by $C_{\tilde J,T,\delta} := \frac{\beta +\bar \rho +C_b(T,\delta)}{1-\gamma}$, where $C_b(T,\delta)$ is the bonus bound (Lemma~\ref{lem:bounded_bonus}). 
\end{lemma}
These properties follow from the bounded and discounted setting (proof in Appendix~\ref{app:ec_value_function}), and allow us to show the following optimism result ($T+1$ is needed for downstream regret analysis).
\begin{lemma}{\textbf{(Valid optimism)}}\label{lem:valid_optimism}
Let $\mac E_\theta\left(\nicefrac{\delta}{2}\right)$ and $\mac E_\mu\left(\nicefrac{\delta}{2}\right)$ hold (Propositions~\ref{prop:theta_confidence_set}-\ref{prop:mu_confidence_set}) for $\delta\in(0,1)$. Then $\tilde J_{t,\delta}(x)\geq J^*(x)$ for any $x\in \mac X$, and $t\in \{1,\dots,T+1\}$.
\end{lemma}
The proof is provided in Appendix~\ref{app:ec_value_function}. Broadly, it shows that a single application of $\tilde {\mac T}_{t,\delta}$ to $J^*$ is optimistic (i.e., $\tilde{\mac T}_{t,\delta}J^*(x) \geq J^*(x)$), then uses the monotonicity and uniform convergence properties of the operators to show that this implies $\tilde J_{t,\delta}(x)\geq J^*(x)$. The core proof step is leveraging the $L_1$-Lipschitz continuity of $J^*(x)$ (see Lemma~\ref{lem:j*_xlip}) to express value function differences as estimation errors (i.e., $\vert \hat \theta_t - \theta^*\vert$), which are bounded by the confidence set Propositions~\ref{prop:theta_confidence_set}-\ref{prop:mu_confidence_set}. This implicitly extends to \textsc{UCB-BOLD}'s epoch quantities, which are a subset of the time-indexed quantities.

\subsection{Regret analysis}
We use a notion of regret from the discounted MDP literature \citep{heNearlyMinimaxOptimal2021} that measures the cumulative sub-optimality gap of the enacted policy: $R_T=\sum_{t=1}^T J^*(x_t)-J^{\pi_t}(x_t)$. This is the natural choice for regret in our bounded and discounted setting, as it captures convergence to an optimal policy even when the value functions themselves are bounded. From this definition, we construct our main regret result for the performance of \textsc{UCB-BOLD}.
\begin{theorem}{\textbf{(Main regret result)}}\label{thm:regret}
Let $\delta_{all}\in (0,1)$ and Assumptions~\ref{asm:observed_quantities}-\ref{asm:rewards} hold. Let $C_d,C_N>0$ be given by Algorithm~\ref{alg:ucb_epoch}. Then with probability at least $1-\delta_{all}$ it holds that:
\begin{align}
    R_T &\leq \tfrac{1}{4(1-\gamma)}[(2\bar \rho + \gamma \bar c (L_1+L_2))\nu_T(\nicefrac{\delta_{all}}{6})\sqrt{1+C_N}[\tfrac{M}{\sqrt{\lambda_2}}+2\sqrt{\tfrac{MT}{\underline q}} ] ]\nonumber\\
    &\qquad +\tfrac{\gamma}{1-\gamma}[ C_J+C_{\tilde J,T+1,\nicefrac{\delta_{all}}{3}}+(L_1+L_2)\alpha_T^\theta(\nicefrac{\delta_{all}}{6}) [C_1+C_2+C_3] ]\nonumber\\
    &\qquad + \tfrac{\gamma}{1-\gamma}[2(C_{\tilde J,T+1,\nicefrac{\delta_{all}}{3}}+C_J)(\sqrt{2T\log(\nicefrac{3}{\delta_{all}})}+K_d +K_N) ]
\end{align}
with $L_1=\frac{\beta+\bar \rho}{4(1-\gamma \bar a)}\left(1+\frac{2\gamma}{1-\gamma}\right)$, 
$L_2 = \frac{1}{1-\gamma\bar a}[\nicefrac{(\beta+\bar\rho)}{4}+(\nicefrac{e^{2\bar \mu}}{4})(\bar \rho + \gamma L_1\bar c)\overline \alpha^\mu(\nicefrac{\delta_{all}}{6})+\frac{5\gamma L_1\alpha^\theta_T(\nicefrac{\delta_{all}}{6})}{4\sqrt{\lambda_1}}+(\nicefrac{\gamma}{2})C_{\tilde J,T,\nicefrac{\delta_{all}}{3}}]$,
$\nu_T(\delta) = e^{3\bar \mu}[\frac{\sqrt{\lambda_2}}{2}+\frac{2\bar\mu}{\sqrt{\lambda_2}}+\frac{2}{\sqrt{\lambda_2}}\log (\frac{2M\sqrt{\lambda_2 + \nicefrac{T}{4}}}{\delta \sqrt{\lambda_2}} )+e^{-\bar \mu}\sqrt{\lambda_2} \bar \mu]$, 
$C_1=\frac{\nu_T(\nicefrac{\delta_{all}}{6})\sqrt{1+C_N}}{4\sqrt{\lambda_1}}[\nicefrac{M}{\sqrt{\lambda_2}}+2\sqrt{\nicefrac{MT}{\underline q}}]$, 
$C_2=\sqrt{2\log(\frac{3}{\delta_{all}})\frac{T}{\lambda_1}}$, 
$C_3=\sqrt{T(1+C_d)(2M+1)\log (1+\frac{T(C_x^2+2)}{(2M+1)\lambda_1} ) [2 +\frac{C_x^2+2}{\log(2)\lambda_1}]}$, 
$K_d = \frac{(2M+1)}{\log(1+C_d)}\log (1+\frac{T(C_x^2+2)}{(2M+1)\lambda_1} )$, 
$K_N= M(1+\frac{\log(T)}{\log(1+C_N)})$, $\underline q=\sigma'(C_x+\bar \mu)$, and 
$\overline \alpha^\mu(\nicefrac{\delta}{2})=\frac{e^{3\bar \mu}}{\sqrt{\lambda_2}}[\frac{\sqrt{\lambda_2}}{2}+\frac{2}{\sqrt{\lambda_2}}(\bar \mu+ \log (\frac{4M}{\delta \sqrt{\lambda_2}} ))]+\frac{2e^{3\bar \mu -1}}{\sqrt{\lambda_2}} +e^{2\bar \mu}\bar \mu$.
\end{theorem}

\begin{corollary}\label{cor:expected_regret}
Under the conditions of Theorem~\ref{thm:regret} and $\delta_{all}=\nicefrac{1}{T}$, $\ex[R_T] = \tilde{\mac O}\left(\frac{M^{\nicefrac{3}{2}}\sqrt{T}}{(1-\gamma)^3} \right)$.
\end{corollary}
Proofs of these results can be found in Appendix~\ref{app:ec_regret}. The proof of Theorem~\ref{thm:regret} relies on the Lipschitz continuity of the rewards, bonus, and the value functions $J^*$ and $\tilde J_{t,\delta}$. We add and subtract cross terms to generate pairs of quantities (e.g., differences in the value-to-go) under different $\theta$ and $\mu$ parameterizations. Similar to the approach used to demonstrate valid optimism for the surrogate system, we use Lipschitz continuity to express differences as estimation errors in either $\mu$ or $\theta$. Propositions~\ref{prop:theta_confidence_set} and \ref{prop:mu_confidence_set} then provide high probability bounds on these estimation errors. We bound the remaining terms either by standard concentration arguments for martingale difference sequences or using the epoch-based construction of the algorithm to show that the number of bounded terms grows sublinearly. We then union bound these events to obtain the final result.

Thus \textsc{UCB-BOLD} achieves the $\tilde{\mac O}(\sqrt{T})$ regret rate typical for related settings, such as discounted tabular MDPs \citep{heNearlyMinimaxOptimal2021} and linear mixture MDPs \citep{zhouProvablyEfficientReinforcement2021}, linear and logistic bandits \citep{abbasi-yadkoriImprovedAlgorithmsLinear2011,fauryImprovedOptimisticAlgorithms2020} and the average reward linear quadratic regulator \citep{simchowitzNaiveExplorationOptimal2020}. The leading order term in our regret bound comes from the $(L_1+L_2)\alpha_T^\theta(\nicefrac{\delta_{all}}{6}) \left[C_1+C_2+C_3\right]$ quantity; this term includes the $\frac{\gamma}{1-\gamma}$ term from the telescoping decomposition, the $\tilde{\mac O}\left(\frac{\sqrt{M}}{(1-\gamma)^2}\right)$ $L_2$-Lipschitz continuity of $\tilde J_{t,\delta}$, the $\tilde{\mac O}(\sqrt{M})$ confidence set width $\alpha^\theta_T$, and the $\tilde{\mac O}(\sqrt{MT})$ sum of $C_1$ and $C_3$. $C_1$ captures joint interactions between $\mu$ uncertainty and $\theta$-based exploration while $C_3$ captures typical dynamics exploration costs for learning $\theta$ in linear bandit or linear MDP settings. A regret lower bound in our model class remains open. For context, \textcite{zhouProvablyEfficientReinforcement2021} established the $\Omega\left(\frac{d\sqrt{T}}{(1-\gamma)^{1.5}} \right)$ rate for the related linear mixture MDP setting, which assumes a transition kernel that is linear in a $d$-dimensional feature set. Our model's noise and state-dependent adherence do not fit this modeling class, but their bound suggests that tighter analysis may be possible in our setting and serves as a motivating reference point for future work. 

\section{Case study and computational experiments}\label{sec:computational_experiments}
We evaluate \textsc{UCB-BOLD} against benchmark algorithms on a synthetic patient cohort, tasking the algorithms with recommending daily treatments from a small ($M=3$) treatment set. Our experiments are derived from an MRT where digital prompts were used to promote non-sedentary behavior in the workplace. We use Bayesian inference to fit a hierarchical model (HM) to the MRT data and generate synthetic patients. We evaluate algorithm performance on simulated patient trajectories, including extensive ablation studies to highlight algorithm performance across parameter regimes. We find that \textsc{UCB-BOLD} consistently outperforms the benchmarks across ablation conditions.

\subsection{Physical activity case study}
We rely on MRT study data from \textcite{zotero-item-2999}. This study followed 16 participants for up to 4 work weeks. Within each hour of the participant's 8-hour workday, participants were randomly assigned a message type: (1) no message, (2) a prompt to stand up, and (3) a prompt to move (i.e., to take a brief walk). The study set internal thresholds for standing and moving to convert measured physical activity to binary adherence outcomes for each prompt.

\subsection{Synthetic cohort generation}
This MRT largely maps onto the model described in Section~\ref{sec:model}. One exception is the absence of a self-reported engagement state; this is treated as a latent quantity during Bayesian HM fitting. The other key difference is the MRT's hourly, not daily, timescale. Modeling both within- and between-day effects was impractical for a dataset of this size, so we chose to treat each day as a standalone trajectory. Thus, we consider the following patient-level model: $ x_{t+1}^{i,j} = a^ix_t^{i,j}+{b^i}^\top u^{i,j}_t+{c^i}^\top d^{i,j}_t +w^j_t$, $x^{i,j}_1\sim w^j_0$, where $i$ indexes the participant, $j$ indexes the trajectory for that participant, and $t$ indexes the time step within the day. Adherence decisions $d^{i,j}_t|x^{i,j}_t,u^{i,j}_t$ follow the model in Section~\ref{sec:model}. The HM models patient parameters as samples from a population distribution, i.e., letting $m$ index the non-null action type, $b^i_m \sim \mac N(\bar b_m, \sigma_{b_m}^2)$. The same approach is used for $a$, $c$, and $\mu$, with a logit transformation to ensure $a\in(0,1)$. We emphasize that we do not use the MRT dataset to validate our particular model structure against alternatives; rather, we use the dataset as a principled way to generate a plausible patient parameter distribution for this model class.

We use \verb|Stan| \citep{carpenterStanProbabilisticProgramming2017a} to perform Markov chain Monte Carlo Bayesian inference on the HM and MRT dataset described above. We used weakly informative priors, resulting in a roughly uniform prior for the sampling distribution of $a$ and zero-mean Gaussian priors for the sampling distributions of $b$, $c$, and $\mu$. Figure~\ref{fig:priors_posteriors} shows the sampling distribution for the synthetic patients under the weakly informative priors and under the posterior of the MRT-fitted model.
\begin{figure}
\FIGURE
{\includegraphics[width=0.85\textwidth]{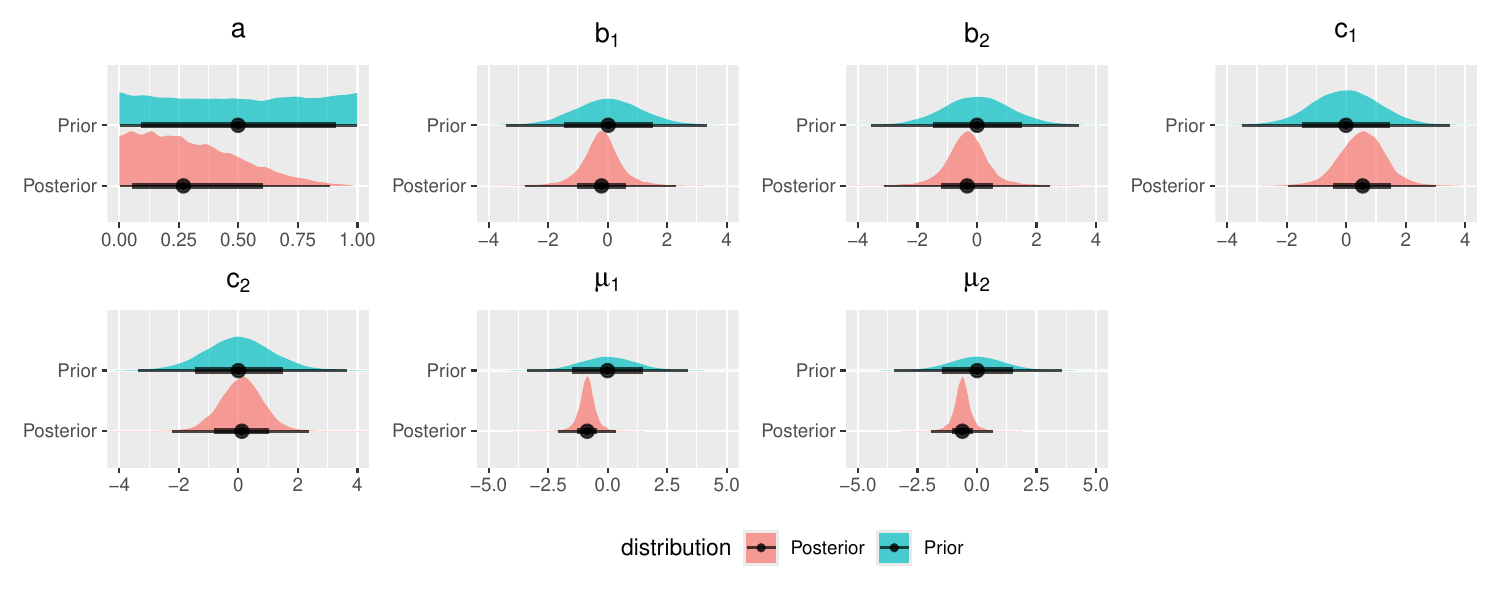}}
{Synthetic patient population sampling distributions \label{fig:priors_posteriors}}
{Plots display kernel density estimates, medians, and $80\%$ and $99\%$ credible intervals for each model parameter.}
\end{figure}
While there is still meaningful uncertainty in the posterior, its shape agrees with intuition. The uniform prior over $a$ becomes right skewed, suggesting low-to-moderate state persistence in this context. The zero-mean Gaussian priors for $b$ tighten and shift left (i.e., indicating burden), with the more demanding move request being more negative than the stand request. The zero-mean Gaussian priors for $c$ shrink and shift right, suggesting adherence can increase engagement capacity in this context (with standing having a stronger effect than moving). Both $\mu$ quantities tighten and shift left, capturing that baseline adherence in this context is less than $50\%$ for both prompt types.

\subsection{Experiment description}
All experiments simulated patient trajectories with once-daily treatment recommendations. We performed $25$ replications with different random seeds for each patient. While $x$ was a latent quantity for Bayesian inference using the MRT data, here it is fully-observed, matching Section~\ref{sec:model}. Using the joint parameter posterior described previously, we sampled $100$ synthetic patients for evaluation. The sampling process fixed most relevant patient parameters, i.e., $a^i$, $b^i_1$, $b^i_2$, $c^i_1$, $c^i_2$, $\mu^i_1$, and $\mu^i_2$. The process noise was zero-mean truncated Gaussian noise with original $\sigma_w^2=1$, truncated using $\bar w=2.5$. Given the importance of reducing dropout in DT interventions, we tested the impact of including an additional, motivating action. As a result, the experiments had $3$ non-null actions (two from the MRT plus the motivating action) and the null action. We set $\rho_3$ and $b_3$ for the motivating action to $0$, meaning it carries no treatment reward but also no cost or benefit associated with recommendation. We also fixed $\mu_3$ for this action to $0$, implying $50\%$ adherence at $x=0$. Ablation studies explored the effect of the remaining free parameters, namely the strength of the motivating action ($c_3$), the relative scaling of rewards ($\rho_1$ vs. $\rho_2$), the trajectory length ($T$), the discount factor ($\gamma$), and the state-based penalty (via $\beta$ and $\beta_0$).

Our experiments compared \textsc{UCB-BOLD} to several benchmark algorithms. \textsc{Optimal} is the optimal policy, obtained by VI using the true system; its value function serves as the reference point for regret calculations. \textsc{Fixed1} and \textsc{Fixed2} are fixed action policies for the two treatment actions. \textsc{Random} selects actions uniformly at random from $\mac U$. We also included epsilon-greedy Q-learning algorithms, using linear function approximation (\textsc{LFA-Q}) and tile-coding (\textsc{TC-Q}), given as Algorithms~\ref{alg:linq} and \ref{alg:tcq} in Appendix~\ref{app:computation}, to serve as model-free baselines. As a myopic benchmark, we include \textsc{GLM-Bandit}, given as Algorithm~\ref{alg:glm_bandit} in Appendix~\ref{app:computation}, which estimates only the shift parameters (i.e., Algorithm~\ref{alg:ucb_epoch} with no planning). We performed grid-search hyperparameter tuning to ensure the Q-learners were configured to be fair model-free baselines (see Appendix~\ref{app:computation}).

\subsubsection{Experiment 1: Reward scaling and motivation strength}
This experiment fixed $\gamma = 0.8$, $T=730$ days, and explored a $4\times 4$ ablation grid, studying the impact of varied reward scaling between treatments and different values of $c$ (i.e., the strength) for the motivating action. For simplicity, we fixed the reward associated with the first treatment action to $\rho_1=1$ and vary $\rho_2=[0.5, 1.0, 1.5, 2.0]$. For the $c$-dimension, we vary a patient-scaled value of $c^i_3=[0,1, 2,3]\cdot (1-a^i)$. We chose this structure as it normalizes $c_3$ to have a more similar impulse response across patients. For simplicity, this experiment assumes $\beta=0$ such that the state-based penalty is absent; we would expect planning-based algorithms to improve in the presence of a state-based penalty (compared to myopic approaches), so removing the penalty represents a floor for their performance.
\subsubsection{Experiment 2: Shortened trajectory}
This experiment tests the robustness of findings with respect to trajectory length, testing Experiment 1's conditions for shorter, $180$-day trajectories. 

\subsubsection{Experiment 3: Effect of discount factor}
The discount factor is ultimately a clinical decision, tied to the particular condition of interest. This experiment shows how regret performance varies under different discounting assumptions. We tested a $730$-day horizon, fixing a plausible setting on the interior of the primary ablation grid: reward scaling of $1.5$ and a patient-scaled $c$ value of $2$. For this configuration, we tested discount factors $\gamma =[0.0, 0.5, 0.8, 0.9, 0.95, 0.98]$, representing effective discounting horizons of $1$, $2$, $5$, $10$, $20$, and $50$ days. 

\subsubsection{Experiment 4: Effect of the state-based penalty}
DT-targeted conditions may carry meaningfully different disengagement or dropout costs; our model allows providers to encode this motivation via the state penalty. This experiment evaluated algorithm performance across a range of penalty types. Similar to the prior experiment, we used a $730$-day horizon, reward scaling of $1.5$, and patient scaled $c$ value of $2$. With these conditions fixed, we tested a grid of penalty magnitudes $\beta=[0,1,2,3]$ and penalty shift parameters $\beta_0=[-2,-4]$ (i.e., setting the midpoint of the penalty sigmoid roughly $2$ or $4$ standard deviations of process noise below the baseline state value). 

\subsection{Results}
To place experiments with different patients (and thus different optimal value functions) on a common scale, we normalized cumulative regret by the expected regret of the \textsc{Random} algorithm for the same patient and trajectory length. We averaged replications for each patient-algorithm-ablation point combination, so plotted variation reflects performance differences across the population.
\subsubsection{Experiment 1: Reward scaling and motivation strength}
Figure~\ref{fig:ablation_box} shows box plots for the normalized regret of the algorithms across the simulated patient population. 
\begin{figure}[]
\FIGURE
{\includegraphics[width=0.8\textwidth]{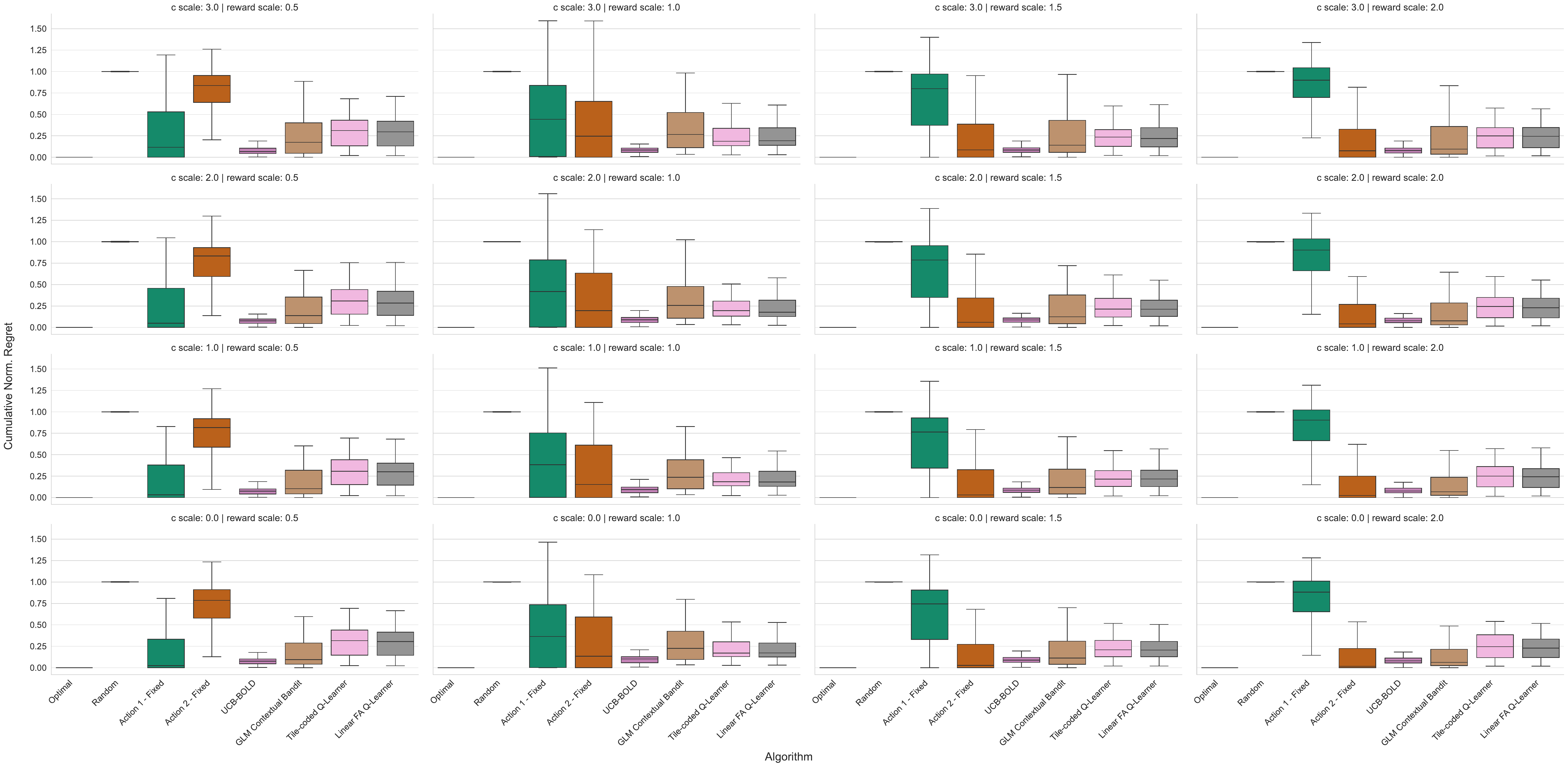}}
{Cumulative normalized regret after $730$ days over the population for the primary ablation grid \label{fig:ablation_box}}
{The scaling of the reward for the second treatment compared to the first increases from $0.5$ to $2.0$ from left-to-right and the patient-scaled effect of the motivating action increases from $0.0$ to $3.0$ from bottom-to-top.}
\end{figure}
The impact of varied reward scaling is most visible in the performance reversal of the fixed action policies, as it does not significantly impact the learning algorithms. The impact of the motivation strength ablation is most visible in the worsening tail performance of the fixed policies and the myopic \textsc{GLM-Bandit}. This makes sense as the presence of a stronger motivating action can increase the optimal value function by managing the engagement state (i.e., selecting the motivating action, even though it carries no reward, to increase the engagement state and thus the likelihood of future adherence). \textsc{Optimal} and \textsc{Random} are included for reference, achieving normalized regret of $0$ and roughly $1$, respectively. Overall, \textsc{UCB-BOLD} shows near-uniform best performance among the algorithms across the grid, both in terms of median regret and regret variation across the population. We see exceptions in the most extreme cases of the ablation grid. For instance, the bottom right subplot shows a case where the \textsc{Fixed2} policy and the myopic \textsc{GLM-Bandit} show better median performance; here the reward scaling and lack of a strong motivating action make it optimal, for many patients, to recommend the second treatment at most (or all) engagement states. Even in such cases, \textsc{UCB-BOLD} shows much lower variation in performance across the population, indicating that while simpler policy structures may serve some patients well, they perform substantially worse for the remaining population (i.e., yielding more extreme worst-case outcomes). The model-free benchmarks perform similarly to one another, typically ranked between \textsc{UCB-BOLD} and \textsc{GLM-Bandit}, depending on the ablation point. \textsc{TC-Q} and \textsc{LFA-Q} are capable of planning with the engagement state, but suffer from sample complexity limitations compared to model-based approaches in this setting. 

Figure~\ref{fig:ablation_ecdf} provides empirical cumulative distribution functions (ECDFs), showing the proportion of the population for which the algorithms achieve a given cumulative normalized regret. This type of plot enables a richer assessment of algorithm behavior over the population. \textsc{Optimal} and \textsc{Random} appear as vertical lines, incurring $0$ and roughly $1$ normalized regret for every patient. We clearly see the cases where fixed or myopic policies serve portions (typically $30-40\%$) of the population well (i.e., the curves vertically track on or close to that of \textsc{Optimal} before shifting right). We found that the upper tail contained individuals where \textsc{Random} sets better baseline performance (i.e., because the optimal policy tends to involve the complete action set, not a trivial collapse to a single action). \textsc{UCB-BOLD} separates itself in this part of the population, suggesting that algorithms that plan around engagement in a sample-efficient way show significantly improved performance for a large portion of the population, even if such planning is not a critical concern for every patient.
\begin{figure}[h]
\FIGURE
{\includegraphics[width=0.8\textwidth]{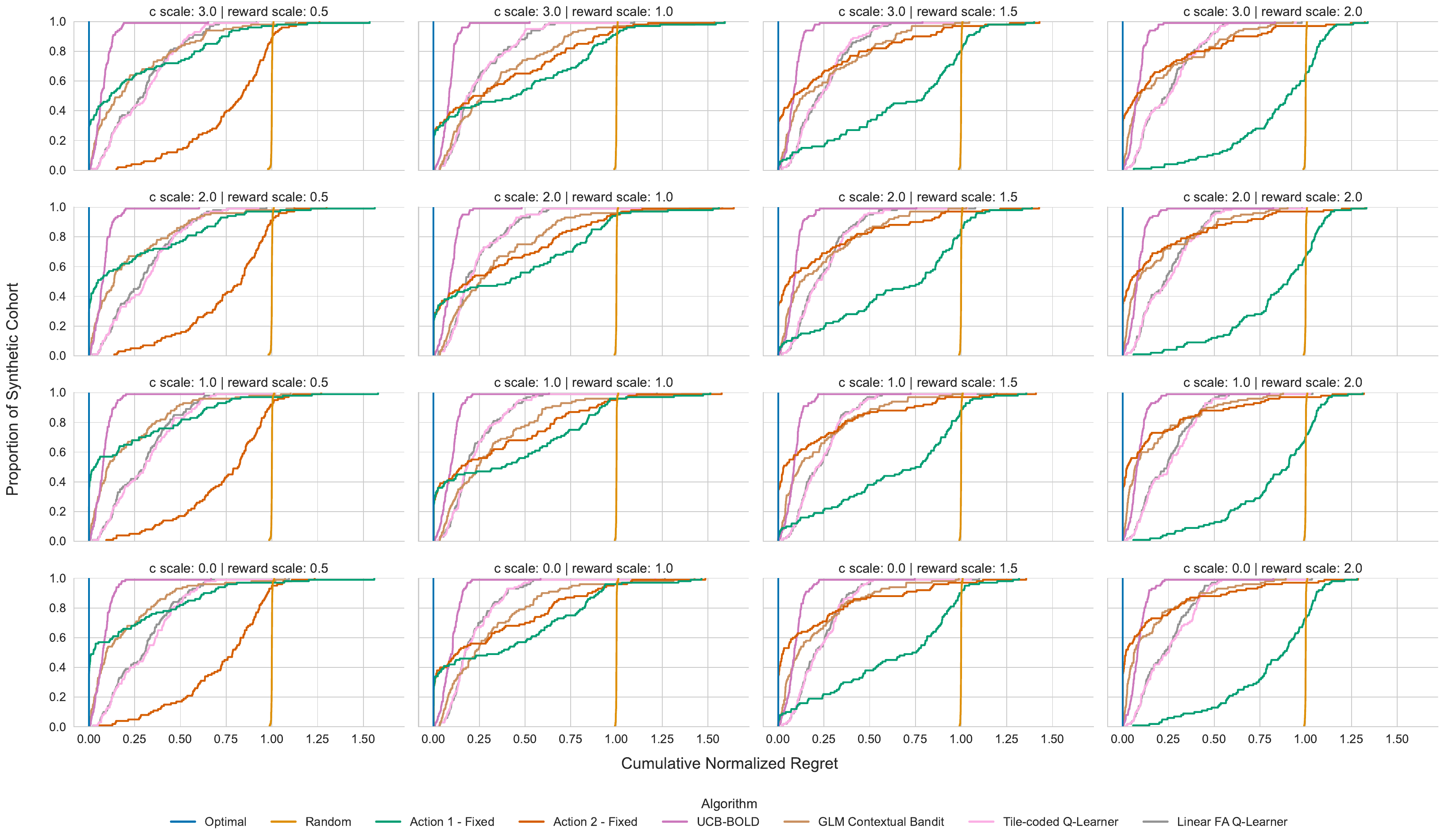}}
{ECDFs for regret across the patient population for the primary ablation grid. \label{fig:ablation_ecdf}}
{The scaling of the reward for the second treatment compared to the first increases from $0.5$ to $2.0$ from left-to-right and the patient-scaled effect of the motivating action increases from $0.0$ to $3.0$ from bottom-to-top.}
\end{figure}
Table~\ref{tab:cvar} shows the empirical conditional value-at-risk (CVaR) regret for various tail widths, averaged over the ablation grid; CVaR effectively measures the expectation of outcomes in the upper $1-\alpha$ tail. \textsc{UCB-BOLD} outperforms the next best algorithm by a factor of $2$-$3\times$, depending on the tail width, with the absolute gap between algorithms widening as the tail width decreases.
\begin{table}[]
    \TABLE
    {Experiment 1 cumulative normalized regret CVaR. \label{tab:cvar}}
    {\begin{tabular}{lcccc}
    \toprule
    Algorithm & $1-\alpha=0.5$ & $1-\alpha=0.25$ & $1-\alpha = 0.1$ & $1-\alpha=0.05$ \\
    \midrule
    \textsc{UCB-BOLD}   & \textbf{0.130} (0.123, 0.134) & \textbf{0.168} (0.156, 0.170) & \textbf{0.229} (0.212, 0.242) & \textbf{0.298} (0.273, 0.322) \\
    \textsc{GLM-Bandit} & 0.395 (0.358, 0.466) & 0.576 (0.521, 0.643) & 0.742 (0.701, 0.827) & 0.870 (0.834, 0.988) \\
    \textsc{TC-Q}       & 0.379 (0.352, 0.406) & 0.464 (0.449, 0.507) & 0.565 (0.551, 0.609) & 0.656 (0.636, 0.700) \\
    \textsc{LFA-Q}      & 0.369 (0.353, 0.397) & 0.463 (0.442, 0.511) & 0.567 (0.552, 0.629) & 0.670 (0.629, 0.713) \\
    \textsc{Fixed1}     & 0.872 (0.685, 0.983) & 1.006 (0.888, 1.067) & 1.123 (1.046, 1.145) & 1.203 (1.158, 1.228) \\
    \textsc{Fixed2}     & 0.532 (0.397, 0.730) & 0.777 (0.643, 0.911) & 1.012 (0.929, 1.058) & 1.111 (1.068, 1.137) \\
    \bottomrule
    \end{tabular}}{CVaR regret for varying tail widths (1-$\alpha$), averaged over the ablation grid. Values given as median (Q1, Q3). Lowest values bolded.}
\end{table}

\subsubsection{Experiment 2: Shortened trajectory}
\begin{figure}[]
\FIGURE
{\includegraphics[width=0.8\textwidth]{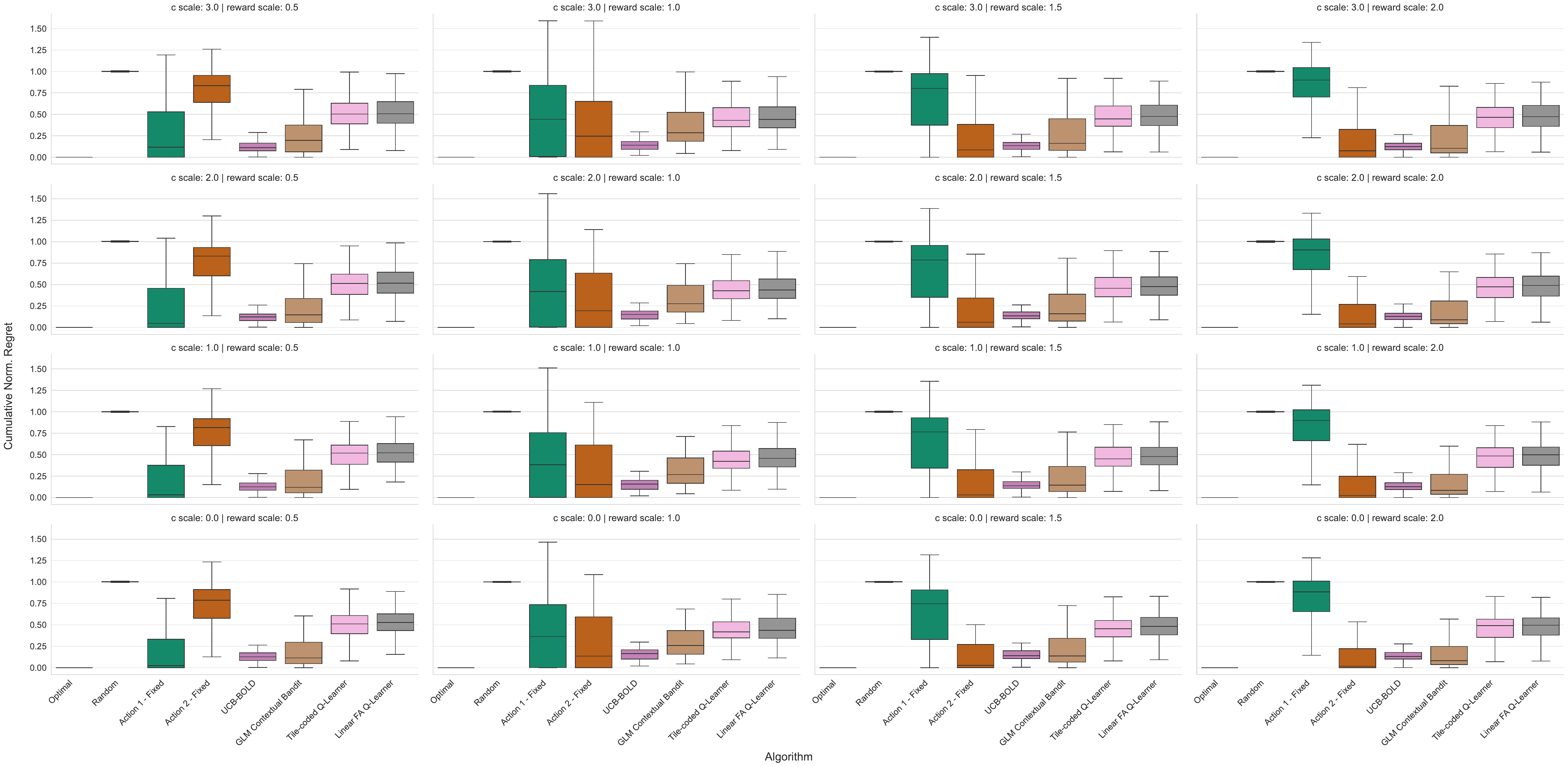}}
{Cumulative normalized regret after $180$ days over the population for the primary ablation grid \label{fig:ablation_box180}}
{The scaling of the reward for the second treatment compared to the first increases from $0.5$ to $2.0$ from left-to-right and the patient-scaled effect of the motivating action increases from $0.0$ to $3.0$ from bottom-to-top.}
\end{figure}
Figure~\ref{fig:ablation_box180} shows box plots for the normalized regret performance of the algorithms for $180$ day trajectories, showing essentially the same performance trends as in the $730$-day trajectories of Figure~\ref{fig:ablation_box}. \textsc{UCB-BOLD} is a consistently strong performer across ablation conditions, with \textsc{GLM-Bandit} showing competitive median performance for some ablation points but substantially larger variance over the population. The Q-learning algorithms perform worse, relative to their competition, in this shorter horizon setting, owing to the greater sample complexity demands of model-free approaches.

\subsubsection{Experiment 3: Effect of discount factor}
\begin{figure}[]
\FIGURE
{\includegraphics[width=0.9\textwidth]{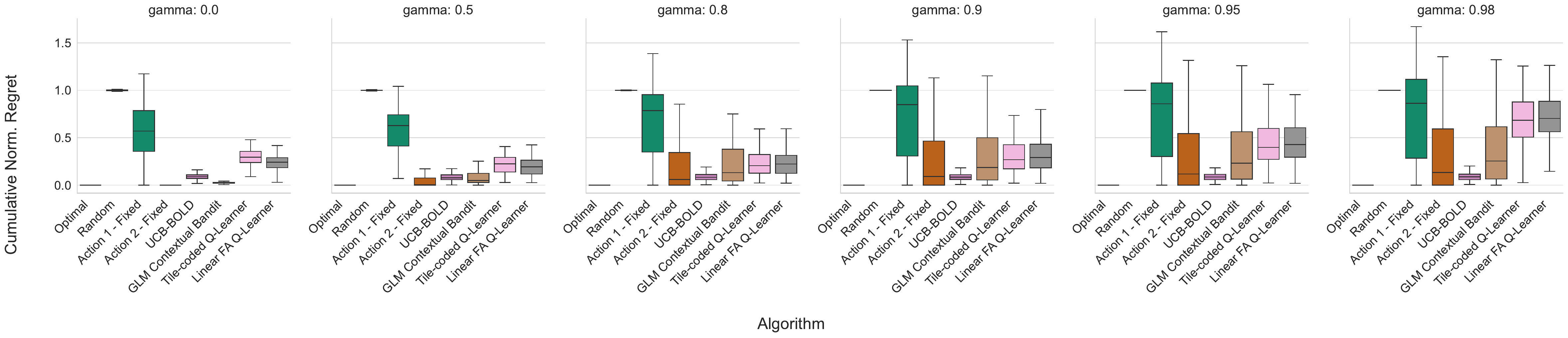}}
{Cumulative normalized regret after $730$ days over the population for different discount values $\gamma$ \label{fig:ablation_box_gamma}}
{Discount factor increases from left-to-right. Effective horizons, $\left(\nicefrac{1}{(1-\gamma)}\right)$, range from $1$ day to $50$ days. This sweep fixes the reward scaling value to $1.5$ and the patient-scaled $c$ value for the motivating action to $2.0$.} 
\end{figure}
Figure~\ref{fig:ablation_box_gamma} shows box plots for $730$-day normalized regret performance with varying discount factors. Myopic approaches such as \textsc{GLM-Bandit} perform well for lower discount factors, but as the discount factor increases, \textsc{UCB-BOLD} separates as the best performer. In practice, $\gamma$ is a clinical choice; these results suggest that planning algorithms are critical for settings focused on long-term outcomes. We see a separation between \textsc{UCB-BOLD} and the Q-learners, owing to a poorer sample complexity scaling of these methods in $\gamma$ compared to model-based approaches. 

\subsubsection{Experiment 4: Effect of state-based penalty}
\begin{figure}[]
\FIGURE
{\includegraphics[width=0.8\textwidth]{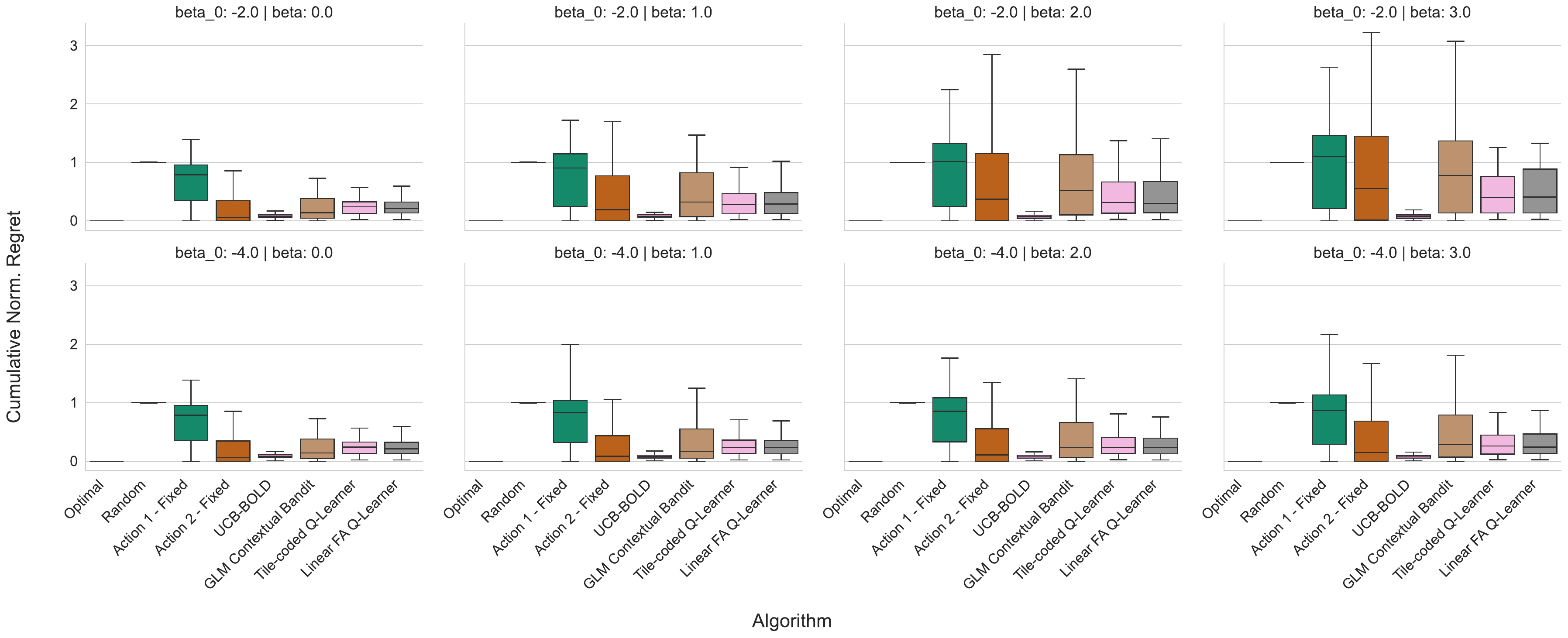}}
{Cumulative normalized regret after $730$ days over the population for various state penalties \label{fig:ablation_box_beta}}
{$\beta$ values, denoting penalty strength, increase from $0$ to $3$ moving left-to-right. The location parameter $\beta_0$ is $-4$ in the bottom row and $-2$ in the top row. A more negative value means the sigmoidal penalty does not `turn on' until the state goes more strongly negative. This sweep fixes the reward scaling value to $1.5$ and the patient-scaled $c$ value for the motivating action to $2.0$.} 
\end{figure}
Figure~\ref{fig:ablation_box_beta} explores different state-based penalties, varying both penalty function magnitude ($\beta$) and location ($\beta_0$). This penalty, designed to reflect a provider's motivation to keep a patient from disengaging, increases the importance of engagement planning. We observe that \textsc{UCB-BOLD} performs uniformly well across this ablation study while myopic approaches like \textsc{GLM-Bandit} suffer with a stronger penalty, particularly in the upper tail of the population. The performance of the model-free Q-learners degrades moderately with higher penalties, likely due to the fact that the fixed set of hyperparameters used across the ablation studies are not optimal in these separate system configurations and that they are not designed to exploit a (partially) known reward structure. However, the purpose of this ablation is not to highlight differences between model-free and model-based approaches, but rather to highlight scenarios where engagement planning is especially important versus where myopic policies suffice.

\section{Managerial implications and limitations}\label{sec:managerial}
Our work suggests the following key implications:
\begin{enumerate}
    \item \textit{The benefits of engagement planning (i.e., selecting treatments based on recommendation and adherence effects on the patient's future engagement state) may differ meaningfully across a patient population.} Our computational experiments relied on a synthetic patient cohort, calibrated using MRT data, allowing examination of algorithm performance across a range of plausible patient types. For some patient types (up to roughly $50\%$ of the population under some ablation conditions), myopic or heuristic policies sufficed due to a single action dominating the optimal policy for those patients. Among the rest of the population, there were clear benefits to strategically planning around the patient's time-varying engagement dynamics (i.e., higher state persistence or larger recommendation and adherence effects meant the optimal policy was not dominated by a single action). \textsc{UCB-BOLD} showed drastically better performance for this sub-cohort, showing 2-3 times lower CVaR regret than the next-best benchmark (Table~\ref{tab:cvar}). These trends were robust to trajectory length (Figure~\ref{fig:ablation_box180}) and were even stronger in ablation settings that favor planning, for instance higher discount factors (Figure~\ref{fig:ablation_box_gamma}), a stronger motivating action (Figure~\ref{fig:ablation_box}), and a stronger state-based penalty (Figure~\ref{fig:ablation_box_beta}). Explicitly accounting for engagement dynamics may be an important way for DT providers to ensure quality treatment recommendation across a complete patient population.
    \item \textit{Dynamical modeling can bridge descriptive behavioral models with long-run planning for mHealth interventions.} Personalized mHealth interventions operate in a sample-scarce regime and incorporating explicit dynamics models allows decision-makers to exploit problem structure and plan for concerns like burden and fatigue in ways that myopic approaches cannot. However, doing so can introduce model misspecification concerns when models do not approximate reality well. Our case study suggested that the benefits of engagement state planning may be meaningful for some, but not all, patients; decision-makers will need to explicitly weigh the tradeoffs of imposing planning-oriented structure in their intervention context. Our work provides a blueprint for demonstrating relevant theoretical guarantees under the posited dynamics. As the psychology community develops better-grounded dynamical models for behavior \citep{brigantiNetworkAnalysisOverview2024a,perskiIterativeDevelopmentRefinement2025}, this same approach can be used to match methodological development to current behavioral science.
\end{enumerate}
\subsection{Limitations and extensions}
We focus on a purposefully parsimonious model class; there are many extensions not addressed here, for instance, hierarchical models, time-varying dynamics, multi-dimensional or latent states, or integration of context into the adherence model. Further, treating engagement as synonymous with adherence behavior may be incomplete (i.e., a patient can invest different amounts of energy into completing treatment) \citep{nahum-shaniEngagementDigitalInterventions2022a}. Models capturing these distinctions offer rich theoretical and practical directions for future work. Next, our MRT data have an hourly timescale, which we mapped to a daily timescale for our computational experiments. It is not clear whether interventions operating at a daily timescale will involve more or less persistent dynamics; these effects may also vary significantly by intervention setting. Additional exploratory and descriptive work is needed to refine the functional form of model-based approaches, but existing work has shown that such approximations perform well empirically \citep{elmistiriModelPredictiveControl2025a,pulickIdiographicLapsePrediction2025b}.

\section{Conclusion}\label{sec:conclusion}
In this paper, we developed a framework modeling how DT treatment decisions affect patient adherence over time based on a time-varying engagement state. Our model operates on a daily timescale and is aimed at long-running interventions, where sustaining engagement is especially important for providers. We chose a parsimonious linear dynamical system model to capture the phenomena of interest, namely separate effects for recommending a treatment and for patient adherence. We showed finite-time system identification guarantees for this model under a flexible policy class. We proposed an optimism-based algorithm, \textsc{UCB-BOLD}, and proved that it achieves sublinear regret. Last, we used MRT data to generate synthetic patients and demonstrate that \textsc{UCB-BOLD} outperformed benchmarks over a wide range of system conditions. Our work offers a model-based approach for JITAIs to incorporate dynamical systems models from behavioral psychology, and our framework serves as a foundation for natural model extensions in future work.



\bibliographystyle{informs2014} 
\bibliography{refs} 

\begin{APPENDICES}
\section{Computational experiments}\label{app:computation}
We provide pseudocode for \textsc{LFA-Q}, \textsc{TC-Q}, and \textsc{GLM-Bandit} as Algorithms~\ref{alg:linq}-\ref{alg:glm_bandit}. 
\begin{algorithm}[htbp]
\footnotesize
\caption{Linear Function Approximation Q-Learner (LFA-Q)}
\label{alg:linq}
\begin{algorithmic}[1]
\Require Learning rate $\alpha\in(0,1)$, decay $\epsilon_d\in(0,1)$, radial basis function (RBF) count $n$, discount factor $\gamma\in(0,1)$, grid bound $C_g$
\State Place RBF centers $\{c_j\}_{j=1}^n$ uniformly on $[-1,1]$, set width $\sigma=0.8\cdot(c_2-c_1)$
\State Define features $\phi_j(x)=\exp\big(-(\tilde x - c_j)^2 / (2\sigma^2) \big)$, where $\tilde x$ rescales $x$ to $[-1,1]$ using $C_g$
\State Initialize weights $W \in \R^{(M+1)\times n}$
\For{t = 1,...,T}
    \State Observe state $x_t$, set exploration level $\epsilon_t = \nicefrac{1}{t^{\epsilon_d}}$
    \State With prob. $\epsilon_t$, select $u_t$ randomly from $\mac U$, otherwise $u_t=\arg \max_{i}W_i^\top \phi(x_t)$
    \State Observe $r(x_t,u_t,d_t)$ and $x_{t+1}$
    \State Set TD error $\delta_t = r(x_t,u_t,d_t) + \gamma \max_{i} W_i^\top \phi(x_{t+1}) - W_{u_t}^\top \phi(x_t)$
    \State Update $W_{u_t} = W_{u_t} + \alpha \delta_t \phi(x_t)$
\EndFor
\end{algorithmic}
\end{algorithm}
\begin{algorithm}[htbp]
\footnotesize
\caption{Tile-coded Q-Learner (TC-Q)}
\label{alg:tcq}
\begin{algorithmic}[1]
\Require Learning rate $\alpha\in(0,1)$, decay $\epsilon_d\in(0,1)$, bin width $w$, tiling count $n_t$, discount factor $\gamma\in(0,1)$, grid bound $C_g$
\State Construct $n_t$ tilings of the state space, each with bin width $w$, offset by $\nicefrac{kw}{n_t}$ for $k=0,\dots,n_t-1$
\State Let $b_k(x)$ denote the bin index for state $x$ under tiling $k$
\State Initialize $Q_k(b,u)$ for all tilings $k$, bin indices $b$, and actions $u$, set $Q(x,u)=\sum_kQ_k(b_k(x),u)$
\For{t = 1,...,T}
    \State Observe state $x_t$, set exploration level $\epsilon_t = \nicefrac{1}{t^{\epsilon_d}}$
    \State With prob. $\epsilon_t$, select $u_t$ randomly from $\mac U$, otherwise $u_t=\arg \max_{u}Q(x_t,u)$
    \State Observe $r(x_t,u_t,d_t)$ and $x_{t+1}$
    \State Set TD error $\delta_t = r(x_t,u_t,d_t) + \gamma \max_{u} Q(x_{t+1},u) - Q(x_t,u_t)$
    \State Update $Q_k(b_k(x_t),u_t) = Q_k(b_k(x_t),u_t) + \nicefrac{\alpha}{n_t} \delta_t$ for $k=0,\dots,n_t-1$
\EndFor
\end{algorithmic}
\end{algorithm}
\begin{algorithm}[htbp]
\footnotesize
\caption{GLM-Bandit}
\label{alg:glm_bandit}
\begin{algorithmic}[1]
\Require Confidence $\delta \in (0,1)$, $C_N>0$, epoch count $k=1$, action counts $N_{i}=0$ ($i=1,\dots,M$).
\State Estimate $\hat \mu_1$, set $\alpha_{i}^{\mu,(k)}=\alpha^\mu_{i,1}(\nicefrac{\delta}{2})$, $\tilde \mu_i^{(k)} = \Pi_{[-\bar \mu, \bar \mu]} \big[\hat \mu_{i,1}+\alpha^{\mu}_{i,1}\big]$, action counts $N_{i}^{(k)}=N_i$ ($i=1,\dots,M)$.
\For{t = 1,...,T}
    \State Observe $d_{t-1}$, $x_t$
    \If {$N_{i} > (1+C_N) N_{i}^{(k)}$ for any $i=1,\dots,M$}
        \State Set $k = k +1$, $\alpha^{\mu,(k)}_{i} = \alpha^\mu_{i,t}(\nicefrac{\delta}{2})$, $\tilde \mu_i^{(k)} = \Pi_{[-\bar \mu, \bar \mu]} \big[\hat \mu_{i,t}+\alpha^{\mu}_{i,t}\big]$, $N_{i}^{(k)}=N_{i}$ for $i=1,\dots,M$
    \EndIf
    \State Take action $u_t=\arg\max_{u \in \mac U} \ex_{\tilde \mu^{(k)}}\big [r(x_t,u_t,d_t)\big]$
    \State Set $N_{i}=N_{i}+\mathbf{1}_{\{u_t=e_i\}}$, $i=1,\dots, M$
\EndFor
\end{algorithmic}
\end{algorithm}
\vspace{-16pt}
We ran experiments on a university computing cluster, using parameters: $\bar a = 0.85$, $\bar b = 3.75$, $\bar c = 2.75$, $\bar w=2.5$, and $\bar \mu=2.5$. Where needed, we discretized the state $x$ to a grid with resolution $0.1$. While $C_x=\frac{\bar b + \bar c +\bar w}{1-\bar a}$ bounds the state, we used the tighter range $[\nicefrac{-C_x}{3}, \nicefrac{C_x}{3}]$ since $C_x$ is loose in practice. 
For hyperparameter tuning, we generated a separate $50$-person cohort from our HM posterior and ran coarse and fine grid searches. We selected the tuning configurations with the lowest end-of-trajectory cumulative regret, averaged over $10$ replications, the $4 \times 4$ ablation grid from Experiment 1, and the population. For computational efficiency during tuning, we evaluated algorithm value functions every $20$ time steps for regret calculations. During coarse tuning we found that both Q-learners performed well with fast $\epsilon$ decay rates (i.e., greater than $1$), due to the built-in exploration from optimistic initialization. We found a negligible difference between fast-decay cases and the no exploration case, so we fixed the $\epsilon$ decay parameter ($1.5$) to shrink the tuning grids. The final grid search for \textsc{TC-Q} explored bin widths of $[20, \mathbf{40}]$, tiling counts of $[\mathbf{64},96,128,160,192]$, and learning rates of $[0.3,0.4,\mathbf{0.5},0.6]$ (final values bolded). As \textsc{TC-Q} constructs bins to be the given value or finer, these bin widths capture when the grid is divided into either $2$ or $1$ interior bins; the algorithm performed best with very few bins and a high tiling count. Tile counts of $64$ and $192$ performed similarly, so we chose the simpler model. The final grid search for \textsc{LFA-Q} tested RBF center counts of $[4,5,\mathbf{6},7]$ and learning rates of $[0.2,\mathbf{0.3},0.4,0.5]$. To reduce computation, Experiments 2-4 used the same hyperparameters as Experiment 1. We tested $C_d,C_N=[0.5, 0.75, 1.0]$ for \textsc{UCB-BOLD} and \textsc{GLM-Bandit} and found marginal improvements with lower values, so we selected $0.5$ for each to balance performance and computational efficiency. Both algorithms used regularization values of $1$. As the theoretically derived constants for \textsc{UCB-BOLD} and \textsc{GLM-Bandit} are loose we found the algorithms performed well when the $\theta$- and $\mu$-derived bonus terms were rescaled so that the maximum initial values over the grid were $\nicefrac{\bar \rho}{2}$. 

\section{Supporting proofs}
\subsection{Supporting analysis for model definition (Section~\ref{sec:model})}\label{app:gen_model}
Define filtrations $\mac F^d_t:=\sigma\left(x_{0:t},u_{0:t},d_{0:t-1},w_{0:t-1}\right)$, $\mac F^w_t := \sigma(x_{0:t},u_{0:t},d_{0:t},w_{0:t-1})$, and $\mac F_t:=\sigma\left(x_{0:t},u_{0:t},d_{0:t},w_{0:t}\right)$ for separate parts of the following analysis (we use $x_0:=0$, $u_0,d_0:=\ve{0}_M$, $w_0\sim w$ for notational convenience). 
\proof{Proof of Proposition~\ref{prop:absolute_bounded_x}}
$x_1\sim w$ and $|w|\leq \bar w\leq C_x$ so $x_1\in\mac X$ by construction and we proceed by induction on $t$. Assume $x_{t} \in \mac X$ for $t \geq 1$. Then $|x_{t+1}| \leq |ax_{t}| + |b^\top u_{t}| + |c^\top d_{t}| + |w_{t}| \leq \bar a C_x + \bar b + \bar c +\bar w =\frac{\bar b +\bar c +\bar w}{1-\bar a}=C_x$ implying $x_{t+1}\in \mac X$. The bound on the feature vector follows by the $1$-sparsity of $u_t$ and $d_t$ under the assumed model structure. \halmos
\endproof
\begin{remark}{\textbf{(Parameter space bounded)}}\label{rem:parameter_space_bounded}
For $\Theta=\Theta_a\times \Theta_b\times \Theta_c$, $\max_{\theta \in \Theta}\Vert\theta\Vert_2\leq \sqrt{\bar a^2 + M\bar b^2 + M \bar c^2}$
\end{remark}

\subsection{OLS identification supporting proofs (Section~\ref{sec:identification})}\label{app:ols_identification}
\begin{corollary}{(Corollary to Anderson's Inequality)}\label{cor:shift}
    For log-concave and zero-symmetric RV $X\in\R$, $\nu>0$, and $y\in \R$, then $\pr(|X+y|\geq \nu)\geq \pr(|X|\geq \nu)$. If RV $Y\in \R$ is independent of $X$ then $\pr(|X+Y|\geq \nu)\geq \pr(|X|\geq \nu)$.
\end{corollary}
\proof{Proof of Corollary~\ref{cor:shift}}
    Let $E:=\{x\in \R: -\nu< x <\nu\}$. We have $f(x)=f(-x)$ by assumption. Note that log-concavity of a distribution implies quasiconcavity, meaning that the superlevel sets for $f(x)$ are convex \citep{boydConvexOptimization2004}. Then by Anderson's Inequality (\textcite{andersonIntegralSymmetricUnimodal1955}, Corollary 2 and Theorem 2) with $k=0$ we have $\pr(|X|< \nu) \geq \pr(|X+y|< \nu)\Rightarrow \pr(|X+y|\geq \nu)\geq \pr(|X|\geq \nu)$. The random variable case follows by an identical argument.\halmos
\endproof
The following assume Assumptions~\ref{asm:observed_quantities}-\ref{asm:noise_properties} hold and an $(r,k)$-exploratory policy. Lemma~\ref{lem:tau} demonstrates that there exists an appropriate choice of $\tau$ for Proposition~\ref{prop:bmsb_formal}. $\tau$ is a proof tool that comes from the decomposition of the unit sphere into three cases and the optimal choice of $\tau$ depends on the properties of the chosen noise distribution.
\begin{lemma}{}\label{lem:tau}
There exists $\tau\in(0,1)$ with $\pr\big(|w|\geq \frac{\sqrt{(1-\tau^2)}}{\tau}\frac{1+4\sqrt{M}}{2\sqrt{2M}}\big)>0$.
\end{lemma}
\proof{Proof of Lemma~\ref{lem:tau}}
By Assumption \ref{asm:noise_properties}, $w$ must be fully supported on $[-C,C]$ for $C>0$. Let $g(\tau)=\frac{\sqrt{(1-\tau^2)}}{\tau}\frac{1+4\sqrt{M}}{2\sqrt{2M}}$. $g(\tau)$ is continuous for $\tau \in (0,1)$ and $\lim_{\tau\to 1^-}g(\tau)=0$. Thus $\exists\tau$ such that $g(\tau)<C$ and $\pr(\vert w\vert\geq g(\tau))>0$.
\halmos
\endproof
\begin{proposition}{}\label{prop:bmsb_formal}
Choose $\tau\in(0,1)$ such that $p_1(\tau):=\pr\big(|w_{t-1}|\geq \frac{\sqrt{(1-\tau^2)}}{\tau}\frac{1+4\sqrt{M}}{2\sqrt{2M}}\big)>0$. Let $p:=\min\{p_1(\tau),\frac{r\underline p}{M+1}\}$. Then this system satisfies the $(k,\frac{1-\tau^2}{8M}I_{2M+1},p)$ BMSB condition.
\end{proposition}
\proof{Proof of Proposition~\ref{prop:bmsb_formal}}
Note that $(z_t)_{t\geq 1}$ is $\mac F^w_t$-adapted. We consider a block of $k$ time steps following some $j\geq 0$. Let $\lambda=[\lambda_1,\lambda_2^\top,\lambda_3^\top]^\top$ with $\lambda_1 \in \R$, $\lambda_2\in \R^M$, $\lambda_3\in \R^M$. We must show that for any $\lambda$ with $\|\lambda\|_2=1$, $\frac{1}{k}\sum_{i=1}^k\pr (| \lambda^\top z_{j+i}| \geq \sqrt{\lambda^\top (\frac{1-\tau^2}{8M} I_{2M+1})\lambda} | \mac F^w_j)\geq p$ holds, or equivalently $\frac{1}{k}\sum_{i=1}^k\pr (| \lambda^\top z_{j+i}| \geq \nu | \mac F^w_j)\geq p$ for $\nu:=\frac{1}{2}\sqrt{\frac{1-\tau^2}{2M}}$. We prove this condition for three cases of $\lambda$: $\vert\lambda_1\vert\geq \tau$, $\vert\lambda_1\vert < \tau$ with $\Vert\lambda_3\Vert_2\geq \sqrt{\frac{1-\tau^2}{2}}$, and $\vert\lambda_1\vert < \tau$ with $\Vert\lambda_3\Vert_2<\sqrt{\frac{1-\tau^2}{2}}$.
Cases 2 and 3 represent when at least half of the remainder magnitude ($1-\tau^2$) comes from $\lambda_3$ or $\lambda_2$, respectively.

\textbf{Case 1:}
Let $| \lambda_1 | \geq \tau$. Consider arbitrary time step $t>j$ in the $k$-block. If $|\lambda_1|\geq \tau$, then $\|[\lambda_2^\top,\lambda_3^\top]\|_2\leq \sqrt{1-\tau^2}$. As $u_t$ and $d_t$ are $1$-sparse, the reverse triangle inequality and Cauchy-Schwarz give $\lvert \lambda^\top z_t\rvert =\lvert\lambda_1x_t+\lambda_2^\top u_t + \lambda_3^\top d_t\rvert \geq \left\lvert|\lambda_1x_t|-|\lambda_2^\top u_t+\lambda_3^\top d_t|\right\rvert\geq \lvert \lambda_1x_t\rvert-\sqrt{2(1-\tau^2)}$, so it suffices to lower bound $\pr(|\lambda_1x_t|\geq \nu +\sqrt{2(1-\tau^2)}\vert \mac F^w_j)$. Per the system dynamics: $\lambda_1x_t = \lambda_1w_{t-1}+[\lambda_1ax_{t-1}+\lambda_1b^\top u_{t-1}+\lambda_1c^\top d_{t-1}]$. Let $Q_{t-1}$ denote the bracketed term. Observe that $Q_{t-1}|\mac F^w_j$ is fixed for  $t=j+1$ and a random variable, independent of $w_{t-1}$, if $t > j+1$. Thus by Corollary~\ref{cor:shift}:
\begin{align}
    &\pr(|\lambda_1x_t|\geq \nu +\sqrt{2(1-\tau^2)}\vert \mac F^w_j) = 
    \pr(|\lambda_1w_{t-1}+Q_{t-1}|\geq \nu+\sqrt{2(1-\tau^2)}\vert \mac F^w_j)\\
    &\quad \geq \pr(|\lambda_1w_{t-1}|\geq \nu+\sqrt{2(1-\tau^2)})
    \geq\pr(|w_{t-1}|\geq \tfrac{1}{\tau}[\tfrac{1}{2}\sqrt{\nicefrac{(1-\tau^2)}{2M}}+\sqrt{2(1-\tau^2)}])=p_1(\tau)
\end{align}
For the $k$-block we have $ \frac{1}{k}\sum_{i=1}^k\pr (| \lambda^\top z_{j+i}| \geq \nu | \mac F^w_j)\geq \frac{1}{k}\sum_{i=1}^kp_1(\tau)=p_1(\tau)$.

\textbf{Case 2:} Let $|\lambda_1|<\tau$ and $\|\lambda_3\|_2\geq \sqrt{\frac{1-\tau^2}{2}}$. Let $E_t$ be the ($\mac F_{t-1}$-measurable) event that an exploratory action is taken in step $t$. We consider $t>j$ in the $k$-block with $\mathbf{1}(E_t)=1$ and focus on the action index $i^*\in\{1,\dots,M\}$ where $| \lambda_3^i|$ is largest. Observe that $| \lambda_3^{i^*}|$ is smallest when all components of $\lambda_3$ have equal magnitude. Since $\| \lambda_3 \|_2 \geq \sqrt{\frac{1-\tau^2}{2}}$, it follows that $| \lambda_3^{i^*}| \geq \sqrt{\frac{1-\tau^2}{2}}\frac{1}{\sqrt{M}}=2\nu$. By definition of the $(r,k)$-policy, $\pr(u_t=e_{i^*}|\mac F_{t-1},E_t)=\frac{1}{M+1}$. We show that for any state, $x_t\in \R$, the system is sufficiently excited by at least one outcome; we define two sets corresponding to the outcomes $d^{i^*}_t=1$ and $d^{i^*}_t=0$. Let $A=\{x_t \in \R:|\lambda_1x_t+\lambda^{i^*}_2+\lambda^{i^*}_3|\geq \nu\}$ and $B=\{x_t\in \R:|\lambda_1x_t+\lambda^{i^*}_2|\geq \nu\}$. We claim $A^C\cap B^C= \emptyset$, i.e., any $x_t\in \R$ belongs to at least one set. By contradiction, let $\exists x \in \R$ with $x\in (A^C \cap B^C)$. Above we established $|\lambda^{i^*}_3|\geq 2\nu$, so the triangle inequality and the set properties produce the desired contradiction: $|\lambda^{i^*}_3| = |(\lambda_1x+\lambda^{i^*}_2+\lambda^{i^*}_3)-(\lambda_1x+\lambda^{i^*}_2)| \leq |\lambda_1x+\lambda^{i^*}_2+\lambda^{i^*}_3|+|\lambda_1x+\lambda^{i^*}_2| < 2\nu$. This implies that for any $x_t\in \R$, $\mathbf{1}_A(x_t)+\mathbf{1}_B(x_t)\geq 1$. Recall that $x_t$ is $\mac F_{t-1}$-measurable. We lower bound with the $u_t=e_{i^*}$ case:
\begin{align}
    &\pr(|\lambda^\top z_t|\geq \nu|\mac F_{t-1},E_t)\geq \pr(u_t = e_{i^*}| \mac F_{t-1},E_t)\pr(|\lambda^\top z_t|\geq \nu|\mac F_{t-1},E_t,u_t=e_{i^*})\\
    &\quad\overset{(a)}{=}\tfrac{1}{M+1}[\mathbf{1}_{A}(x_t)\hat p_{i^*}(x_t,1)+\mathbf{1}_{B}(x_t)(1-\hat p_{i^*}(x_t,1))]
    \overset{(b)}{\geq} \min \{\tfrac{\underline{p}}{M+1} ,\tfrac{(1-\overline p)}{M+1}\}=:p_2
\end{align}
$(a)$ substitutes the action probability, expands adherence outcomes, and uses the set definitions. $(b)$ substitutes $\hat p$ bounds by Remark~\ref{rem:adh_prob_bounded}, holding $\forall x_t \in \mac X$, and the set conclusion above. Since $E_t$ is $\mac F_{t-1}$ measurable $\pr(E_t|\mac F_{t-1})=\mathbf{1}(E_t)$. We have $\pr(| \lambda^\top z_t| \geq \nu | \mac F_{t-1})\geq \pr(| \lambda^\top z_t| \geq \nu, E_t | \mac F_{t-1}) = \pr(E_t|\mac F_{t-1})\pr(| \lambda^\top z_t| \geq \nu | \mac F_{t-1},E_t)\geq \mathbf{1}(E_t)p_2$. Thus:
\begin{align}
    &\tfrac{1}{k}\textstyle\sum\nolimits_{i=1}^k\pr (| \lambda^\top z_{j+i}| \geq \nu | \mac F^w_j)
    \overset{(c)}{=} \tfrac{1}{k}\textstyle\sum\nolimits_{i=1}^k\ex[\pr (| \lambda^\top z_{j+i}| \geq \nu | \mac F_{j+i-1})\vert \mac F^w_j]\\
    &\qquad \overset{(d)}{\geq} \tfrac{1}{k}\textstyle\sum\nolimits_{i=1}^k\ex[\mathbf{1}(E_{j+i})p_2| \mac F^w_j]\overset{(e)}{=}\tfrac{p_2}{k}\ex[\textstyle\sum\nolimits_{i=1}^k\mathbf{1}(E_{j+i})|\mac F^w_j] \overset{(f)}{\geq} \tfrac{p_2\lceil rk\rceil}{k}\geq rp_2
\end{align}
$(c)$ uses the tower property as $\mac F_j^w \subseteq \mac F_{j+i-1}$ for $i\geq 1$. $(d)$ uses the per-step bound above. $(e)$ uses the linearity of expectation and $(f)$ uses the guaranteed exploratory steps per block of $(r,k)$-policies.

\textbf{Case 3:} Let $| \lambda_1 | <\tau$ and $\| \lambda_3 \|_2 <\sqrt{\frac{1-\tau^2}{2}}$. $\| \lambda \|_2 =1$ implies that $\| \lambda_2 \|_2^2 = 1- | \lambda_1|^2 - \| \lambda_3 \|_2^2 \geq \frac{1-\tau^2}{2} $. Similar to Case 2, this implies some maximizing index $i^*\in\{1,\dots,M\}$ with $\vert \lambda_2^{i^*}\vert \geq \sqrt{\frac{1-\tau^2}{2}}\frac{1}{\sqrt{M}}=2\nu$. Define $E_t$ as in Case 2 and consider $t>j$ with $\mathbf{1}(E_t)=1$. We make a similar argument to Case 2, but consider two actions: 
$e_{i^*}$ and $\ve{0}_M$.
As before, the $(r,k)$-policy definition gives $\pr(u_t=e_{i^*}|\mac F_{t-1},E_t)=\frac{1}{M+1}$. We define two relevant sets, $A=\{x_t\in \R:|\lambda_1x_t+\lambda_2^{i^*}|\geq \nu\}$ and $B=\{x_t\in \R:|\lambda_1x_t|\geq \nu\}$. These correspond to the cases $u^{i^*}_t=1,d^{i^*}_t=0$ and $u_t=\ve{0}_M,d_t=\ve{0}_M$, respectively. We claim $A^C\cap B^C= \emptyset$, i.e., any $x_t\in \R$ belongs to at least one set. For a contradiction, assume $\exists x \in \R$ with $x \in \left(A^C \cap B^C\right)$. Earlier we established that $\vert \lambda^{i^*}_2 \vert\geq 2\nu$, so the triangle inequality and the set properties produce the desired contradiction: $|\lambda_2^{i^*}| = |(\lambda_1x+\lambda_2^{i^*})-(\lambda_1x)|\leq |\lambda_1x+\lambda_2^{i^*}|+|\lambda_1x| < 2\nu$. This implies that for any $x_t\in \R$, $\mathbf{1}_A(x_t)+\mathbf{1}_B(x_t)\geq 1$. We lower bound with the $u_t=\ve{0}_M$ and $u_t=e_{i^*}$ cases:
\begin{align}
    \pr(|\lambda^\top z_t|\geq \nu|\mac F_{t-1},E_t)
    &\geq\pr(u_t=\ve{0}_M| \mac F_{t-1},E_t) \pr(|\lambda^\top z_t|\geq \nu|\mac F_{t-1},E_t,u_t=\ve{0}_M)  \nonumber\\
    &\qquad+ \pr(u_t=e_{i^*}\vert \mac F_{t-1},E_t) \pr(|\lambda^\top z_t|\geq \nu|\mac F_{t-1},E_t,u_t=e_{i^*})\\
    &\overset{(a)}{\geq} \tfrac{1}{M+1} [\mathbf{1}_{B}(x_t)+ (1-\hat p_{i^*}(x_t,1)) \mathbf{1}_{A}(x_t) ] 
    \overset{(b)}{\geq}  \min \{\tfrac{1}{M+1}, \tfrac{1-\overline p}{M+1}\}=\tfrac{1-\overline p}{M+1}=:p_3
\end{align}
$(a)$ substitutes the action probabilities, expands the Bernoulli adherence outcomes (adherent case is dropped), and uses the set definitions. $(b)$ uses the set conclusion above and substitutes $\hat p$ bounds by Remark~\ref{rem:adh_prob_bounded}, holding $\forall x_t \in \mac X$. Following an identical tower property argument to Case 2, we have for the $k$-block: $\frac{1}{k}\sum_{i=1}^k\pr (| \lambda^\top z_{j+i}| \geq \nu | \mac F^w_j)\geq rp_3$. Thus the BMSB condition holds for all cases with $p=\min\{p_1(\tau),rp_2,rp_3\}=\min\{p_1(\tau),\frac{r\underline p}{M+1}\}$ as $\underline p=1-\overline p$ by the symmetry of the sigmoid.\halmos
\endproof
\proof{Proof of Proposition~\ref{prop:ub-cond-formal}}
By Proposition~\ref{prop:absolute_bounded_x}, $\Vert z_t\Vert^2_2 \leq C_x^2+2$ for $t\geq 1$ almost surely. Let $u \in \R^{2M+1}$ and observe by Cauchy-Schwarz $u^\top z_tz_t^\top u=(z_t^\top u)^2 \leq \Vert z_t \Vert^2_2 \Vert u\Vert^2_2 \leq u^\top ((C_x^2+2)\ma{I}_{2M+1}) u$. Thus $z_tz_t^\top \preceq (C_x^2+2)\ma{I}_{2M+1}$ almost surely. The proposition follows by summing over time steps.
\halmos
\endproof
\end{APPENDICES}

\ECSwitch
\ECHead{Electronic Companion}
This EC includes proofs and supporting results for system identification (Section~\ref{sec:identification}) and regret analysis for \textsc{UCB-BOLD} (Section~\ref{sec:policy_optimization}).

\section{MLE-related system identification supporting results}\label{app:GLM_identification}
The following results support the MLE-portion of our identification result. Let $p_{i,t} = \sigma(x_t+\mu_i)$ and $p_{i,t}' = \frac{\partial p_{i,t}}{\partial \mu_i}=p_{i,t}(1-p_{i,t})$. Define the per-step log-likelihood $\ell_{i,t}(\mu_i) := \mathbf{1}_{\{u_t=e_i\}}[d^i_t \log(p_{i,t}) + (1-d^i_t)\log(1-p_{i,t})]$, the per-step score function $s_{i,t}(\mu_i) := \frac{\partial}{\partial \mu_i}\ell_{i,t}(\mu_i)=\mathbf{1}_{\{u_t=e_i\}}[d^i_t-p_{i,t}]$, and the per-step second derivative of the log-likelihood $h_{i,t}(\mu_i):=\frac{\partial^2}{\partial \mu_i^2}\ell_{i,t}(\mu_i)=-\mathbf{1}_{\{u_t=e_i\}}p_{i,t}(1-p_{i,t})$. $S_{i}(\mu_i) := \sum_{t=1}^Ts_{i,t}(\mu_i)$, $H_i(\mu_i)=\sum_{t=1}^Th_{i,t}(\mu_i)$, and $\ell_i(\mu_i)=\sum_{t=1}^T\ell_{i,t}(\mu_i)$ are the trajectory-summed quantities. Recall the MLE estimator $\hat \mu_i=\arg\max_{\mu_i \in [-\bar \mu, \bar \mu]}\ell_i(\mu_i)$. $N_i:=\sum_{t=1}^T\mathbf{1}_{\{u_t=e_i\}}$. Lemmas~\ref{lem:score_bound}-~\ref{lem:n_event} assume Assumptions~\ref{asm:observed_quantities}-\ref{asm:noise_properties} hold and actions come from an $(r,k)$-exploratory policy.
\begin{lemma}{\textbf{(Score function bound)}}\label{lem:score_bound}
Let $\delta\in(0,1)$, then $\pr(\vert S_i(\mu_i^*)\vert \leq \sqrt{2T\log\bigl(\nicefrac{2}{\delta}\bigr)}) \geq 1-\delta$. 
\end{lemma}
\proof{Proof of Lemma~\ref{lem:score_bound}}
Consider $\mu_i \in [-\bar \mu, \bar \mu]$. Let $p_{i,t}^*= \sigma(x_t+\mu_i^*)$. Define $\xi_t=s_{i,t}(\mu_i^*)$ and observe that $\{\xi_t\}_{t=1}^T$ is an $\mac F^d_{t+1}$-adapted martingale difference sequence since $x_t,u_t,d^i_t \in \mac F^d_{t+1}$ and $\ex[\xi_t\vert \mac F^d_{t}]=0$ by \eqref{eq:dynamics_adh} as $\xi_t$ is defined using $\mu_i^*$. Note that $ \xi_t \in [-p^*_{i,t},1-p^*_{i,t}]$ (implying $\vert \xi_t\vert \leq 1$) and $\sum_{t=1}^T\xi_t=S_i(\mu_i^*)$. We apply Azuma-Hoeffding \citep{wainwrightHighdimensionalStatisticsNonasymptotic2019}: $\pr(\vert S_i(\mu_i^*)\vert \geq q) \leq 2\exp\left(-\nicefrac{q^2}{(2\sum_{t=1}^T1^2)}\right)$. We set the RHS to $\delta$, solve for $q$, and the result follows.
\halmos
\endproof
\begin{lemma}{\textbf{(Second derivative bound)}}\label{lem:hessian_bound}
$\vert H_i(\mu_i)\vert \geq N_i\underline p^2$ for all $\mu_i \in [-\bar \mu, \bar \mu]$.
\end{lemma}
\proof{Proof of Lemma~\ref{lem:hessian_bound}}
$\vert H_i(\mu_i)\vert =\vert \sum_{t=1}^Th_{i,t}(\mu_i)|= \sum_{t=1}^T\mathbf{1}_{\{u_t=e_i\}}p_{i,t}(1-p_{i,t})$. Remark~\ref{rem:adh_prob_bounded} bounds $p_{i,t}$ and $1-p_{i,t}$ uniformly for $\mu_i \in [-\bar \mu, \bar \mu]$, giving $\vert H_i(\mu_i)\vert \geq N_i\underline p (1-\overline p)=N_i \underline p^2$.
\halmos
\endproof
\begin{lemma}{\textbf{($\mu_i$ bound)}}\label{lem:action_bound}
Let $\delta \in (0,1)$. Then $\pr\bigl(\vert \hat \mu_i - \mu_i^*\vert \leq \frac{\sqrt{2T\log(\frac{2}{\delta})}}{N_i\underline p^2}\bigr) \geq 1-\delta$.
\end{lemma}
\proof{Proof of Lemma~\ref{lem:action_bound}}
$S_i(\mu_i)$ is a smooth function of $\mu_i \in [-\bar \mu, \bar \mu]$, so by the mean value theorem, there exists some $\tilde \mu$ between $\hat \mu_i$ and $\mu^*_i$ such that $H_i(\tilde \mu)(\hat \mu_i - \mu_i^*)=S_i(\hat \mu_i)-S_i(\mu^*_i)$. The first order optimality condition for $\hat \mu_i \in [-\bar \mu,\bar \mu]$ gives $S_i(\hat \mu_i)(\hat \mu_i - \mu_i^*)\geq 0$ (i.e., the inequality is tight for $\hat \mu_i\in(-\bar \mu,\bar \mu)$ and the terms have the same signs at the boundaries). Multiplying the MVT expression by $(\hat \mu_i - \mu_i^*)$ and dropping $S_i(\hat \mu_i)(\hat \mu_i - \mu_i^*)\geq 0$ gives $H_i(\tilde \mu)(\hat \mu_i - \mu_i^*)^2 \geq -S_i(\mu_i^*)(\hat \mu_i - \mu_i^*)$. $H_i \leq 0$ by definition so $-H_i(\tilde\mu) = |H_i(\tilde\mu)|$, implying $|H_i(\tilde \mu)||\hat \mu_i-\mu_i^*|^2 \leq |S_i(\mu_i^*)||\hat \mu_i - \mu_i^*|$. If $\hat \mu_i \neq \mu_i^*$, then simplifying gives $\vert\hat \mu_i - \mu^*_i\vert \leq \frac{\vert S_i(\mu_i^*)\vert }{\vert H_i(\tilde \mu)\vert}$ (the same holds trivially for $\hat \mu_i=\mu_i^*$ as the RHS is non-negative). The numerator is upper bounded by Lemma~\ref{lem:score_bound} with probability at least $1-\delta$ and the denominator is lower bounded by Lemma~\ref{lem:hessian_bound}, so the lemma follows by substitution.
\halmos
\endproof
\begin{lemma}{\textbf{($\mac E_N$ bound)}}\label{lem:n_event} Let $\mac E_N(r,T,k)$ be the event that $\{N_i > \frac{r(T-k)}{2(M+1)}\text{ for all } i=1,\dots,M\}$. Then for $\delta \in (0,1)$: $\pr(\mac E_N(r,T(\delta),k) ) \geq 1-\delta$ when $T(\delta) \geq k + \nicefrac{8}{r}(M+1)\log(\nicefrac{M}{\delta})$.
\end{lemma}
\proof{Proof of Lemma~\ref{lem:n_event}}
Consider $i \in \{1,\dots,M\}$. The $(r,k)$ policy a.s. takes at least $N:=\lfloor \frac{T}{k}\rfloor\lceil r k\rceil\geq r(T-k)$ exploratory steps in the trajectory. Let $\tilde N_i$ be the count of steps where $u_t=e_i$ in the first $N$ exploratory steps. By the $(r,k)$ policy's uniformly random exploration, $\tilde N_i\sim \text{Bin}(N,\frac{1}{M+1})$, so by the multiplicative form of the Chernoff bound for sums of Bernoulli random variables \citep{gerbessiotisSurveyChernoffHoeffding2025}, 
we have $\pr(\tilde N_i \leq \frac{N}{2(M+1)})\leq \exp\{-\frac{N}{8(M+1)}\} \leq \exp\{-\frac{r(T-k)}{8(M+1)}\}$. As $N_i\geq \tilde N_i$ by construction and $N\geq r(T-k)$ a.s., this implies $\{N_i \leq \frac{r(T-k)}{2(M+1)} \}\subseteq \{\tilde N_i \leq \frac{N}{2(M+1)}\}$ and thus $\pr(N_i \leq \frac{r(T-k)}{2(M+1)}) \leq \exp\{-\frac{r(T-k)}{8(M+1)}\}$. Applying the union bound yields $\pr(\mac E_N(r,T,k) ) \geq 1- M\exp\{-\frac{r(T-k)}{8(M+1)}\}$. Setting the RHS equal to $\delta$ and solving for $T(\delta)$ completes the proof. 
\halmos
\endproof

\section{System identification result proof}\label{app:ec_sys_id_result}
\proof{Proof of Theorem~\ref{thm:system_sample_complexity}}
We analyze the OLS and MLE results separately (allocating $\nicefrac{\delta}{2}$ to each) and then apply a union bound. Let $\delta_1= \nicefrac{\delta}{6}$ (used for the OLS portion), $\delta_2 = \nicefrac{\delta}{4}$ (used for $\mac E_N$), and $\delta_3=\nicefrac{\delta}{(4M)}$ (used for the MLE portion).
We evaluate the preconditions in Theorem~\ref{thm:simch_sample_complexity} using $(Z_t,Y_t)_{t\geq 1}=(z_t,x_{t+1})_{t\geq 1}$, filtration $\{\mac F^w_t\}_{t\geq 0}$, and $\eta_t = w_t$. Precondition (a) is satisfied trivially under Assumption~\ref{asm:noise_properties} for $x_{t+1} = \theta^{*\top} z_t +w_t$. By Proposition~\ref{prop:bmsb_formal}, for the given $(r,k,\tau)$, the system satisfies the $(k,\frac{1-\tau^2}{8M}I_{2M+1},\min\{p_1(\tau),\frac{r\underline p}{M+1}\})$-BMSB condition, establishing precondition (b). The system satisfies precondition (c) by Proposition~\ref{prop:ub-cond-formal} with $\bar \Gamma = (C_x^2+2)I_{2M+1}$. The OLS result then follows by substituting the system-specific quantities into Theorem~\ref{thm:simch_sample_complexity} (using $\delta_1$ gives an OLS-related failure probability of no more than $\nicefrac{\delta}{2}$). Theorem~\ref{thm:simch_sample_complexity} uses a factor of $3\delta$, but as our precondition (c) is deterministic this could be refined to $2\delta$ (omitted for brevity). For the MLE result, we first have that $\pr(\mac E_N(r,T(\delta_2),k)) \geq 1-\delta_2$ when $T\geq k + \nicefrac{8(M+1)}{r}\log(\nicefrac{4M}{\delta})$ by Lemma~\ref{lem:n_event}. By Lemma~\ref{lem:action_bound}, for index $i\in\{1,\dots,M\}$ and $\delta_3$, we have $\pr(\vert \hat \mu_i - \mu_i^*\vert \leq \frac{\sqrt{2T\log(\nicefrac{8M}{\delta})}}{N_i\underline p^2}) \geq 1-\delta_3$. On $\mac E_N(r,T(\delta_2),k)$, $N_i>\frac{r(T-k)}{2(M+1)}$ for all $i=1,\dots,M$, meaning $\vert \hat \mu_i - \mu_i^*\vert \leq \frac{2(M+1)}{\underline p^2}\frac{\sqrt{2T\log(\nicefrac{8M}{\delta})}}{r(T-k)}$. Applying the union bound ensures this holds for all $i=1,\dots,M$ with probability at least $1-\nicefrac{\delta}{4}$. For bound clarity, the assumed $T\geq 2k$ allows lower bounding $T-k$ by $\nicefrac{T}{2}$, giving the upper bound $\frac{4(M+1)}{\underline p^2}\frac{\sqrt{2\log(\nicefrac{8M}{\delta})}}{r\sqrt{T}}$ used in the result. A union bound ensures the MLE result holds with probability at least $1-\nicefrac{\delta}{2}$. The complete result then holds by an additional union bound and taking the maximum of the burn-in times. 
\halmos
\endproof

\section{Regret analysis supporting results and proofs}\label{app:control}
\subsection{Confidence set results}\label{app:ec_confidence_sets}
\begin{theorem}\label{thm:gen_confidence_set}
    (\textcite{abbasi-yadkoriImprovedAlgorithmsLinear2011} Theorem 2) Let $\{\mac F_t\}_{t=0}^{\infty}$ be a filtration. Let $\{\eta_t\}_{t=1}^\infty$ be an $\mac F_t$-adapted random process taking value in $\R$, with $\eta_t|\mac F_{t-1}$ $R$-subgaussian. Let $\{z_t\}_{t=1}^\infty$ be an $\R^d$-valued stochastic process such that $z_t$ is $\mac F_{t-1}$ measurable. Let $\lambda>0$ and define $y_t = \langle z_t,\theta_*\rangle+\eta_t$. If $\Vert\theta_*\Vert_2 \leq S$ and, for all $t\geq 1$, $\Vert z_t\Vert_2 \leq \bar Z $, then with probability at least $1-\delta$ for all $t\geq 0$, $\theta^*$ lies in the following set $\mac C_t(\delta) = \{ \theta \in \R^d : \|\hat{\theta}^{RLS}_t - \theta\|_{\overline{V}_t} \leq R \sqrt{d\log(\frac{1 +t\bar Z^2/\lambda}{\delta})} +\sqrt{\lambda}S  \}$.
\end{theorem}
\proof{Proof of Proposition~\ref{prop:theta_confidence_set}}
We construct a confidence set centered on $\hat \theta_t^{RLS}$ then re-center to $\hat \theta_t$. Let $y_t=x_{t}$, $z_t=[x_{t-1} \ u_{t-1}^\top \ d_{t-1}^\top]^\top$, and $\eta_t=w_{t-1}$, matching $y_t = \langle z_t, \theta^* \rangle + \eta_t$ in Theorem~\ref{thm:gen_confidence_set}. Recall the filtration $\mac F^w_t$, containing $x_t,u_t,d_t$ but not $w_t$. Thus $\eta_t$ is $\mac F^w_{t}$-measurable and is $\mac F^w_{t-1}$ conditionally $\sigma_s^2$-subgaussian by Assumption~\ref{asm:noise_properties}. $\{z_t\}^\infty_{t=1}$ is a $\mac F^w_{t-1}$-adapted $\R^{2M+1}$ valued stochastic process. $\Vert z_t \Vert_2 \leq \sqrt{C_x^2 +2}$ for $t\in \mathbb N$ (Prop.~\ref{prop:absolute_bounded_x}) and $\max_{\theta \in \Theta}\Vert \theta \Vert_2 \leq \sqrt{\bar a^2 + M \bar b^2 + M\bar c^2}$ (Remark~\ref{rem:parameter_space_bounded}). Thus our system satisfies the conditions of Theorem~\ref{thm:gen_confidence_set}. Let $\delta\in(0,1)$. Substituting system-specific quantities implies $\pr(\mac E_\theta^{RLS}(\delta)) \geq 1-\delta$ for $\mac E_{\theta}^{RLS}(\delta):=\{\theta^* \in \mac C_t^{\theta,RLS}(\delta),\; t\geq 0\}$ and $\mac C_t^{\theta,RLS}(\delta) := \big\{ \theta \in \R^{2M+1} : \|\hat{\theta}^{RLS}_t - \theta\|_{\overline{V}_t} \leq \alpha^\theta_t(\delta) \big\}$. By Lemma~\ref{lem:est_bound}, if $\mac E_\theta^{RLS}(\delta)$ holds then $ \Vert \hat \theta_t - \theta^*  \Vert_{\bar V_t} \leq \alpha_t^\theta(\delta)$ for $t\geq 1$. Thus the proposition follows for the confidence set centered at $\hat \theta_t$.
\halmos
\endproof

\begin{lemma}{}\label{lem:est_bound}
Let event $\mac E^{RLS}_\theta(\delta)$ hold. Then $ \Vert \hat \theta_t - \theta^*  \Vert_{\bar V_t} \leq \Vert \hat \theta_t^{RLS} - \theta^* \Vert_{\bar V_t} \leq \alpha_t^\theta(\delta)$ for $t\geq 1$.
\end{lemma}
\proof{Proof of Lemma~\ref{lem:est_bound}}
Let $t\geq 1$, define $\langle a,b \rangle_{\bar V_t} := a^\top \bar V_t b$ and note $\Vert a \Vert_{\bar V_t}^2=\langle a,a\rangle_{\bar V_t}$. $\hat \theta_t$ and $\theta^*\in \Theta$ by construction. We consider their convex combination $\tilde \theta := \hat \theta_t + q(\theta^* - \hat \theta_t)$ for $q\in(0,1]$. $\Theta$ is a convex set so $\tilde \theta\in \Theta$. Since $\hat \theta_t=\arg\min_{\theta \in \Theta} \Vert \theta - \hat \theta^{RLS}_t\Vert_{\bar V_t}^2$, we have that $\Vert\tilde \theta -\hat \theta_t^{RLS} \Vert_{\bar V_t}^2 \geq \Vert \hat \theta_t-\hat \theta_t^{RLS} \Vert_{\bar V_t}^2$. By expanding the LHS norm we can simplify this inequality to $q^2\Vert\theta^*- \hat \theta_t \Vert^2_{\bar V_t} + 2q\langle\hat \theta_t-\hat \theta_t^{RLS} , \theta^*- \hat \theta_t\rangle_{\bar V_t} \geq 0$. Dividing by $q$ and taking $\lim_{q\rightarrow0^+}$ we obtain the inequality $\langle\hat \theta_t^{RLS} -\hat \theta_t,\hat \theta_t - \theta^*  \rangle_{\bar V_t} \geq 0$. Next we expand the quantity $\Vert \hat\theta_t^{RLS} - \theta^* \Vert_{\bar V_t}^2 = \Vert ( \hat\theta_t^{RLS} - \hat\theta_t ) + (\hat \theta_t-\theta^* )\Vert_{\bar V_t}^2= \Vert \hat \theta_t^{RLS} - \hat \theta_t \Vert_{\bar V_t}^2 + \Vert \hat \theta_t-\theta^* \Vert_{\bar V_t}^2+2\langle \hat \theta_t^{RLS} - \hat \theta_t, \hat \theta_t-\theta^*\rangle_{\bar V_t}$. The first and third terms are nonnegative, implying $\Vert \hat \theta_t - \theta^* \Vert^2_{\bar V_t} \leq \Vert \hat \theta_t^{RLS} - \theta^* \Vert^2_{\bar V_t}$. We assume $\mac E_{\theta}^{RLS}(\delta)$ holds, so $\Vert \hat \theta_t - \theta^* \Vert_{\bar V_t} \leq \Vert \hat \theta_t^{RLS} - \theta^* \Vert_{\bar V_t}\leq \alpha^{\theta}_t(\delta)$.
\halmos
\endproof
\begin{theorem}{\textbf{(Scalar form of \textcite{fauryImprovedOptimisticAlgorithms2020} Theorem 1)}}\label{thm:faury_tail}
Let $\{\mac F_t\}_{t=1}^\infty$ be a filtration. Let $\{x_t\}_{t=1}^\infty$ be a stochastic process in $[-1,1]$ such that $x_t$ is $\mac F_t$ measurable. Let $\{\epsilon_t\}_{t=1}^\infty$ be a martingale difference sequence such that $\epsilon_{t}$ is $\mac F_{t+1}$ measurable. Assume that conditionally on $\mac F_t$ we have $\vert \epsilon_{t}\vert \leq 1$ almost surely, and note $\sigma^2_t:=\ex[\epsilon_{t}^2\vert \mac F_t ]$. Let $\lambda >0$ and for any $t\geq 1$ define $H_t:= \sum_{s=1}^{t-1}\sigma_s^2x_s^2 + \lambda$ and $S_t:=\sum_{s=1}^{t-1}\epsilon_{s}x_s$. Then for any $\delta\in(0,1]$: $\pr(\exists t \geq 1,\; \frac{\vert S_t\vert}{\sqrt{H_t}}\geq \frac{\sqrt{\lambda}}{2}+\frac{2}{\sqrt{\lambda}}\log (\frac{2\sqrt{H_t}}{\delta \sqrt{\lambda}} )) \leq \delta$.
\end{theorem}
\proof{Proof of Proposition~\ref{prop:mu_confidence_set}}
We focus on arbitrary action $i$ and then apply the union bound. First, we adapt Theorem~\ref{thm:faury_tail} into a bound in terms of the regularized score function $g_{i,t}(\mu_i):= \frac{\partial \ell_{i,t}(\mu_i)}{\partial \mu_i}=-\lambda_2 \mu_i +\sum_{s=1}^{t-1}\mathbf{1}_{\{u_s=e_i\}}[d^i_s - \sigma(x_s+\mu_i)]$, where $\ell_{i,t}(\mu_i)$ is the regularized log likelihood for $\mu_i$ on the trajectory $1,\dots,t-1$, as given in \eqref{eq:regloglik}. Then we convert this result into a confidence set for $\mu_i$. We show how our estimation problem maps to Theorem~\ref{thm:faury_tail}. Recall the filtration $\mac F^d_t$ (containing $x_t,u_t$ but not $d_t$). Define a stochastic process $\{q_t\}_{t=1}^\infty$ with $q_t=1\;\forall t$. $q_t\in[-1,1]$ and is $\mac F^d_t$-measurable as it is deterministic. $q_t$ is $x_t$ in Theorem~\ref{thm:faury_tail}'s notation, which we set to $1$ as we are estimating a shift parameter, not a scale parameter. Define $\{\epsilon_t\}_{t=1}^\infty$ where $\epsilon_t:=\mathbf{1}_{\{u_t=e_i\}}[d^i_t-\sigma(x_t+\mu^*_i\cdot q_t)]$. Note that this is a martingale difference sequence that is $\mac F^d_{t+1}$-measurable. $\vert \epsilon_t \vert \leq 1$ by the model assumptions. Note that $\ex[\epsilon_t^2\vert \mac F^d_t]=\mathbf{1}_{\{u_t=e_i\}}[\sigma'(x_t+\mu_i^*)]$. Following \textcite{fauryImprovedOptimisticAlgorithms2020}, let $S_t:= \sum_{s=1}^{t-1}\epsilon_s q_s = \sum_{s=1}^{t-1}\mathbf{1}_{\{u_s=e_i\}}[d^i_s - \sigma(x_s+\mu_i^*)]$ and $H_t :=  \lambda_2 +\sum_{s=1}^{t-1}\ex[\epsilon^2_s\vert \mac F^d_s] = \lambda_2 +\sum_{s=1}^{t-1}\mathbf{1}_{\{u_s=e_i\}}\sigma'(x_s+\mu_i^*)$. These are constructed to connect to our system as $S_t=g_{i,t}(\mu_i^*)+\lambda_2 \mu_i^*$ and $H_t = -\frac{\partial g_{i,t}(\mu_i)}{\partial \mu_i}\vert_{\mu_i=\mu_i^*}$. We first consider the unconstrained MLE $\hat \mu_{i,t}^{MLE}$, for which $g_{i,t}(\hat \mu_{i,t}^{MLE})=0$. We want to form a quantity to bound by Theorem~\ref{thm:faury_tail}. Consider $g_{i,t}(\hat \mu_{i,t}^{MLE}) - g_{i,t}(\mu_i^*) = -S_t + \lambda_2 \mu_i^*$. Using the triangle inequality, dividing by $\sqrt{H_t}$, and applying Theorem~\ref{thm:faury_tail} to $\frac{\vert S_t \vert}{\sqrt{H_t}}$, we have with probability at least $1-\nicefrac{\delta}{M}$:
\begin{align}
    \tfrac{\vert g_{i,t}(\hat \mu_{i,t}^{MLE}) - g_{i,t}(\mu_i^*) \vert }{\sqrt{H_t}} &\leq \tfrac{\vert S_t \vert}{\sqrt{H_t}} +\tfrac{\vert \lambda_2 \mu_i^*\vert}{\sqrt{H_t}}
    \leq \tfrac{\vert S_t \vert}{\sqrt{H_t}} +\tfrac{\lambda_2 \bar \mu}{\sqrt{H_t}}
    \leq \tfrac{\sqrt{\lambda_2}}{2}+\tfrac{2}{\sqrt{\lambda_2}}\log (\tfrac{2M\sqrt{H_t}}{\delta \sqrt{\lambda_2}} ) + \tfrac{\lambda_2 \bar \mu}{\sqrt{H_t}}\label{eq:mle_bound}
\end{align}
Next, $g_{i,t}(\mu_i)$ is a smooth function of $\mu_i$, so the mean value theorem implies there exists $\tilde \mu$ between $\mu_i^*$ and the projected estimator $\hat \mu_{i,t}\in[-\bar \mu, \bar \mu]$ such that: $\frac{\partial g_{i,t}(\mu_i)}{\partial \mu_i}\vert_{\mu_i=\tilde \mu}(\hat \mu_{i,t}-\mu_i^*)= g_{i,t}(\hat \mu_{i,t})-g_{i,t}(\mu_i^*)$. For brevity, let $H_t(\mu) :=\frac{-\partial g_{i,t}(\mu_i)}{\partial \mu_i}= \lambda_2 + \sum_{s=1}^{t-1}\mathbf{1}_{\{u_s=e_i\}}\sigma'(x_s+\mu)$, so taking absolute values and rearranging (note $H_t(\mu)>0$) gives $\vert \hat \mu_{i,t} - \mu_i^*\vert = \frac{\vert g_{i,t}(\hat \mu_{i,t})-g_{i,t}(\mu_i^*)\vert}{\left\vert H_t(\tilde \mu)\right\vert}$. We upper bound the numerator to allow use of the bound in \eqref{eq:mle_bound} and lower bound the denominator in terms of known quantities. For the numerator, observe that $g'_{i,t}(\mu_i) = -\lambda_2 -\sum_{s=1}^{t-1}\mathbf{1}_{\{u_s=e_i\}}\sigma'(x_s+\mu_i)<0$, implying that $g_{i,t}(\mu_i)$ is decreasing in $\mu_i$. If $\hat \mu_{i,t}^{MLE}\in [-\bar \mu,\bar \mu]$ then the projected estimator is trivially $\hat \mu_{i,t}=\hat\mu_{i,t}^{MLE}$. Next if $\hat \mu_{i,t}^{MLE} < -\bar \mu$, then $\hat \mu_{i,t}=-\bar \mu$. Since $g_{i,t}(\mu_i)$ is decreasing, $g_{i,t}(\hat \mu_{i,t}^{MLE})> g_{i,t}(\hat \mu_{i,t}) \geq g_{i,t}(\mu_i^*)$, implying $\vert g_{i,t}(\hat \mu_{i,t})-g_{i,t}(\mu_i^*)\vert \leq \vert g_{i,t}(\hat \mu_{i,t}^{MLE})-g_{i,t}(\mu_i^*)\vert$. By a symmetrical argument the same holds when $\hat \mu_{i,t}^{MLE} > \bar \mu$, meaning the numerator can be upper bounded using \eqref{eq:mle_bound}. While $\tilde \mu$ in the denominator is unknown, since $\hat \mu_{i,t}, \mu_i^* \in [-\bar \mu, \bar \mu]$ it follows that $\vert \tilde \mu - \hat \mu_{i,t}\vert \leq 2\bar \mu$. By the sigmoid's self-concordance property (\textcite{fauryImprovedOptimisticAlgorithms2020} Lemma 9), $\sigma'(x+\hat \mu_{i,t})e^{-2\bar \mu}\leq \sigma'(x+\tilde \mu) \leq \sigma'(x+\hat \mu_{i,t})e^{2\bar \mu} $. The first inequality implies $\vert H_t(\tilde \mu) \vert = H_t(\tilde \mu) \geq \lambda_2 +e^{-2\bar \mu}\sum_{s=1}^{t-1}\mathbf{1}_{\{u_s=e_i\}}\sigma'(x_s+\hat \mu_{i,t})=:\underline {H_t}(\hat \mu_{i,t})$. The second gives $H_t \leq \lambda_2 + e^{2\bar \mu}\sum_{s=1}^{t-1}\mathbf{1}_{\{u_s=e_i\}}\sigma'(x_s+\hat \mu_{i,t})=:\overline {H_t}(\hat \mu_{i,t})$. Thus with probability at least $1-\nicefrac{\delta}{M}$:
\begin{align}
    \vert \hat \mu_{i,t} - \mu_i^*\vert 
    &
    \overset{(a)}{\leq} \tfrac{\sqrt{H_t}}{\vert H_t(\tilde\mu)\vert}\tfrac{\vert g_{i,t}(\hat \mu_{i,t}^{MLE}) - g_{i,t}(\mu_i^*) \vert}{\sqrt{H_t}}
    \overset{(b)}{\leq} \tfrac{\sqrt{\overline {H_t}(\hat \mu_{i,t})}}{\underline {H_t}(\hat \mu_{i,t})}[\tfrac{\sqrt{\lambda_2}}{2}+\tfrac{2}{\sqrt{\lambda_2}}\log (\tfrac{2M\sqrt{\overline {H_t}(\hat \mu_{i,t})}}{\delta \sqrt{\lambda_2}} )]+\tfrac{\lambda_2 \bar \mu}{\underline {H_t}(\hat \mu_{i,t})} \\
    &\overset{(c)}{\leq} \tfrac{e^{3\bar \mu}}{\sqrt{H_t(\hat \mu_{i,t})}}[\tfrac{\sqrt{\lambda_2}}{2}+\tfrac{2\bar\mu}{\sqrt{\lambda_2}}+\tfrac{2}{\sqrt{\lambda_2}}\log (\tfrac{2M\sqrt{H_t(\hat \mu_{i,t})}}{\delta \sqrt{\lambda_2}} )]+\tfrac{e^{2\bar \mu}\lambda_2 \bar \mu}{H_t(\hat \mu_{i,t})}
\end{align}
$(a)$ multiplies by $\nicefrac{\sqrt{H_t}}{\sqrt{H_t}}$ so that $(b)$ can substitute $H_t$ bounds and \eqref{eq:mle_bound}. $(c)$ loosens the bound for clarity, specifically using $\underline {H_t}\geq e^{-2\bar \mu}[\lambda_2 + \sum_{s=1}^{t-1}\mathbf{1}_{\{u_s=e_i\}}\sigma'(x_s+\hat\mu_{i,t})]=e^{-2\bar \mu}H_t(\hat \mu_{i,t})$, $\overline {H_t}\leq e^{2\bar \mu}[\lambda_2 + \sum_{s=1}^{t-1}\mathbf{1}_{\{u_s=e_i\}}\sigma'(x_s+\hat\mu_{i,t})]=e^{2\bar \mu}H_t(\hat \mu_{i,t})$, and pulling $e^{\bar \mu}$ (from substituting the $\overline {H_t}$) out of the log. The result follows by a union bound over the actions.
\halmos
\endproof

\subsection{Value function supporting analysis}\label{app:ec_value_function}
Let Assumptions~\ref{asm:observed_quantities}-\ref{asm:rewards} hold for the following supporting results. Here we provide proofs for properties of the relevant operators (Lemma~\ref{lem:unique_opt}), boundedness of the bonus (Lemma~\ref{lem:bounded_bonus}), $J^*(x)$ and $\tilde J_{t,\delta}(x)$ Lipschitz continuity (Lemmas~\ref{lem:j*_xlip}-\ref{lem:tj_xlip}), and valid optimism (Lemma~\ref{lem:valid_optimism}). 

\proof{Proof of Lemma~\ref{lem:unique_opt}}
$r(x,u,d)$ is bounded by construction, $b_{t,\delta}(x,u)$ is bounded by $C_b(T,\delta)$ (Lemma~\ref{lem:bounded_bonus}) for $t\leq T$, and rewards are discounted (Assumption~\ref{asm:rewards}), so contraction Assumption~1.6.1 in \textcite{bertsekasDynamicProgrammingOptimal2012} is satisfied and repeated applications of the operator converge uniformly to unique fixed points $J^*$ and $\tilde J_{t,\delta}$ of these operators, respectively, by Proposition~1.6.1 \citep{bertsekasDynamicProgrammingOptimal2012}. These operators also satisfy monotonicity Assumption~1.6.2 (by Lemma 1.1.1), implying that they are the optimal value functions by Proposition~1.6.2 \citep{bertsekasDynamicProgrammingOptimal2012}. $J^*(x)$ is clearly bounded by $\nicefrac{(\bar \rho +\beta)}{(1-\gamma})$ by the bounded construction of $r(x,u,d)$ (meaning $J^\pi$ is identically bounded). $\tilde J_{t,\delta}(x)$ is bounded by an identical argument using the boundedness of the bonus.\halmos
\endproof

\begin{lemma}{}\label{lem:bounded_bonus}
$\Vert b_{t,\delta}(x,u) \Vert_\infty \leq \frac{(\bar \rho + \gamma L_1\bar c)\overline \alpha^\mu(\nicefrac{\delta}{2})}{4} + \frac{\gamma L_1\alpha^\theta_T(\nicefrac{\delta}{2})\sqrt{C_x^2+2}}{\sqrt{\lambda_1}} $ for $t\in\{1,\dots,T\}$, $\delta\in(0,1)$.
\end{lemma}
\proof{Proof of Lemma~\ref{lem:bounded_bonus}}
The bonus is non-negative, so we show an upper bound. Let $x \in \mac X$, $u\in \mac U$, $\delta \in (0,1)$, $t\in \{1,\dots,T\}$. As $\alpha^\theta_t$ is increasing in $t$, $\alpha^\theta_t(\nicefrac{\delta}{2}) \leq \alpha^\theta_T(\nicefrac{\delta}{2})$. $\bar V_0^{-1} \succeq \bar V_t^{-1}$ by construction. Recall $\alpha^\mu_{i,t}(\nicefrac{\delta}{2}):=\frac{e^{3\bar \mu}}{\sqrt{H_t(\hat \mu_{i,t})}}[\frac{\sqrt{\lambda_2}}{2}+\frac{2}{\sqrt{\lambda_2}}( \bar \mu +\log (\frac{4M\sqrt{H_t(\hat \mu_{i,t})}}{\delta \sqrt{\lambda_2}} ))]+\frac{e^{2\bar \mu}\lambda_2 \bar \mu}{H_t(\hat \mu_{i,t})}$ for $H_t(\hat\mu_{i,t})= \lambda_2 + \sum_{s=1}^{t-1}\mathbf{1}_{\{u_s=e_i\}}\sigma'(x_s+\hat\mu_{i,t})$. Separating the log term we equivalently have $\alpha_{i,t}^\mu(\nicefrac{\delta}{2})=\frac{e^{3\bar \mu}}{\sqrt{H_t(\hat \mu_{i,t})}}[\frac{\sqrt{\lambda_2}}{2}+\frac{2}{\sqrt{\lambda_2}}( \bar \mu +\log (\frac{4M}{\delta \sqrt{\lambda_2}} ))] + \frac{2e^{3\bar \mu}\log(\sqrt{H_t(\hat \mu_{i,t})})}{\sqrt{\lambda_2}\sqrt{H_t(\hat \mu_{i,t})}}+\frac{e^{2\bar \mu}\lambda_2 \bar \mu}{H_t(\hat \mu_{i,t})}$. Using the bounds $H_t(\hat \mu_{i,t}) \geq \lambda_2$ and $\frac{\log x}{x}\leq e^{-1}$ we have $\alpha_{i,t}^\mu(\nicefrac{\delta}{2}) \leq \frac{e^{3\bar \mu}}{\sqrt{\lambda_2}}[\frac{\sqrt{\lambda_2}}{2}+\frac{2}{\sqrt{\lambda_2}}(\bar \mu+ \log (\frac{4M}{\delta \sqrt{\lambda_2}} ))]+\frac{2e^{3\bar \mu -1}}{\sqrt{\lambda_2}} +e^{2\bar \mu}\bar \mu=:\overline \alpha^\mu(\nicefrac{\delta}{2})$.
Thus:
\begin{align}
    b_{t,\delta}(x,u) &:= (\bar \rho + \gamma L_1\bar c)\textstyle\sum\nolimits_{i=1}^Mu^i\kappa(x,\hat \mu_{i,t})\alpha_{i,t}^\mu(\nicefrac{\delta}{2}) + \gamma L_1 \alpha^\theta_t(\nicefrac{\delta}{2})\ex_{\hat \mu_t}[ \Vert z  \Vert_{\bar V_t^{-1}}\vert x,u]\\
    &\leq (\bar \rho + \gamma L_1\bar c)(\nicefrac{1}{4})\overline \alpha^\mu(\nicefrac{\delta}{2}) + \gamma L_1\alpha^\theta_T(\nicefrac{\delta}{2})(\nicefrac{\sqrt{C_x^2+2}}{\sqrt{\lambda_1}})=:C_b(T,\delta)\Halmos
\end{align}
\endproof

\begin{lemma}{}\label{lem:j*_xlip}
$J^*(x)$ is $L_1$-Lipschitz continuous in $x$ on $\mac X$ with $L_1=\frac{\beta+\bar \rho}{4(1-\gamma \bar a)}(1+\frac{2\gamma}{1-\gamma})$.
\end{lemma}
\proof{Proof of Lemma~\ref{lem:j*_xlip}}
By definition, $\mac T J(x) = \max_{u \in \mac U}  \{\ex_{\theta^*,\mu^*}[r(x,u,d) +\gamma J(f(x,u,d,w))] \}$. By Lemma~\ref{lem:reward_lip}, $\ex_{\theta^*,\mu^*}[r(x,u,d)]$ is uniformly Lipschitz continuous in $x$ with constant $L_r =\nicefrac{(\beta+\bar \rho)}{4}$. By Lemma~\ref{lem:valuetogo_lip_loose}, if a function $J(x)$ is $L$-Lipschitz continuous in $x$ and bounded by $\bar J$ then $\gamma \ex_{\theta^*,\mu^*}[J(f(x,u,d,w)]$ is uniformly Lipschitz continuous in $x$ with constant $L_v=\gamma (L\bar a + \nicefrac{\bar J}{2})$. By Lemma~\ref{lem:unique_opt}, $\bar J = \frac{\beta +\bar \rho}{1-\gamma}$. This implies that $\mac T J(x)$ is $L'$-Lipschitz with $L' =  \nicefrac{(\beta+\bar \rho)}{4}+\gamma L\bar a +\gamma\frac{\beta +\bar \rho}{2(1-\gamma)}$. Solving for the smallest $L'$ such that $L'\geq L$ gives $L_1=\frac{\frac{\beta+\bar \rho}{4}(1+\frac{2\gamma}{1-\gamma})}{1-\gamma\bar a}$. Note that this denominator is strictly positive as $\gamma \in (0,1)$ and $\bar a <1$ by assumption. Let $J_0(x)$ be the zero function, which is trivially $L_1$-Lipschitz. By induction on $n$, we observe that $J_n = \mac T^n J_0$ is also $L_1$-Lipschitz for all $n\in \mathbb N$. By Lemma~\ref{lem:unique_opt}, $J_n$ converges uniformly to $J^*$. Thus for every $\epsilon>0$ there exists $N \in \mathbb N$ such that $|J^*(x)-J_n(x)| < \nicefrac{\epsilon}{2}$ for all $x\in \mac X$ and $n\geq N$. For $x_1,x_2\in \mac X$: $\left \vert J^*(x_1)-J^*(x_2)\right \vert \leq \left \vert J^*(x_1)-J_N(x_1) \right \vert +\left \vert J_N(x_1)-J_N(x_2) \right \vert+\left \vert J_N(x_2)-J^*(x_2) \right \vert\leq \epsilon +L_1\left \vert x_1-x_2\right \vert$. This holds for any $\epsilon >0$, implying $\left \vert J^*(x_1)-J^*(x_2)\right \vert \leq  L_1\left \vert x_1-x_2\right \vert$ $\Leftrightarrow$ $J^*$ is $L_1$-Lipschitz.\halmos
\endproof

\begin{lemma}{}\label{lem:tj_xlip}
$\tilde J_{t,\delta}$ is $L_2$-Lipschitz continuous in $x$ on $\mac X$ for $t\in\{1,\dots,T\}$ with $L_2 = \frac{1}{1-\gamma\bar a}\big[\nicefrac{(\beta+\bar\rho)}{4}+\nicefrac{e^{2\bar \mu}}{4}(\bar \rho + \gamma L_1\bar c)\overline \alpha^\mu(\nicefrac{\delta}{2})+\frac{5\gamma L_1\alpha^\theta_T(\nicefrac{\delta}{2})}{4\sqrt{\lambda_1}}+\nicefrac{\gamma}{2}C_{\tilde J,T,\delta}\big]$.
\end{lemma}
\proof{Proof of Lemma~\ref{lem:tj_xlip}}
We follow a similar approach to Lemma~\ref{lem:j*_xlip}. Let $t\in\{1,\dots,T\}$, $\hat \theta_t\in \Theta$ and $\hat \mu_t \in \Phi$. Recall $\tilde {\mac T}_{t,\delta} J(x):=\max_{u\in \mac U}\{ \ex_{\hat \theta_t,\hat \mu_t}[r(x,u,d)+b_{t,\delta}(x,u) + \gamma J(f(x,u,d,w))]\}$. By Lemmas~\ref{lem:reward_lip} and \ref{lem:bonus_lip}, $\ex_{\hat\theta_t,\hat \mu_t}[r(x,u,d)]$ and $\ex_{\hat\theta_t,\hat \mu_t}[b_{t,\delta}(x,u)]=b_{t,\delta}(x,u)$ are uniformly $L_r$- and $L_b$-Lipschitz continuous in $x$. By Lemma~\ref{lem:valuetogo_lip_loose}, if $J$ is $L$-Lipschitz continuous in $x$ and bounded by $\bar J$, then $\gamma\ex_{\hat \theta_t,\hat \mu_t}[J(f(x,u,d,w)]$ is uniformly $L_v$-Lipschitz with $L_v = \gamma L\bar a + \nicefrac{\gamma\bar J}{2}$ (here $\| \tilde J_{t,\delta}\|_\infty \leq C_{\tilde J,T,\delta}$ by Lemma~\ref{lem:unique_opt}). The maximum over actions is ($L_r+L_b+L_v)$-Lipschitz. Thus if $J$ is $L$-Lipschitz, $\tilde {\mac T}_{t,\delta}J$ is $L'$-Lipschitz with $L' = L_r+L_b+\gamma L\bar a + \tfrac{\gamma}{2} C_{\tilde J, T,\delta}$. Finding the smallest $L'\geq L$ gives $L_2 = \frac{1}{1-\gamma\bar a}[L_r+L_b+\frac{\gamma}{2}C_{\tilde J,T,\delta}]$. Let $J_0(x)$ be the zero function, which is trivially $L_2$-Lipschitz and bounded by $C_{\tilde J,T,\delta}$. Thus, by an identical inductive argument to that presented in Lemma~\ref{lem:j*_xlip} it follows that $\tilde J_{t,\delta}$ is $L_2$-Lipschitz. The result follows by substituting $L_r$ and $L_b$.\halmos
\endproof

\proof{Proof of Lemma~\ref{lem:valid_optimism}}
Let $t\in \{1,\dots,T+1\}$, $\delta \in (0,1)$, $x \in \mac X$, $\mac E_\theta(\nicefrac{\delta}{2})$, and $\mac E_\mu(\nicefrac{\delta}{2})$ hold. By Lemma~\ref{lem:unique_opt}, $\tilde {\mac T}_{t,\delta}$ is monotone and converges uniformly to its fixed point, so showing that $\tilde{\mac T}_{t,\delta}J^*(x) \geq J^*(x)$ implies $\tilde J_{t,\delta}(x)\geq J^*(x)$. Recall $J^*(x) = \max_{u \in \mac U}\{\ex_{\theta^*,\mu^*}[r(x,u,d)+\gamma J^*(f(x,u,d,w))]\}$. Let $u^*_x$ be the optimal action to the true system at $x$. $u^*_x$ is feasible for the surrogate system, so $\tilde{\mac T}_{t,\delta}J^*(x) \geq \ex_{\hat \theta_t, \hat \mu_t}[r(x,u^*_x,d)+b_{t,\delta}(x,u^*_x)+\gamma J^*(f(x,u^*_x,d,w))]$. Let $\Delta_J(x) := \tilde{\mac T}_{t,\delta}J^*(x)-J^*(x)$. We have:
\begin{multline}
    \Delta_J(x) \geq \ex_{\hat \mu_t}[b_{t,\delta}(x,u^*_x) ]+\ex_{\hat \mu_t}[r(x,u^*_x,d) ] - \ex_{\mu^*}[r(x, u^*_x,d) ]\\
    + \gamma  [\ex_{\hat \mu_t} [J^*(z^\top \hat \theta_t+w)\vert x,u^*_x] -\ex_{\mu^*} [J^*(z^\top \theta^*+w)\vert x,u^*_x]]
\end{multline}
We first assume that $u^*_x=e_i$ for some $i=1,\dots,M$ and address the $u^*_x=\ve{0}_M$ case at the end. In the reward difference, the $\beta$-sigmoid terms cancel leaving the adherence terms: $\ex_{\hat \mu_t}[r(x,u^*_x,d) ] - \ex_{\mu^*}[r(x, u^*_x,d) ] = \sum_{i=1}^M\mathbf{1}_{\{u^*_x=e_i\}}\rho_i(\sigma(x+\hat \mu_{i,t})-\sigma(x+\mu^*_i))\geq -\bar \rho |\sigma(x+\hat \mu_{i,t})-\sigma(x+\mu^*_i)|$. We bound the RHS with the sigmoid's $\nicefrac{1}{4}$-Lipschitz continuity and a MVT/self-concordance argument. By the MVT there exists $\tilde \mu$ between $\hat \mu_{i,t}$ and $\mu_i^*$ such that $\sigma'(x+\tilde \mu)(\hat \mu_{i,t} - \mu_i^*)= \sigma(x+\hat\mu_{i,t})-\sigma(x+\mu_i^*)$ and thus $|\sigma'(x+\tilde \mu)||(\hat \mu_{i,t} - \mu_i^*)|= |\sigma(x+\hat\mu_{i,t})-\sigma(x+\mu_i^*)|$. As $\hat \mu_{i,t}, \mu_i^* \in [-\bar \mu, \bar \mu]$, we have $\vert \tilde \mu - \hat \mu_{i,t}\vert \leq 2\bar \mu$ and by the sigmoid's self-concordance (\textcite{fauryImprovedOptimisticAlgorithms2020} Lemma 9):  $\sigma'(x+\hat \mu_{i,t})e^{-2\bar \mu} \leq \sigma'(x+\tilde \mu) \leq \sigma'(x+\hat \mu_{i,t})e^{2\bar \mu}$. Thus, $-\bar \rho |\sigma(x+\hat \mu_{i,t})-\sigma(x+\mu^*_i)| \geq -\bar \rho \vert \hat \mu_{i,t} - \mu_i^*\vert \min(\nicefrac{1}{4},e^{2\bar \mu}\sigma'(x+\hat \mu_{i,t}))$, where the arguments of the minimum come from the Lipschitz continuity and self-concordance approaches. For brevity, let $\kappa(x,\hat \mu_{i,t}) = \min(\nicefrac{1}{4}, e^{2\bar \mu} \sigma'(x+\hat \mu_{i,t}))$. Adding and subtracting the cross term $\gamma \ex_{\hat \mu_t}\left[J^*(z^\top\theta^*+w) \right]$, we have:
\begin{align}
   \Delta_J(x) &\geq \ex_{\hat \mu_t}[b_{t,\delta}(x,u^*_x) ] -\bar \rho\kappa(x,\hat\mu_{i,t})\vert \hat \mu_{i,t} - \mu^*_i\vert + \gamma \ex_{\hat \mu_t}[J^*(z^\top\hat \theta_t+w)-J^*(z^\top \theta^* + w)\vert x,u^*_x ]\nonumber\\
    &\qquad + \gamma \ex_{\hat \mu_t}[J^*(z^\top\theta^*+w)\vert x,u^*_x]-\gamma\ex_{\mu^*}[J^*(z^\top\theta^*+w)\vert x,u^*_x]\\
    &\geq \ex_{\hat \mu_t}[b_{t,\delta}(x,u^*_x) ] -\bar \rho\kappa(x,\hat\mu_{i,t})\vert \hat \mu_{i,t} - \mu^*_i\vert - \gamma \ex_{\hat \mu_t}[\vert J^*(z^\top\hat \theta_t+w)-J^*(z^\top \theta^* + w)\vert \big\vert x,u^*_x ]\nonumber\\
    &\qquad + \gamma (\sigma(x+\hat\mu_{i,t})-\sigma(x+\mu_i^*))\ex[
    J^*(z_1^\top\theta^* +w)-J^*(z_0^\top \theta^*+w)]
\end{align}
where $z_1 = [x \ e_i^\top \ e_i^\top]^\top$ and $z_0 = [x \ e_i^\top \ \ve{0}_M^\top]^\top$ denote state vectors for the adherence outcomes. We further simplify using the $L_1$-Lipschitz continuity of $J^*$ (Lemma~\ref{lem:j*_xlip}), expanding the Bernoulli outcomes, the confidence set bounds for $\hat \theta$ and $\hat \mu$ (Propositions~\ref{prop:theta_confidence_set} and~\ref{prop:mu_confidence_set}) under events $\mac E_\theta(\nicefrac{\delta}{2})$ and $\mac E_\mu(\nicefrac{\delta}{2})$, and Jensen's inequality:
\begin{align}
    \Delta_J(x) &\geq \ex_{\hat \mu_t}[b_{t,\delta}(x,u^*_x) ] -\bar \rho\kappa(x,\hat\mu_{i,t})\vert\hat \mu_{i,t}-\mu^*_i \vert- \gamma L_1\| \hat \theta_t - \theta^*\|_{\bar V_t}\ex_{\hat \mu_t}[\| z \|_{\bar V_t^{-1}}\vert x,u^*_x]\nonumber\\
    &\qquad - \gamma \vert \sigma(x+\hat\mu_{i,t})-\sigma(x+\mu_i^*) \vert \ex [L_1\vert(z_1 - z_0)^\top \theta^*\vert \vert x,u^*_x ]\\
    &\geq b_{t,\delta}(x,u^*_x) -\bar\rho\kappa(x,\hat\mu_{i,t})\alpha_{i,t}^\mu(\nicefrac{\delta}{2})- \gamma L_1 (\alpha_t^\theta(\nicefrac{\delta}{2})\ex_{\hat \mu_t}[\Vert z \|_{\bar V_t^{-1}}\vert x,u^*_x] + \bar c\kappa(x,\hat\mu_{i,t})\alpha_{i,t}^\mu(\nicefrac{\delta}{2}))
\end{align}
The bonus is constructed such that the RHS exactly cancels to $0$. If $u^*_x=\ve{0}_M$, the expression simplifies significantly (rewards and terms from the different $\mu$ parameterizations cancel) to $\Delta_J(x)\geq b_{t,\delta}(x,u^*_x)-\gamma L_1\alpha_t^\theta(\nicefrac{\delta}{2})\ex_{\hat \mu_t}[\Vert z \|_{\bar V_t^{-1}}\vert x,u^*_x]\geq 0$. Thus we have that $\tilde{\mac T}_{t,\delta}J^*(x) \geq J^*(x)$ for arbitrary $x\in \mac X$ and time step $t$. $\tilde{\mac T}_{t,\delta}$ is monotone (Lemma~\ref{lem:unique_opt}), so we have by induction on $n$ that $\tilde{\mac T}_{t,\delta}^nJ^*(x) \geq J^*(x)$ for all $n \in \mathbb N$. Suppose for a contradiction that $\exists x'\in \mac X$ with $\tilde J_{t,\delta}(x')<J^*(x')$. Let $\epsilon := J^*(x')-\tilde J_{t,\delta}(x')>0$. Per Lemma~\ref{lem:unique_opt}, $\tilde{\mac T}_{t,\delta}J$ converges uniformly to $\tilde J_{t,\delta}$. Thus $\exists N\in\mathbb N$ such that $ \vert \tilde{\mac T}_{t,\delta}^nJ^*(x')-\tilde J_{t,\delta}(x')  \vert < \epsilon$ for $n\geq N$. Thus we have $\tilde{\mac T}^N_{t,\delta}J^*(x') < \tilde J_{t,\delta}(x')+ \epsilon= \tilde J_{t,\delta}(x')+ J^*(x')-\tilde J_{t,\delta}(x')= J^*(x')$, producing a contradiction and the lemma follows. \halmos
\endproof

\subsection{Lipschitz constants}\label{sec:J*_lip_supporting}
For the supporting lemmas described here, we assume that Assumptions~\ref{asm:observed_quantities}-\ref{asm:noise_properties} hold.
\begin{lemma}{\textbf{(Reward Lipschitz)}}\label{lem:reward_lip}
For $\theta \in  \Theta$, $\mu \in \Phi$, $\ex_{\theta,\mu}[r(x,u,d)]$ is uniformly $L_r$-Lipschitz continuous in $x$ over $\mac X$, with $L_r = \nicefrac{(\beta+\bar\rho)}{4}$.
\end{lemma}
\proof{Proof of Lemma~\ref{lem:reward_lip}}
Let $x_1,x_2 \in \mac X$, $u\in \mac U$, $\theta \in \Theta$, $\mu \in \Phi$, $\Delta(u,x_1,x_2):=\ex_{\theta,\mu}[r(x_1,u,d)-r(x_2,u,d)]$. Then $\Delta(u,x_1,x_2)= -\beta\big(\sigma(\beta_0-x_1)-\sigma(\beta_0-x_2)\big) + \sum_{i=1}^Mu^i\rho_i\big(\sigma(x_1+\mu_i)-\sigma(x_2+\mu_i)\big)$. Applying the triangle inequality and the sigmoid's $\nicefrac{1}{4}$-Lipschitz continuity gives $\Delta(u,x_1,x_2) \leq \frac{\beta+\bar\rho}{4}|x_1-x_2|$. The Lipschitz result follows by a symmetrical argument (swap $x_1$ and $x_2$).
\halmos
\endproof

\begin{lemma}{\textbf{(Value-to-go Lipschitz)}}\label{lem:valuetogo_lip_loose}
For any $\theta \in \Theta$ and $\mu \in \Phi$, if a function $J$ is $L$-Lipschitz and $\| J \|_{\infty} \leq \bar J$, then $\ex_{\theta,\mu}\left[J(f(x,u,d,w)) \right]$ is $L_v$-Lipschitz with $L_v = L\bar a+\nicefrac{\bar J}{2}$.
\end{lemma}
\proof{Proof of Lemma~\ref{lem:valuetogo_lip_loose}}
Let $x_1,x_2\in \mac X$, $u=e_i\in \mac U\setminus \ve{0}_M$, $\theta \in \Theta$, and $\mu \in \Phi$. Let $\Delta(u,x_1,x_2):=\ex_{\theta,\mu}\left[ J(f(x_1,e_i,d,w))\right]-\ex_{\theta,\mu}\left[ J(f(x_2,e_i,d,w))\right]$. Let $J_1(x) := J(ax+b_i+c_i+w)$, and $J_0(x) := J(ax+b_i+w)$. Expanding the expectation with respect to adherence and simplifying, we have:
\begin{equation}
\label{eq:delta_dff_ec}
    \Delta(u,x_1,x_2) = \ex[J_0(x_1) + \hat p_i(x_1,1)( J_1(x_1)-J_0(x_1)) - J_0(x_2) - \hat p_i(x_2,1)( J_1(x_2)-J_0(x_2)) ]
\end{equation}
Add and subtract $\ex[\hat p_i(x_2,1)[J_1(x_1)-J_0(x_1)]]$, rearrange, and apply the triangle inequality:
\begin{multline}
    \eqref{eq:delta_dff_ec} \leq \ex [\underbrace{|(1-\hat p_i(x_2,1))[J_0(x_1) - J_0(x_2)]|}_{(i)} + \underbrace{|(\hat p_i(x_1,1)-\hat p_i(x_2,1))[ J_1(x_1)-J_0(x_1)]|}_{(ii)} \\+ \underbrace{|\hat p_i(x_2,1) [J_1(x_1)-J_1(x_2) ] |}_{(iii)}]
\end{multline}
$J$ is $L$-Lipschitz by assumption, so $(i)\leq (1-\hat p_i(x_2,1))L|a||x_1-x_2|$ and $(iii)\leq \hat p_i(x_2,1)L|a||x_1-x_2|$. The sigmoid is $\nicefrac{1}{4}$-Lipschitz and $\| J\|_{\infty} \leq \bar J$, so $(ii)\leq \nicefrac{2\bar J}{4}|x_1-x_2|$. Recall that $|a|\leq\bar a$ as $\theta \in \Theta$. The expression no longer depends on $w$ so the expectation drops and we have $\Delta(u,x_1,x_2)\leq [L\bar a+\nicefrac{\bar J}{2} ]|x_1-x_2|$. This constant also holds for the $u=\ve{0}_M$ case. The Lipschitz result follows by a symmetrical argument with $x_1$ and $x_2$ swapped.\halmos
\endproof
\begin{lemma}{\textbf{(Bonus Lipschitz)}}\label{lem:bonus_lip}
$b_{t,\delta}(x,u)$ is uniformly $L_b$-Lipschitz continuous in $x$ on $\mac X$ with $L_b =\nicefrac{e^{2\bar \mu}}{4}(\bar \rho + \gamma L_1\bar c)\overline \alpha^\mu(\nicefrac{\delta}{2})+\frac{5\gamma L_1\alpha^\theta_T(\nicefrac{\delta}{2})}{4\sqrt{\lambda_1}}$ for any $t\in\{1,\dots,T\}$, and $\delta \in (0,1)$.
\end{lemma}
\proof{Proof of Lemma~\ref{lem:bonus_lip}}
Consider $x_1,x_2\in \mac X$, $u=e_i\in\mac U\setminus\ve{0}_M$, $\delta \in (0,1)$, and $t\in\{1,\dots,T\}$. Define $\Delta_t(u,x_1,x_2) := b_{t,\delta}(x_1,u)-b_{t,\delta}(x_2,u)$. Consider the difference of the $\mu$-derived bonus term: $\Delta_t^\mu(u,x_1,x_2):= (\bar \rho + \gamma L_1\bar c)\alpha_{i,t}^\mu(\nicefrac{\delta}{2})[\kappa(x_1,\hat \mu_{i,t})-\kappa(x_2,\hat \mu_{i,t})]$. Recall that $\kappa(x,\mu) := \min(\nicefrac{1}{4},e^{2\bar \mu}\sigma'(x+ \mu))$. Note the sigmoid's $\nicefrac{1}{4}$-Lipschitz continuity and $\vert \sigma''(x+\mu)\vert=\vert \sigma'(x+\mu)(1-2\sigma(x+\mu))\vert\leq \nicefrac{1}{4}\cdot 1$, implying $e^{2\bar \mu}\sigma'(x+\mu)$ is $\nicefrac{e^{2\bar \mu}}{4}$-Lipschitz. Since the minimum with a constant is non-expansive, the composition (i.e., $\kappa(x,\mu)$) is $\nicefrac{e^{2\bar \mu}}{4}$-Lipschitz. Thus $\Delta^\mu_t(u,x_1,x_2)\leq \nicefrac{e^{2\bar \mu}}{4}(\bar \rho + \gamma L_1\bar c)\alpha_{i,t}^\mu(\nicefrac{\delta}{2})\vert x_1-x_2\vert$. Per Lemma~\ref{lem:bounded_bonus}, we bound $\alpha_{i,t}^\mu(\nicefrac{\delta}{2}) \leq \overline \alpha^\mu(\nicefrac{\delta}{2})$. Next consider the $\theta$-derived bonus term (i.e., $\Delta_t(u,x_1,x_2) = \Delta_t^\mu(u,x_1,x_2) + \Delta_t^\theta(u,x_1,x_2)$), letting $\hat p(x):=\sigma(x+\hat\mu_i)$:
\begin{align}
    \Delta^\theta_t(u,x_1,x_2) &= \gamma L_1 \alpha^\theta_t(\nicefrac{\delta}{2})[\ex_{\hat \mu_t}[ \Vert z  \Vert_{\bar V_t^{-1}}\vert x_1,u=e_i] -\ex_{\hat \mu_t}[ \Vert z  \Vert_{\bar V_t^{-1}}\vert x_2,u=e_i]]\\
    &=\gamma L_1 \alpha^\theta_t(\nicefrac{\delta}{2}) [\hat p(x_1)  (\Vert z_{1,1} \Vert_{\bar V_t^{-1}}-\Vert z_{1,0}\Vert_{\bar V_t^{-1}} )- \hat p(x_2) (\Vert z_{2,1} \Vert_{\bar V_t^{-1}}-\Vert z_{2,0}\Vert_{\bar V_t^{-1}} )\nonumber\\
    &\qquad + (\Vert z_{1,0}\Vert_{\bar V_t^{-1}}-\Vert z_{2,0}\Vert_{\bar V_t^{-1}})\pm \hat p(x_1)(\Vert z_{2,1} \Vert_{\bar V_t^{-1}}-\Vert z_{2,0}\Vert_{\bar V_t^{-1}} ) ]
\end{align}
where $z_{j,k}=[x_j \ u^\top \ {d^i_k}^\top]^\top$ and $d^i_k =k e_i$. Thus we have:
\begin{align}
    \Delta^\theta_t(u,x_1,x_2) &= \gamma L_1 \alpha^\theta_t(\nicefrac{\delta}{2})  [\hat p(x_1)(\Vert z_{1,1} \Vert_{\bar V_t^{-1}} -\Vert z_{2,1} \Vert_{\bar V_t^{-1}})+ (1-\hat p(x_1))(\Vert z_{1,0}\Vert_{\bar V_t^{-1}}-\Vert z_{2,0}\Vert_{\bar V_t^{-1}}) \nonumber\\
    &\qquad  + (\hat p(x_1)-\hat p(x_2))(\Vert z_{2,1} \Vert_{\bar V_t^{-1}}-\Vert z_{2,0}\Vert_{\bar V_t^{-1}} ) ]\\
    &\overset{(a)}{\leq} \gamma L_1 \alpha^\theta_t(\nicefrac{\delta}{2})  [\Vert z_{1,1}- z_{2,1} \Vert_{\bar V_t^{-1}}+ \nicefrac{1}{4}\vert x_1-x_2\vert \Vert z_{2,1}- z_{2,0}\Vert_{\bar V_t^{-1}} ]\\
    &\overset{(b)}{\leq} \gamma L_1 \alpha^\theta_t(\nicefrac{\delta}{2}) \vert x_1-x_2 \vert \nicefrac{1}{\sqrt{\lambda_1}}[ 1+\nicefrac{1}{4}]
\end{align}
$(a)$ comes from applying the reverse triangle inequality to the norm differences, the fact that $z_{1,1}-z_{2,1}=z_{1,0}-z_{2,0}$ to combine $\hat p(x_1)$ terms, and the sigmoid's $\nicefrac{1}{4}$-Lipschitz continuity. $(b)$ follows from evaluating the norms and because the diagonal elements of $\bar V_t^{-1}$ are upper bounded by $\nicefrac{1}{\lambda_1}$. $\alpha^\theta_t(\nicefrac{\delta}{2})$ is increasing in $t$, so $\alpha^\theta_t(\nicefrac{\delta}{2})\leq \alpha^\theta_T(\nicefrac{\delta}{2})$ for a $T$-length trajectory. Substituting the $\Delta^\mu_t$ and $\Delta^\theta_t$ results we obtain $\Delta_t(u,x_1,x_2)\leq [\nicefrac{e^{2\bar \mu}}{4}(\bar \rho + \gamma L_1\bar c)\overline \alpha^\mu(\nicefrac{\delta}{2})+\frac{5\gamma L_1\alpha^\theta_T(\nicefrac{\delta}{2})}{4\sqrt{\lambda_1}}] \vert x_1 - x_2 \vert$. The same holds if $u=\ve{0}_M$. 
The Lipschitz result follows by a symmetrical argument with $x_1$ and $x_2$ swapped.\halmos
\endproof
\subsection{Regret proof and supporting lemmas}\label{app:ec_regret}
\proof{Proof of Theorem~\ref{thm:regret}}
We divide $\delta_{all}$ into thirds such that the proposition holds after a union bound. Let $\delta = \nicefrac{\delta_{all}}{3}$ be used to establish the optimistic surrogate system, i.e., with probability at least $1-\nicefrac{\delta_{all}}{3}$, $\mac E_{\theta}(\nicefrac{\delta}{2})$ (Proposition~\ref{prop:theta_confidence_set}) and $\mac E_{\mu}(\nicefrac{\delta}{2})$ (Proposition~\ref{prop:mu_confidence_set}) hold. The remaining $\nicefrac{2\delta_{all}}{3}$ is used for two martingale concentration arguments in the regret bound, each using $\nicefrac{\delta_{all}}{3}$. We upper bound the regret using optimistic value functions $\tilde J_{t,\delta}$. This proof relies on a telescoping-style simplification, so it is helpful to index on time steps $t$ rather than epochs. We add the ${}^\circ$ superscript to denote the quantity the algorithm is using at time $t$. We denote the sequence of optimistic Bellman operators $\{\tilde {\mac T}^{\circ}_{t,\delta}\}_{t=1}^{T+1}$ to be defined using the sequence of bonus functions $\{b^{\circ}_{t,\delta}(x,u) \}_{t=1}^{T+1}$, estimators $\{(\hat \theta^{\circ}_t,\hat \mu^{\circ}_t)\}_{t=1}^{T+1}$, confidence set radii $\{(\alpha^{\theta,\circ}_t,\alpha^{\mu,\circ}_t)\}_{t=1}^{T+1}$, and regularized empirical covariances matrices $\{\bar V^{\circ}_t\}_{t=1}^{T+1}$ specifically used by Algorithm~\ref{alg:ucb_epoch}, i.e. $\tilde {\mac T}^{\circ}_{t,\delta}J := \max_{u\in \mac U} \{\ex_{\hat \theta_t^\circ,\hat \mu_t^\circ} [r(x,u,d)+b^{\circ}_{t,\delta}(x,u)+\gamma J(f(x,u,d,w))]\}$. $\{\tilde J^{\circ}_{t,\delta}\}_{t=1}^{T+1}$ is the sequence of fixed points associated with this sequence of operators (Lemma~\ref{lem:unique_opt}). As an example, for some $t\in\{1,\dots,T+1\}$ we have that $\tilde {\mac T}^{\circ}_{t,\delta}$ is equivalent to $\tilde {\mac T}_{\tau_k,\delta}$, where $\tau_k$ is the time index of the start of epoch $k$ active at step $t$. By Lemma~\ref{lem:valid_optimism}, when $\mac E_\theta(\nicefrac{\delta}{2})$ and $\mac E_{\mu}(\nicefrac{\delta}{2})$ hold, the sequence of value functions $\{\tilde J^{\circ}_{t,\delta}\}_{t=1}^{T+1}$ overestimate the true value function $J^*$ for $t=1,\dots,T+1$. Substituting into the regret definition: $R_T = \sum_{t=1}^T J^* (x_t) - J^{\pi_t}(x_t)\leq \sum_{t=1}^T\tilde J^{\circ}_{t,\delta}(x_t) - J^{\pi_t}(x_t)$. For brevity, define $\tilde R_t:=\tilde J^\circ_{t,\delta}(x_t)- J^{\pi_t}(x_t)$, i.e., $R_T \leq \sum_{t=1}^T \tilde R_t$. We expand $\tilde R_t$ (a single time step) then substitute back into the summation for a telescoping simplification:
\begin{align}
    \tilde R_t &= \max_{u\in \mac U}\{\ex_{\hat \mu^\circ_t,\hat \theta^\circ_t}[r(x_t,u,d)+b^{\circ}_{t,\delta}(x_t,u)+ \gamma\tilde J^{\circ}_{t,\delta}(f(x_t,u,d,w))]\}\nonumber\\
    &\qquad- \ex_{\mu^*,\theta^*}[r(x_t,\pi_t(x_t),d) + \gamma J^{\pi_t}(f(x_t,\pi_t(x_t),d,w))]
\end{align}
Per Algorithm~\ref{alg:ucb_epoch}, $\pi_t$ is derived from $\tilde J^{\circ}_{t,\delta}$, so we can express both using $u_t:=\pi_t(x_t)$. We add and subtract the following: $\gamma\ex_{\theta^*,\hat \mu^\circ_t}[\tilde J^{\circ}_{t,\delta}(f(x_t,u_t,d,w)) ]$, $\gamma\ex_{\theta^*,\mu^*}[\tilde J^{\circ}_{t,\delta}(f(x_t,u_t,d,w)) ]$, $\gamma [\tilde J^{\circ}_{t,\delta}(x_{t+1}) - J^{\pi_t}(x_{t+1}) ]$, $\gamma \tilde J^{\circ}_{t+1,\delta}(x_{t+1})$, and $\gamma J^{\pi_{t+1}}(x_{t+1})$, where $x_{t+1}$ is the observed state at $t+1$. Rearranging and using $f(\cdot) = f(x_t,u_t,d,w)$ for brevity we have:
\begin{align}
    \tilde R_t &=\ex_{\hat \mu^\circ_t}[r(x_t,u_t,d)] - \ex_{\mu^*}[r(x_t,u_t,d) ] +\ex_{\hat\mu_t^\circ}[b^{\circ}_{t,\delta}(x_t,u_t)] \label{eq:one_step_regret}\\
    &\quad+ \gamma [\ex_{\theta^*,\hat\mu_t^\circ}[\tilde J^{\circ}_{t,\delta}(f(\cdot)) ] - \ex_{\theta^*,\mu^*}[\tilde J^{\circ}_{t,\delta}(f(\cdot)) ] ] + \gamma [\ex_{\hat \theta_t^\circ,\hat\mu_t^\circ}[\tilde J^{\circ}_{t,\delta}(f(\cdot)) ] - \ex_{\theta^*,\hat\mu_t^\circ}[\tilde J^{\circ}_{t,\delta}(f(\cdot)) ]]\nonumber\\
    &\quad  +\gamma[\ex_{\theta^*,\mu^*}[\tilde J^{\circ}_{t,\delta}(f(\cdot)) - J^{\pi_t}(f(\cdot))] -(\tilde J^{\circ}_{t,\delta}(x_{t+1}) - J^{\pi_t}(x_{t+1}) )]\nonumber\\
    &\quad + \underbrace{\gamma [ \tilde J^\circ_{t+1,\delta}(x_{t+1})-J^{\pi_{t+1}}(x_{t+1})]}_{= \gamma \tilde R_{t+1}}+ \gamma[\tilde J^\circ_{t,\delta}(x_{t+1})-\tilde J^{\circ}_{t+1,\delta}(x_{t+1})+J^{\pi_{t+1}}(x_{t+1})-J^{\pi_t}(x_{t+1}) ]\nonumber
\end{align}

Thus we recover the $\gamma\tilde R_{t+1}$ term needed for a telescoping-style simplification. Let $(*)_t$ denote the non-braced terms in Eq.~\ref{eq:one_step_regret}. Substituting (\ref{eq:one_step_regret}) into the regret summation gives $R_T \leq \sum_{t=1}^T \tilde R_t = \sum_{t=1}^T[(*)_t + \gamma \tilde R_{t+1} ]= \gamma\tilde R_{T+1} + \sum_{t=1}^T(*)_t +\gamma\sum_{t=2}^T\tilde R_t$. Observe that $\sum_{t=1}^T\tilde R_t = (1-\gamma) \tilde R_1 + \gamma \tilde R_1 + \sum_{t=2}^T\tilde R_t$. Rearranging and dividing by $1-\gamma$ gives $R_T\leq\sum_{t=1}^T \tilde R_t \leq \frac{1}{1-\gamma}[\gamma(\tilde R_{T+1}-\tilde R_1) +\sum_{t=1}^T(*)_t ]$. Expanding the $(*)_t$ terms (including splitting the bonus) and using $\kappa(x,\mu) \leq \nicefrac{1}{4}$ we have:
\begin{align}
    R_T &\leq \tfrac{1}{1-\gamma}\underbrace{\textstyle\sum\nolimits_{t=1}^T(\ex_{\hat \mu^\circ_t}[r(x_t,u_t,d)] - \ex_{\mu^*}[r(x_t,u_t,d) ] + \tfrac{(\bar \rho + \gamma L_1 \bar c)}{4}\textstyle\sum\nolimits_{i=1}^Mu_t^i \alpha_{i,t}^{\mu,\circ}(\nicefrac{\delta}{2}))}_{(i)}\nonumber\\
    &\qquad +\tfrac{\gamma}{1-\gamma} \underbrace{[\tilde R_{T+1}-\tilde R_1 ]}_{(ii)}+ \tfrac{\gamma}{1-\gamma}\underbrace{\textstyle\sum\nolimits_{t=1}^T(\ex_{\theta^*,\hat\mu_t^\circ}[\tilde J^{\circ}_{t,\delta}(f(\cdot)) ] - \ex_{\theta^*,\mu^*}[\tilde J^{\circ}_{t,\delta}(f(\cdot)) ])}_{(iii)}\nonumber\\
    &\qquad + \tfrac{\gamma}{1-\gamma}\underbrace{\textstyle\sum\nolimits_{t=1}^T(\ex_{\hat \theta_t^\circ,\hat\mu_t^\circ}[\tilde J^{\circ}_{t,\delta}(f(\cdot)) ] - \ex_{\theta^*,\hat\mu_t^\circ}[\tilde J^{\circ}_{t,\delta}(f(\cdot)) ] +L_1\alpha_t^{\theta,\circ}(\nicefrac{\delta}{2})\ex_{\hat \mu_t^\circ}[ \Vert z\Vert_{(\bar V_t^\circ)^{-1}}\vert x_t,u_t])}_{(iv)}\nonumber\\
    &\qquad + \tfrac{\gamma}{1-\gamma}\underbrace{\textstyle\sum\nolimits_{t=1}^T (\ex_{\theta^*,\mu^*}[\tilde J^{\circ}_{t,\delta}(f(\cdot)) - J^{\pi_t}(f(\cdot))] -(\tilde J^{\circ}_{t,\delta}(x_{t+1}) - J^{\pi_t}(x_{t+1}) ))}_{(v)}\nonumber\\
    &\qquad + \tfrac{\gamma}{1-\gamma}\underbrace{\textstyle\sum\nolimits_{t=1}^T [\tilde J^\circ_{t,\delta}(x_{t+1})-\tilde J^{\circ}_{t+1,\delta}(x_{t+1})+J^{\pi_{t+1}}(x_{t+1})-J^{\pi_t}(x_{t+1}) ]}_{(vi)}
\end{align}
We apply Lemma~\ref{lem:bound_i} for the given $C_N$ to bound $(i)$ as $\mac E_\mu(\nicefrac{\delta}{2})$ holds by assumption. We apply Lemma~\ref{lem:bound_ii} to bound $(ii)$ as $\mac E_\theta(\nicefrac{\delta}{2})$ and $\mac E_\mu(\nicefrac{\delta}{2})$ hold by assumption. We apply Lemma~\ref{lem:bound_iii} to bound $(iii)$ as $\mac E_\mu(\nicefrac{\delta}{2})$ holds by assumption. We apply Lemma~\ref{lem:bound_iv} to bound $(iv)$ as $\mac E_\theta(\nicefrac{\delta}{2})$ and $\mac E_\mu(\nicefrac{\delta}{2})$ hold by assumption with $\delta'=\nicefrac{\delta_{all}}{3}$. We apply Lemma~\ref{lem:bound_v} to bound $(v)$ with $\delta'=\nicefrac{\delta_{all}}{3}$. We apply Lemma~\ref{lem:bound_vi} to bound $(vi)$. The result follows by the union bound.
\halmos
\endproof
\proof{Proof of Corollary~\ref{cor:expected_regret}}
Per Theorem~\ref{thm:regret}, with probability at least $1-\delta_{all}$, the given regret bound holds, denoted $B(T,\delta_{all})$ for brevity. Note that $B(T,\delta_{all})=\tilde{\mac O}\left(\frac{M^{\nicefrac{3}{2}}\sqrt{T}}{(1-\gamma)^3} \right)$, with $(L_1+L_2)\alpha_T^\theta(\nicefrac{\delta_{all}}{6}) \left[C_1+C_2+C_3\right]$ being the leading regret term. With probability at most $\delta_{all}$, the necessary events for Theorem~\ref{thm:regret} do not hold. In this case, the per-step regret is bounded by $J^*(x_t) - J^{\pi_t}(x_t) \leq 2 C_J$ (see Lemma~\ref{lem:unique_opt}). Thus, $\ex[R_T] \leq (1-\delta_{all}) B(T,\delta_{all}) + \delta_{all} T(2C_J) \leq B(T,\delta_{all}) + 2T\delta_{all}\left(\frac{\beta + \bar \rho}{1-\gamma}\right)$. Choosing $\delta_{all}=\nicefrac{1}{T}$ means the regret rate is dominated by the $B(T,\delta_{all})$ term and this choice of $\delta_{all}$ adds only logarithmic factors of $T$ to $B(T,\delta_{all})$, so the corollary follows.
\halmos
\endproof

\subsubsection{Supporting lemmas}
The following lemmas assume that Assumptions~\ref{asm:observed_quantities}-\ref{asm:rewards} hold and $C_N,C_d>0$ are given by Algorithm~\ref{alg:ucb_epoch}. We define a set of common terms for brevity: $C_J = \frac{\beta + \bar \rho}{1-\gamma}$, $C_{\tilde J,t,\delta} = \frac{\beta +\bar \rho +C_b(t,\delta)}{1-\gamma}$, $\nu_T(\delta) := e^{3\bar \mu}[\frac{\sqrt{\lambda_2}}{2}+\frac{2\bar\mu}{\sqrt{\lambda_2}}+\frac{2}{\sqrt{\lambda_2}}\log (\frac{2M\sqrt{\lambda_2 + \nicefrac{T}{4}}}{\delta \sqrt{\lambda_2}} )+e^{-\bar \mu}\sqrt{\lambda_2} \bar \mu]$, $C_1=\frac{\nu_T(\nicefrac{\delta}{2})\sqrt{1+C_N}}{4\sqrt{\lambda_1}}[\nicefrac{M}{\sqrt{\lambda_2}}+2\sqrt{\nicefrac{MT}{\underline q}} ]$, $C_2=\sqrt{2\log(\nicefrac{1}{\delta'})\nicefrac{T}{\lambda_1}}$, $C_3=\sqrt{T(1+C_d)(2M+1)\log (1+\frac{T(C_x^2+2)}{(2M+1)\lambda_1} ) [2 +\frac{C_x^2+2}{\log(2)\lambda_1}]}$, $K_d = \frac{(2M+1)}{\log(1+C_d)}\log (1+\frac{T(C_x^2+2)}{(2M+1)\lambda_1} )$, $ K_N = M(1+\frac{\log(T)}{\log(1+C_N)})$.
\begin{lemma}{}\label{lem:bound_i}
Let $\mac E_\mu(\nicefrac{\delta}{2})$ hold. For any trajectory $t=1,\dots,T$, $\sum_{t=1}^T(\ex_{\hat \mu^\circ_t}[r(x_t,u_t,d)] - \ex_{\mu^*}[r(x_t,u_t,d)] + \frac{(\bar \rho + \gamma L_1 \bar c)}{4}\sum_{i=1}^Mu_t^i \alpha_{i,t}^{\mu,\circ}(\nicefrac{\delta}{2})) \leq \frac{(2\bar \rho + \gamma L_1\bar c)\nu_T(\nicefrac{\delta}{2})\sqrt{1+C_N}}{4}[\nicefrac{M}{\sqrt{\lambda_2}}+2\sqrt{\nicefrac{MT}{\underline q}} ]$.
\end{lemma}
\proof{Proof of Lemma~\ref{lem:bound_i}}
Let $\Delta_r:=\sum_{t=1}^T\big(\ex_{\hat \mu^\circ_t}\big[r(x_t,u_t,d)\big] - \ex_{\mu^*}\big[r(x_t,u_t,d) \big]\big)$. Note that the $\beta$-sigmoid portion of the reward cancels between each pair of terms:
$\Delta_r \leq \sum_{t=1}^T  \sum_{i=1}^M  \vert u_t^i\rho_i(\sigma(x_t+\hat\mu_{i,t}^\circ)-\sigma(x_t+\mu^*_i)) \vert \leq \sum_{t=1}^T\sum_{i=1}^M u_t^i \frac{\bar \rho}{4} \vert \hat\mu_{i,t}^\circ - \mu^*_i\vert \overset{(a)}{\leq} \frac{\bar \rho}{4}\sum_{t=1}^T\sum_{i=1}^M u^i_t \alpha_{i,t}^{\mu,\circ}(\nicefrac{\delta}{2})$.
$(a)$ holds by Proposition~\ref{prop:mu_confidence_set} as $\mac E_\mu(\nicefrac{\delta}{2})$ holds by assumption. The lemma follows by applying Lemma~\ref{lem:mu_sums} for the given $C_N$ and $\nicefrac{\delta}{2}$ to the common $\alpha^{\mu,\circ}_t$ summation in the reward difference bound and the bonus terms.
\halmos
\endproof
\begin{lemma}{}\label{lem:mu_sums}
For any realized trajectory $t=1,\dots,T$ following Algorithm~\ref{alg:ucb_epoch}, $\sum_{t=1}^T\sum_{i=1}^Mu^i_t \alpha_{i,t}^{\mu,\circ}(\delta)\leq \nu_T(\delta)\sqrt{1+C_N}[\nicefrac{M}{\sqrt{\lambda_2}}+2\sqrt{\nicefrac{MT}{\underline q}}]$.
\end{lemma}
\proof{Proof of Lemma~\ref{lem:mu_sums}}
First consider terms not indexed by $({\cdot})^\circ$. By Proposition~\ref{prop:mu_confidence_set}, $\alpha^{\mu}_{i,t}(\delta) := \frac{e^{3\bar \mu}}{\sqrt{H_t(\hat \mu_{i,t})}}[\frac{\sqrt{\lambda_2}}{2}+\frac{2\bar\mu}{\sqrt{\lambda_2}}+\frac{2}{\sqrt{\lambda_2}}\log (\frac{2M\sqrt{H_t(\hat \mu_{i,t})}}{\delta \sqrt{\lambda_2}} )]+\frac{e^{2\bar \mu}\lambda_2 \bar \mu}{H_t(\hat \mu_{i,t})}$ and $H_t(\hat \mu_{i,t}) := \lambda_2 + \sum_{s=1}^{t-1}\mathbf{1}_{\{u_s=e_i\}}\sigma'(x_s+\hat \mu_{i,t})$. Note the loose bounds $H_t(\mu)\geq \lambda_2$ (as $\sigma'(z)\geq 0$) and $H_t(\mu) \leq \lambda_2 + \nicefrac{T}{4}$ (as $\sigma'(z)\leq \nicefrac{1}{4}$), giving the bound $\alpha_{i,t}^{\mu}(\delta) \leq \frac{e^{3\bar \mu}}{\sqrt{H_t(\hat \mu_{i,t})}}[\frac{\sqrt{\lambda_2}}{2}+\frac{2\bar\mu}{\sqrt{\lambda_2}}+\frac{2}{\sqrt{\lambda_2}}\log (\frac{2M\sqrt{\lambda_2 + \nicefrac{t}{4}}}{\delta \sqrt{\lambda_2}} )+e^{-\bar \mu}\sqrt{\lambda_2} \bar \mu]$. Define $\nu_t(\delta):=e^{3\bar \mu}[\frac{\sqrt{\lambda_2}}{2}+\frac{2\bar\mu}{\sqrt{\lambda_2}}+\frac{2}{\sqrt{\lambda_2}}\log (\frac{2M\sqrt{\lambda_2 + \nicefrac{t}{4}}}{\delta \sqrt{\lambda_2}} )+e^{-\bar \mu}\sqrt{\lambda_2} \bar \mu]$. Thus, $\alpha_{i,t}^{\mu}(\delta) \leq \frac{\nu_t(\delta)}{\sqrt{H_t(\hat \mu_{i,t})}}\leq \frac{\nu_T(\delta)}{\sqrt{H_t(\hat \mu_{i,t})}}$ ($\nu_t(\delta)$ increasing in $t$). Next observe that for our system $\min_{x\in \mac X, \mu_i \in [-\bar \mu,\bar \mu]}\sigma'(x+\mu_i)=\sigma'(C_x+\bar \mu)=:\underline q$. This $\underline q$ is a uniform worst case information gain for a non-null action. Let $N_{i,t}:=\sum_{s=1}^{t-1}\mathbf{1}_{\{u_s=e_i\}}$ be the number of times action $e_i$ has been taken by the beginning of step $t$. This implies the uniform lower bound $H_t(\mu_i) \geq \lambda_2 + \underline q N_{i,t}$ and further that $\alpha_{i,t}^{\mu}(\delta) \leq \frac{\nu_T(\delta)}{\sqrt{\lambda_2 + \underline q N_{i,t}}}$. Returning to $({\cdot})^\circ$ indexed terms, we have $\alpha^{\mu,\circ}_{i,t}(\delta)\leq \frac{\nu_T(\delta)}{\sqrt{H_t^\circ(\hat \mu^\circ_{i,t})}}$. For time steps in the $k^{th}$ epoch (beginning at step $\tau_k$), $H_t^\circ(\hat \mu^\circ_{i,t})\geq \lambda_2 +\underline q N_{i,\tau_k}$ lower bounds the information available for treatment $i$ (i.e., per the action count at the start of the epoch). We connect $({\cdot})^\circ_t$ values to $({\cdot})_t$ values by noting that at the end of step $t$, Algorithm~\ref{alg:ucb_epoch} triggers a new epoch beginning at $t+1$ (i.e., $\tau_{k+1}=t+1$) if $N_{i,t+1} > (1+C_N)N_{i,\tau_k}$ for any $i=1,\dots,M$. Thus for time steps $\tau_k \leq t <\tau_{k+1}$ within epoch $k$ the trigger has not fired, so $N_{i,t} \leq (1+C_N)N_{i,\tau_k}$ holds. This implies (after using loosening for bound clarity with $\lambda_2(1+C_N)\geq \lambda_2$): $\alpha_{i,t}^{\mu,\circ}(\delta) \leq \frac{\nu_T(\delta)\sqrt{1+C_N}}{\sqrt{\lambda_2 + \underline qN_{i,t}}}$. Now we introduce a count based indexing. Let $t_i^j$ be the time step where action $i$ is taken for the $j^{th}$ time. Observe that $N_{i,t_i^j}=j-1$ (i.e., the action has previously been taken $j-1$ times if it is taken the $j^{th}$ time in the current time step). Thus we can re-index with $j$ rather than $t$ as follows:
\begin{align}
    \textstyle\sum\nolimits_{t=1}^T\sum\nolimits_{i=1}^Mu^i_t \alpha_{i,t}^{\mu,\circ}(\delta) &\leq \textstyle\sum\nolimits_{i=1}^M\sum\nolimits_{t=1}^Tu_t^i\tfrac{\nu_T(\delta)\sqrt{1+C_N}}{\sqrt{\lambda_2+\underline q N_{i,t}}}
    \overset{(a)}{=} \textstyle\sum\nolimits_{i=1}^M\sum\nolimits_{j=1}^{N_{i,T+1}}\tfrac{\nu_T(\delta)\sqrt{1+C_N}}{\sqrt{\lambda_2+\underline q (j-1)}}
\end{align}
where $(a)$ uses the fact that $u^i_t\in\{0,1\}$ ensures only the $N_{i,T+1}$ terms given by $t_i^1,\dots,t^{N_{i,T+1}}_i$ are included in the summation and $N_{i,t_i^j}=j-1$. Since $\nu_T(\delta)(\sqrt{1+C_N})$ factors out of the summation and summands $(\lambda_2 +\underline q(j-1))^{-\nicefrac{1}{2}}$ are decreasing in $j$, we have the following integral upper bound:
\begin{align}
    \textstyle\sum\nolimits_{j=1}^{N_{i,T+1}}\tfrac{1}{\sqrt{\lambda_2+\underline q (j-1)}} \leq \tfrac{1}{\sqrt{\lambda_2}}+\int\nolimits_{1}^{N_{i,T+1}}\tfrac{1}{\sqrt{\lambda_2 +\underline q(j-1)}}dj
    =\tfrac{1}{\sqrt{\lambda_2}}+\tfrac{2}{\underline q}[\sqrt{\lambda_2 +\underline q (N_{i,T+1}-1)} - \sqrt{\lambda_2} ]
\end{align}
Using the subadditivity of square roots and $N_{i,T+1}-1 \leq N_{i,T+1}$ we simplify to the more compact upper bound $\frac{1}{\sqrt{\lambda_2}}+\frac{2\sqrt{N_{i,T+1}}}{\sqrt{\underline q}}$. Substituting back and applying Cauchy-Schwarz we have:
\begin{align}
    \textstyle\sum\nolimits_{i=1}^M [\tfrac{1}{\sqrt{\lambda_2}}+\tfrac{2\sqrt{N_{i,T+1}}}{\sqrt{\underline q}} ] \leq \tfrac{M}{\sqrt{\lambda_2}} + \sqrt{\sum\nolimits_{i=1}^M (1)^2}\sqrt{\sum\nolimits_{i=1}^M (\tfrac{2\sqrt{N_{i,T+1}}}{\sqrt{\underline q}} )^2} \leq \tfrac{M}{\sqrt{\lambda_2}}+2\sqrt{\tfrac{MT}{\underline q}}
\end{align}
The lemma follows by substitution. As $\underline q$ is a uniform lower bound and only $\sum\nolimits_{i=1}^MN_{i,T+1}\leq T$ was used, the bound holds for any realized trajectory.
\halmos
\endproof

\begin{lemma}{}\label{lem:bound_ii}
Let $\mac E_{\theta}(\nicefrac{\delta}{2})$ and $\mac E_{\mu}(\nicefrac{\delta}{2})$ hold. Then $\tilde R_{T+1} - \tilde R_1\leq C_J+C_{\tilde J,T+1,\delta}$.
\end{lemma}
\proof{Proof of Lemma~\ref{lem:bound_ii}}
As $J^*(x)\geq J^{\pi_t}(x)$ by the optimality of $J^*$, the $\tilde R$ terms are positive by chaining this inequality with optimism Lemma~\ref{lem:valid_optimism} ($\mac E_{\theta}(\nicefrac{\delta}{2})$ and $\mac E_{\mu}(\nicefrac{\delta}{2})$ hold by assumption), so it is sufficient to bound $\tilde R_{T+1}$. By definition, $\tilde R_{T+1}=\tilde J^\circ_{T+1,\delta}(x_{T+1})- J^{\pi_{T+1}}(x_{T+1})$. Both value functions are bounded by Lemma~\ref{lem:unique_opt}, i.e. $\Vert J^{\pi_{T+1}}\Vert_\infty \leq C_J$ and $\Vert \tilde J^{\circ}_{T+1,\delta}\Vert_\infty \leq C_{\tilde J,T+1,\delta}$. Thus $\tilde R_{T+1} - \tilde R_1 \leq \tilde R_{T+1}\leq \max_{x \in \mac X}\tilde J^{\circ}_{T+1,\delta}(x) - \min_{x \in \mac X}J^{\pi_{T+1}}(x) \leq C_J +C_{\tilde J,T+1,\delta}$.\halmos
\endproof

\begin{lemma}{}\label{lem:bound_iii}
Let $\mac E_\mu(\nicefrac{\delta}{2})$ hold. Then $\sum_{t=1}^T(\ex_{\theta^*,\hat\mu_t^\circ}[\tilde J^{\circ}_{t,\delta}(f(x_t,u_t,d,w) ] - \ex_{\theta^*,\mu^*}[\tilde J^{\circ}_{t,\delta}(f(x_t,u_t,d,w)) ]) \leq \nicefrac{1}{4}L_2 \bar c \nu_T(\nicefrac{\delta}{2})\sqrt{1+C_N}[\nicefrac{M}{\sqrt{\lambda_2}}+2\sqrt{\nicefrac{MT}{\underline q}} ]$.
\end{lemma}
\proof{Proof of Lemma~\ref{lem:bound_iii}}
Let $\mac E_\mu(\nicefrac{\delta}{2})$ hold. We use a similar simplification to Lemma~\ref{lem:valid_optimism}. Let $z_1:=[x_t \ u_t^\top \ u_t^\top]^\top$ and $z_0:=[x_t \ u_t^\top \ \ve{0}_M^\top]^\top$ denote the possible state vectors under adherence and non-adherence. Consider a single time step and let $u_t=e_i$ (the difference is trivially $0$ for any time step with $u_t=\ve{0}_M$), with $\Delta := \ex_{\theta^*,\hat\mu_t^\circ}[\tilde J^{\circ}_{t,\delta}(f(x_t,u_t,d,w) ] - \ex_{\theta^*,\mu^*}[\tilde J^{\circ}_{t,\delta}(f(x_t,u_t,d,w)) ]$. We have:
\begin{align}
    \Delta &= \ex_{\hat \mu^\circ_t}[\tilde J^\circ_{t,\delta}(z^\top\theta^*+w) \vert x_t,u_t]-\ex_{\mu^*}[\tilde J^\circ_{t,\delta}(z^\top \theta^*+w) \vert x_t,u_t]\\
    &\overset{(a)}{\leq} |\sigma(x_t+\hat \mu_{i,t}^\circ) -\sigma(x_t+\mu_i^*)|\cdot\ex[|\tilde J^\circ_{t,\delta}(z_1^\top \theta^* +w)- \tilde J_{t,\delta}^\circ(z_0^\top \theta^* +w) |]\\
    &\overset{(b)}{\leq} \nicefrac{1}{4}\vert \hat \mu_{i,t}^\circ - \mu_i^*\vert \cdot L_2 \vert (z_1-z_0)^\top \theta^*\vert \overset{(c)}{\leq} \nicefrac{1}{4}\alpha_{i,t}^{\mu,\circ}(\nicefrac{\delta}{2})\cdot L_2 \bar c
\end{align}
$(a)$ upper bounds with absolute values, uses Jensen's inequality, expands adherence outcomes, and simplifies. $(b)$ uses Lipschitz continuity of the sigmoid and $\tilde J_{t,\delta}$ (Lemma~\ref{lem:tj_xlip}). $(c)$ uses Proposition~\ref{prop:mu_confidence_set} and $c_i \leq \bar c$ as $\theta^*\in\Theta$. Thus the full summation is upper bounded by $\frac{L_2 \bar c}{4}\sum_{t=1}^T\sum_{i=1}^Mu_t^i\alpha^{\mu,\circ}_{i,t}(\nicefrac{\delta}{2})$. As $\mac E_{\mu}(\nicefrac{\delta}{2})$ holds, we apply Lemma~\ref{lem:mu_sums} for the given $C_N$ and the lemma follows.\halmos
\endproof

\begin{lemma}{}\label{lem:bound_iv}
Let $\delta'\in(0,1)$ and $\mac E_{\theta}(\nicefrac{\delta}{2})$,$\mac E_{\mu}(\nicefrac{\delta}{2})$ hold. Then $\pr(\sum_{t=1}^T\big(\ex_{\hat \theta_t^\circ,\hat\mu_t^\circ}\big[\tilde J^{\circ}_{t,\delta}(f(\cdot)) \big] - \ex_{\theta^*,\hat\mu_t^\circ}\big[\tilde J^{\circ}_{t,\delta}(f(\cdot)) \big] +L_1\alpha_t^{\theta,\circ}(\nicefrac{\delta}{2})\ex_{\hat \mu_t^\circ}\big[ \Vert z\Vert_{(\bar V^\circ_t)^{-1}}\vert x_t,u_t\big]\big) \leq (L_1+L_2)\alpha_T^\theta(\nicefrac{\delta}{2}) \left[C_1+C_2+C_3\right])\geq 1-\delta'$.
\end{lemma}
\proof{Proof of Lemma~\ref{lem:bound_iv}}
Let $(iv)_1=\sum_{t=1}^T(\ex_{\hat \theta_t^\circ,\hat\mu_t^\circ}[\tilde J^{\circ}_{t,\delta}(f(x_t,u_t,d,w)) ] - \ex_{\theta^*,\hat\mu_t^\circ}[\tilde J^{\circ}_{t,\delta}(f(x_t,u_t,d,w)) ])$ and $(iv)_2$ denote the remaining summation. We have: 
\begin{align}
    (iv)_1 & \overset{(a)}{\leq} \textstyle\sum_{t=1}^T \ex_{\hat \mu_t^\circ} [L_2  \vert \hat \theta^{\circ\top}_t z_t +w - \theta^{*\top} z_t - w  \vert \vert x_t,u_t ]
    = L_2\textstyle\sum_{t=1}^T\ex_{\hat \mu^\circ_t}[ \vert(\hat \theta^{\circ}_t - \theta^*)^\top z_t \vert \vert x_t,u_t]\\
    &\overset{(b)}{\leq} L_2\textstyle\sum\limits_{t=1}^T\ex_{\hat \mu_t^\circ}[ \Vert \hat \theta^{\circ}_t - \theta^*  \Vert_{\bar V^{\circ}_t}  \Vert z_t  \Vert_{ (\bar V^{\circ}_t)^{-1}} \vert x_t,u_t]
    \overset{(c)}{\leq} L_2 \textstyle\sum\limits_{t=1}^T\alpha^{\theta,\circ}_t(\nicefrac{\delta}{2})\ex_{\hat \mu_t^\circ}[ \Vert z_t  \Vert_{(\bar V^{\circ}_t)^{-1}} \vert x_t,u_t]
\end{align}
$(a)$ expands $f(\cdot)$ to remove $\theta$ from the expectation and uses the Lipschitz continuity of $\tilde J_{t,\delta}$ (Lemma~\ref{lem:tj_xlip}). $(b)$ uses Cauchy-Schwarz and $(c)$ uses Proposition~\ref{prop:theta_confidence_set} ($\mac E_\theta(\nicefrac{\delta}{2})$ holds by assumption). Note that $\alpha^{\theta,\circ}_t(\cdot)\leq \alpha^\theta_T(\cdot)$. Thus $(iv):=(iv)_1+(iv)_2\leq (L_1+L_2)\alpha_T^\theta(\nicefrac{\delta}{2})\sum_{t=1}^T\ex_{\hat \mu_t^\circ}[\Vert z_t\Vert_{(\bar V^{\circ}_t)^{-1}}\vert x_t,u_t]$. Define $\xi_t := \ex_{\mu^*}\big[  \Vert z_t  \Vert_{(\bar V^{\circ}_t)^{-1}} \vert x_t,u_t\big] -  \Vert z_t  \Vert_{(\bar V^{\circ}_t)^{-1}}$. We add and subtract $(L_1+L_2)\alpha_{T}^\theta(\nicefrac{\delta}{2})\sum_{t=1}^T\xi_t$:
\begin{align}
    (iv) &\leq (L_1+L_2)\alpha^\theta_T(\nicefrac{\delta}{2})\textstyle\sum\nolimits_{t=1}^T[
    \underbrace{\ex_{\hat \mu^\circ_t}[\Vert z_t \Vert_{(\bar V^{\circ}_t)^{-1}}\vert x_t,u_t ]-\ex_{\mu^*}[ \Vert z_t \Vert_{(\bar V^{\circ}_t)^{-1}}\vert x_t,u_t]}_{(a)_t}\nonumber\\
    &\qquad +\underbrace{\ex_{\mu^*}[ \Vert z_t \Vert_{(\bar V^{\circ}_t)^{-1}}\vert x_t,u_t]-\Vert z_t \Vert_{(\bar V^{\circ}_t)^{-1}}}_{(b)_t}+\underbrace{\Vert z_t \Vert_{(\bar V^{\circ}_t)^{-1}}}_{(c)_t} ]
\end{align}
Note that $\vert(a)_t\vert$ is trivially $0$ if $u_t=\ve{0}_M$, so assume $u_t=e_i$. Let $z_1:=[x_t \ u_t^\top \ u_t^\top]^\top$ and $z_0:=[x_t \ u_t^\top \ \ve{0}_M^\top]^\top$ denote the possible state vectors under adherence and non-adherence. Then by expanding the expectations and simplifying we have $\vert (a)_t\vert \leq |\sigma(x_t+\hat \mu_{i,t}^\circ)-\sigma(x_t+\mu_i^*)|\cdot\lvert\Vert z_1 \Vert_{(\bar V^{\circ}_t)^{-1}} - \Vert z_0 \Vert_{(\bar V^{\circ}_t)^{-1}} \rvert\leq \nicefrac{1}{4}\vert \hat \mu_{i,t}^\circ - \mu_i^*\vert\cdot \Vert z_1 -z_0\Vert_{(\bar V^{\circ}_t)^{-1}}$ (using the Lipschitz continuity of the sigmoid and the reverse triangle inequality). Note that $z_1-z_0=[0 \ \ve{0}_M^\top \ e_i^\top]^\top$ and $(\bar V^{\circ}_t)^{-1}\preceq \bar V_0^{-1}=(\nicefrac{1}{\lambda_1})I_{2M+1}$, so $\Vert z_1 -z_0\Vert_{(\bar V^{\circ}_t)^{-1}} \leq \nicefrac{1}{\sqrt{\lambda_1}}$. Thus by Proposition~\ref{prop:mu_confidence_set}, $\vert(a)_t\vert \leq \frac{1}{4\sqrt{\lambda_1}}\alpha_{i,t}^{\mu,\circ}(\nicefrac{\delta}{2})$. Thus $\vert \sum_{t=1}^T(a)_t\vert \leq \frac{1}{4\sqrt{\lambda_1}}\sum_{t=1}^T\sum_{i=1}^Mu_t^i\alpha_{i,t}^{\mu,\circ}(\nicefrac{\delta}{2})$. Applying Lemma~\ref{lem:mu_sums} with the given $C_N$ gives $\vert \sum_{t=1}^T(a)_t \vert \leq \frac{\nu_T(\nicefrac{\delta}{2})\sqrt{1+C_N}}{4\sqrt{\lambda_1}}[\nicefrac{M}{\sqrt{\lambda_2}}+2\sqrt{\nicefrac{MT}{\underline q}}]$. Next, observe that $(b)_t=\xi_t$ and that $\{\xi_t\}$ is a $\{\mac F_t\}$-adapted martingale difference sequence (with $\vert \xi_t\vert \leq \nicefrac{1}{\sqrt{\lambda_1}}$ by the same logic used to bound $\Vert z_1 -z_0\Vert_{(\bar V^{\circ}_t)^{-1}}$). We apply Azuma-Hoeffding \citep{wainwrightHighdimensionalStatisticsNonasymptotic2019} to $\sum_{t=1}^T\xi_t$. For some $\delta'\in(0,1)$ we have that $\pr(\sum_{t=1}^T\xi_t <
\sqrt{2\log(\nicefrac{1}{\delta'})\frac{T}{\lambda_1}})\geq 1-\delta'$. For $(c)_t$: 
\begin{align}
    \textstyle\sum\nolimits_{t=1}^T(c)_t &
    \overset{(i)}{\leq} \sqrt{\textstyle\sum\nolimits_{t=1}^T1^2}\sqrt{\textstyle\sum\nolimits_{t=1}^T\Vert z_t \Vert^2_{(\bar V_t^\circ)^{-1}}} \overset{(ii)}{\leq} \sqrt{T}\sqrt{(1+C_d)\textstyle\sum\nolimits_{t=1}^T\Vert z_t \Vert^2_{(\bar V_t)^{-1}}}
\end{align}
$(i)$ uses Cauchy-Schwarz. $(ii)$ uses Lemma~\ref{lem:old_norm} to convert $({\bar V}^\circ_t)^{-1}$ to $({\bar V}_t)^{-1}$ per $C_d$. Lemma~\ref{lem:epl_bound} bounds $\Vert z_t \Vert^2_{(\bar V_t)^{-1}}$ and the lemma follows by substituting the $(a)_t$, $(b)_t$, $(c)_t$ summation bounds.\halmos
\endproof

\begin{lemma}{\textbf{(\textcite{abbasi-yadkoriImprovedAlgorithmsLinear2011} Lemma 11)}}\label{lem:ay_lem11}
Let $\{X_t \}_{t=1}^\infty$ be a sequence in $\R^d$, $V$ be a $d\times d$ positive definite matrix, and $\bar V_t:= V + \sum_{s=1}^t X_s X_s^\top$. If $ \Vert X_t \Vert_2 \leq L$ for all $t$ then $\sum_{t=1}^n\min\{1, \Vert X_t  \Vert_{\bar V_{t-1}^{-1}}^2 \}\leq 2 \log(\frac{\det(\bar V_n)}{\det (V)} )\leq 2(d\log( \frac{\trace(V)+nL^2}{d})-\log \det (V))$.
\end{lemma}

\begin{lemma}{}\label{lem:epl_bound}
$\sum_{t=1}^T \Vert z_t  \Vert^2_{\bar V_t^{-1}} \leq (2M+1)\log (1+\frac{T(C_x^2+2)}{(2M+1)\lambda_1} ) [2 +\frac{C_x^2+2}{\log(2)\lambda_1}]$.
\end{lemma}
\proof{Proof of Lemma~\ref{lem:epl_bound}}
We define index sets $\mac I_g := \{ t=1,\dots,T: \Vert z_t  \Vert^2_{\bar V_t^{-1}} < 1\}$ and $\mac I_b :=\{t=1,\dots,T: \Vert z_t  \Vert_{\bar V_t^{-1}}^2\geq 1\}$ with $\mac I_g +\mac I_b = \{1,\dots,T\}$ and $\sum_{t=1}^T \Vert z_t  \Vert_{\bar V_t^{-1}}^2=\sum_{t\in \mac I_g} \Vert z_t  \Vert_{\bar V_t^{-1}}^2+\sum_{t\in \mac I_b} \Vert z_t  \Vert_{\bar V_t^{-1}}^2$. For the $\mac I_g$ summation we directly apply Lemma~\ref{lem:ay_lem11} as for $t\in \mac I_g$ we have $ \Vert z_t  \Vert^2_{\bar V_t^{-1}} \leq \min  \{1, \Vert z_t  \Vert^2_{\bar V_t^{-1}}\}$ by definition. Note that our $\bar V_t$ is equivalent to Lemma~\ref{lem:ay_lem11}'s definition of $\bar V_{t-1}$, $z_t$ takes the place of $X_t$, $d=2M+1$, and $L=\sqrt{C_x^2+2}$. This gives $\sum_{t\in \mac I_g} \Vert z_t  \Vert_{\bar V_t^{-1}}^2 = \sum_{t\in \mac I_g} \min\{1,\Vert z_t  \Vert_{\bar V_t^{-1}}^2\} \leq \sum_{t=1}^T \min\{1,\Vert z_t  \Vert_{\bar V_t^{-1}}^2\}\leq 2(2M+1)\log (1+\frac{T(C_x^2+2)}{(2M+1)\lambda_1} )$.
For the $\mac I_b$ summation we upper bound the summands and show that the number of terms in the summation grows logarithmically in $T$. We first note that for $t \in \mac I_b$ we have $\det(\bar V_{t+1})\geq 2\det(\bar V_t)$, which can be seen as follows. By definition, $\bar V_{t+1} = V +\sum_{s=1}^{t}z_sz_s^\top$, so we equivalently have $\bar V_{t+1} = \bar V_{t}+z_{t}z_{t}^\top=\bar V_{t}(I_{2M+1} + \bar V_{t}^{-1}z_{t}z_{t}^\top)$. Taking the determinant of both sides and applying the matrix determinant lemma (\textcite{dingEigenvaluesRankoneUpdated2007} Lemma 1.1) we have $\det(\bar V_{t+1}) = \det(\bar V_{t})(1+ \Vert z_{t}\Vert_{\bar V_{t}^{-1}}^2 )\geq 2 \det(\bar V_{t})$. Since $\bar V_t$ is constructed from the sum of a PD matrix $V$ and PSD matrices, $\det(\bar V_{t+1})\geq \det(\bar V_t)$ during the time steps included in $\mac I_g$. Thus at the end of time step $T$ it must follow that $\det(\bar V_{T+1}) \geq 2^{|\mac I_b|}\det (V)$, further implying that $|\mac I_b|\leq \log_2( \frac{\det(\bar V_{T+1})}{\det (V)})=\frac{1}{\log(2)}\log( \frac{\det(\bar V_{T+1})}{\det (V)})$. Note that this quantity can be bounded directly by Lemma~\ref{lem:ay_lem11}, i.e.: $|\mac I_b| \leq \frac{2M+1}{\log(2)}\log (1+\frac{T(C_x^2+2)}{(2M+1)\lambda_1} )$. Thus we can complete the bounding of the $\mac I_b$ summation by upper bounding $ \Vert z_t  \Vert_{\bar V_t^{-1}}^2=z_t^\top \bar V_t^{-1} z_t$. Note that $z_t\neq 0$ for $t \in \mac I_b$ as this would imply $ \Vert z_t\Vert^2_{\bar V_t^{-1}} =0$ (and thus $t\in \mac I_g$). By the upper bound of the Rayleigh quotient (\textcite{hornMatrixAnalysis1985} Theorem 4.2.2) we have $ \Vert z_t  \Vert_{\bar V_t^{-1}}^2 \leq  \Vert z_t  \Vert_2^2 \lambda_{\text{max}}(\bar V_t^{-1})= \Vert z_t  \Vert_2^2 \frac{1}{\lambda_{\text{min}}(\bar V_t)}\leq  \Vert z_t  \Vert_2^2 \frac{1}{\lambda_1}\leq \frac{C_x^2+2}{\lambda_1}$. This holds uniformly for $t=1,\dots, T$. Thus we have $\sum_{t\in \mac I_b}\left \Vert z_t \right \Vert_{\bar V_t^{-1}}^2 \leq |\mac I_b| \nicefrac{(C_x^2+2)}{\lambda_1}$. The lemma follows by substituting each summation's bound: $\sum_{t=1}^T \left \Vert z_t \right \Vert^2_{\bar V_t^{-1}} = \sum_{t\in \mac I_g}\left \Vert z_t \right \Vert_{\bar V_t^{-1}}^2 + \sum_{t\in \mac I_b}\left \Vert z_t \right \Vert_{\bar V_t^{-1}}^2$.
\halmos
\endproof
\begin{lemma}{\textbf{(Excerpted \textcite{abbasi-yadkoriImprovedAlgorithmsLinear2011} Lemma 12)}}\label{lem:ay_lem12}
Let $A$, $B$, and $C$ be positive semidefinite matrices such that $A=B+C$. Then $\sup_{x\neq 0}\frac{x^\top A x}{x^\top Bx}\leq \frac{\det(A)}{\det(B)}$.
\end{lemma}

\begin{lemma}{}\label{lem:old_norm}
$\left \Vert z_t \right \Vert^2_{(\bar V^{\circ}_t)^{-1}} \leq (1+C_d)\left \Vert z_t \right \Vert^2_{\bar V_{t}^{-1}}$ for $t=1,\dots,T$.
\end{lemma}
\proof{Proof of Lemma~\ref{lem:old_norm}}
If $z_t=\ve{0}_{2M+1}$ then the lemma follows trivially, so we proceed for $z_t\neq \ve{0}_{2M+1}$. $\bar V^{\circ}_t$ is fixed at the beginning of the current epoch $k$, implying that $\bar V_t = \bar V^{\circ}_t + \sum_{s=\tau_k}^{t-1}z_sz_s^\top$. Since $\bar V_t$ and $\bar V^{\circ}_t$ are PD by construction and $\bar V_t \succeq \bar V_t^\circ$, PD matrices $(\bar V_t)^{-1}$ and $(\bar V^{\circ}_t)^{-1}$ exist with $(\bar V^{\circ}_t)^{-1}\succeq (\bar V_t)^{-1}$. Thus the difference $(\bar V^{\circ}_t)^{-1}-(\bar V_t)^{-1}$ is PSD. Trivially we have $(\bar V^{\circ}_t)^{-1} = (\bar V_t)^{-1} + [(\bar V^{\circ}_t)^{-1}-(\bar V_t)^{-1}]$. Applying Lemma~\ref{lem:ay_lem12} with $\ma{A}=(\bar V^{\circ}_t)^{-1}$, $\ma{B}=(\bar V_t)^{-1}$ and $\ma{C}=(\bar V^{\circ}_t)^{-1}-(\bar V_t)^{-1}$ implies for any $t=1,\dots,T$, $\frac{\left \Vert z_t \right \Vert^2_{(\bar V^{\circ}_t)^{-1}}}{\left \Vert z_t \right \Vert^2_{\bar V_{t}^{-1}}}\leq \sup_{z_t\neq0}\frac{z_t^\top(\bar V^{\circ}_t)^{-1}z_t }{z_t^\top(\bar V_t)^{-1} z_t} \leq \frac{\det\left((\bar V^{\circ}_t)^{-1}\right)}{\det\left((\bar V_t)^{-1} \right)} = \frac{\det \left(\bar V_t \right)}{\det \left(\bar V^{\circ}_t \right)}$. Algorithm~\ref{alg:ucb_epoch} triggers new epochs such that $\det\left(\bar V_t\right) \leq (1+C_d)\det\left(\bar V^{\circ}_t \right)$ holds for $t=1,\dots,T$. so the RHS is upper bounded by $(1+C_d)$ and the lemma follows.\halmos
\endproof

\begin{lemma}{}\label{lem:bound_v}
Let $\delta' \in (0,1)$ and define $\zeta_t := \ex_{\theta^*,\mu^*}\big[\tilde J^{\circ}_{t,\delta}(f(\cdot)) - J^{\pi_t}(f(\cdot))\big] -\big(\tilde J^{\circ}_{t,\delta}(x_{t+1}) - J^{\pi_t}(x_{t+1}) \big)$. With probability at least $1-\delta'$, $\sum_{t=1}^T \zeta_t \leq \sqrt{8T(C_{\tilde J,T+1,\delta}+C_J)^2\log(\nicefrac{1}{\delta'})}$. 
\end{lemma}
\proof{Proof of Lemma~\ref{lem:bound_v}}
By Lemma~\ref{lem:unique_opt}, for $t\leq T$ and $\delta\in(0,1)$ $\Vert\tilde J^\circ_{t,\delta}\Vert_{\infty}\leq C_{\tilde J,T,\delta} \leq C_{\tilde J,T+1,\delta}$ and $\Vert J^{\pi_t}\Vert_\infty\leq C_J$, so $\Vert J^\circ_{t,\delta} - J^{\pi_t}\Vert_\infty \leq  C_{\tilde J,T+1,\delta}+C_J$. Observe that $\{\zeta_t\}$ is a $\{\mac F_t\}$-adapted martingale difference sequence (as the expectation is with respect to the true parameters) with $\vert \zeta_t\vert \leq 2(C_{\tilde J,T+1,\delta}+C_J)$. We apply Azuma-Hoeffding \citep{wainwrightHighdimensionalStatisticsNonasymptotic2019} for $\delta'\in(0,1)$. Thus $\pr\big(\sum_{t=1}^T\zeta_t\geq q\big) \leq \exp{(-\frac{q^2}{8\sum_{t=1}^T(C_{\tilde J,T+1,\delta}+C_J)^2})}$. The lemma follows by solving for $q$ so the RHS equals $\delta'$.
\halmos
\endproof

\begin{lemma}{}\label{lem:bound_vi}
$\sum\limits_{t=1}^T [\tilde J^\circ_{t,\delta}(x_{t+1})-\tilde J^{\circ}_{t+1,\delta}(x_{t+1})+J^{\pi_{t+1}}(x_{t+1})-J^{\pi_t}(x_{t+1}) ]\leq 2(K_d +K_N)(C_{\tilde J,T+1,\delta}+C_J)$.
\end{lemma}
\proof{Proof of Lemma~\ref{lem:bound_vi}}
For Algorithm~\ref{alg:ucb_epoch}, $\tilde J^{\circ}_{t,\delta}=\tilde J^{\circ}_{t+1,\delta}$ and $J^{\pi_t}=J^{\pi_{t+1}}$ for every time step except the final time step of each epoch. We define the index set $\mac I_c$ to contain such steps, i.e., $\{t\in\{1,\dots,T\}:t=\tau_{k}-1,\;k=2,\dots,K+1\}$. The lemma summands are nonzero only for time indices in $\mac I_c$. By Lemma~\ref{lem:unique_opt}, $\Vert \tilde J_{t,\delta}\Vert_\infty \leq C_{\tilde J,T,\delta}\leq C_{\tilde J,T+1,\delta}$ and $\Vert J^{\pi_t}\Vert_\infty \leq C_J$ for $t=1,\dots,T+1$, thus the differences between these value function are also bounded. Specifically, $\tilde J^\circ_{t,\delta}(x_{t+1})-\tilde J^{\circ}_{t+1,\delta}(x_{t+1})\leq 2C_{\tilde J,T+1,\delta}$ and $J^{\pi_{t+1}}(x_{t+1})-J^{\pi_t}(x_{t+1}) \leq 2C_J$. We next bound the number of new epoch triggers $\vert \mac I_c\vert$. First consider the determinant trigger. Denote $n_d$ as the number of times the determinant trigger condition in Algorithm~\ref{alg:ucb_epoch}, $\det(\bar V_{t+1}) > (1+C_d)\det (\bar V^{\circ}_t)$, is met. Further note that, as a second moment matrix, time steps where this condition on $\bar V_{t+1}$ is not met cannot decrease the determinant. Thus for a trajectory $t=1,\dots,T$, we have $\det(\bar V_{T+1}) \geq (1+C_d)^{n_d}\det(\bar V^{\circ}_1)=(1+C_d)^{n_d}\det(V)$. Consequently, $n_d \leq \log_{1+C_d}(\frac{\det(\bar V_{T+1})}{\det(V)} )=\frac{1}{\log(1+C_d)}\log(\frac{\det(\bar V_{T+1})}{\det(V)} )$. Using Lemma~\ref{lem:ay_lem11} to bound the $\log \det$ ratio, we then have $n_d \leq \frac{(2M+1)}{\log(1+C_d)}\log (1+\frac{T(C_x^2+2)}{(2M+1)\lambda_1})$. Next, we consider the action count trigger, with $N_{i,t}:=\sum_{s=1}^{t-1}\mathbf{1}_{\{u_s=e_i\}}$. Let $n_{N,i}$ denote the count of triggering time steps, i.e., when $N_{i,t+1} > (1+C_N)N_{i,\tau_k}$, for $i \in \{1,\dots,M\}$. $N_{i,T+1} \leq T$, implying $n_{N,i}\leq1+\log_{1+C_N}(T)=1+\frac{\log(T)}{\log(1+C_N)}$. Thus it follows that $\vert \mac I_c \vert \leq n_d +\sum_{i=1}^Mn_{N,i} \leq \frac{(2M+1)}{\log(1+C_d)}\log (1+\frac{T(C_x^2+2)}{(2M+1)\lambda_1}) + M(1+\frac{\log(T)}{\log(1+C_N)})$. The lemma follows by bounding the number of nonzero summation indices and the maximum contribution of each nonzero summand, i.e., $\vert \mac I_c \vert \cdot 2(C_{\tilde J,T+1,\delta}+C_J)$.
\halmos
\endproof

%
%
%



\end{document}